\documentclass[11pt]{article}
\usepackage{ifthen}
\usepackage{mdwlist}
\usepackage{subcaption}
\usepackage{amsmath,amssymb,amsfonts,amsthm,mathtools,turnstile}
\usepackage{bm}
\usepackage[dvipsnames]{xcolor}
\usepackage[nottoc,notlof,notlot,notindex]{tocbibind}
\usepackage{hyperref}

\definecolor{darkmodelink}{RGB}{156,220,254}
\usepackage{enumitem}
\usepackage{graphicx}
\usepackage{newpxtext} \usepackage{textcomp} \usepackage[varg,bigdelims]{newpxmath}
\usepackage[scr=rsfso]{mathalfa}\usepackage{xspace}
\usepackage{verbatim}
\usepackage{algorithm}
\usepackage{algpseudocode}
\usepackage{bbold}
\usepackage[letterpaper,margin=1in]{geometry}
\usepackage{cleveref}

\usepackage{thm-restate}
\usepackage{latexsym}
\usepackage{epsfig}
\usepackage[colorinlistoftodos]{todonotes}
\usepackage[normalem]{ulem}

\usepackage{microtype,needspace}
\definecolor{darkcolor1}{RGB}{15, 43, 109}
\definecolor{darkcolor2}{RGB}{90, 15, 28}

\hypersetup{
  colorlinks=true,
  urlcolor=blue,
  linktocpage=true,
  linkcolor=darkcolor2,
  citecolor=darkcolor1,
}

\def\withcolors{0}
\def\withnotes{0}

\newcommand{\email}[1]{\href{mailto:#1}{\textcolor{black}{\nolinkurl{#1}}}}

\makeatletter
\@ifpackageloaded{siamproceedings}{\newsiamremark{remark}{Remark}
  \newsiamthm{fact}{Fact}
  \newsiamthm{claim}{Claim}
  \newsiamthm{problem}{Problem}
  \providecommand{\qed}{\hfill\proofbox}
}{\newtheorem{theorem}{Theorem}
  \newtheorem*{theorem*}{Theorem}

  \newtheorem{problem}{Problem}
  
  \newtheorem{proposition}[theorem]{Proposition}
  
  \newtheorem{remark}[theorem]{Remark}
  \newtheorem{fact}[theorem]{Fact}
  \newtheorem{lemma}[theorem]{Lemma}
  
  \newtheorem{corollary}[theorem]{Corollary}

  \newtheorem{definition}[theorem]{Definition}

  \numberwithin{theorem}{section}
  \numberwithin{assumption}{section}
  \numberwithin{tournament}{section}
  \numberwithin{problem}{section}
  \numberwithin{nontheorem}{section}
  \numberwithin{proposition}{section}
  \numberwithin{observation}{section}
  \numberwithin{remark}{section}
  \numberwithin{fact}{section}
  \numberwithin{lemma}{section}
  \numberwithin{claim}{section}
  \numberwithin{corollary}{section}
  \numberwithin{case}{section}
  \numberwithin{dfn}{section}
  \numberwithin{definition}{section}
  \numberwithin{question}{section}
  \numberwithin{openquestion}{section}
  \numberwithin{res}{section}
  \numberwithin{algorithm}{section}

  \makeatletter

}
\makeatother

\newcommand{\sodared}[1]{{#1}}
\newenvironment{sodarevision}{\begingroup}{\endgroup}

\crefname{algorithm}{Algorithm}{Algorithms}
\Crefname{algorithm}{Algorithm}{Algorithms}

\floatstyle{plain}
\restylefloat{algorithm}
\newsavebox{\algbox}
\newenvironment{boxedalgorithm}[1]{\begin{algorithm}[!t]\refstepcounter{algorithm}\centering
  \setlength{\fboxsep}{10pt}\begin{lrbox}{\algbox}\begin{minipage}{0.88\linewidth}\textbf{Algorithm \thealgorithm} (#1).\par\medskip
}{\end{minipage}\end{lrbox}\noindent\fbox{\usebox{\algbox}}\end{algorithm}}

\newcommand{\alginput}{\par\noindent\textbf{Input:}\ }
\newcommand{\algoutput}{\par\noindent\textbf{Output:}\ }
\newcommand{\algoperations}{\par\noindent\textbf{Operations:}\par}
\newenvironment{algsteps}{\begin{enumerate}[leftmargin=2.25em,itemsep=0.25em,topsep=0.35em]
}{\end{enumerate}
}

\ifnum\withcolors=1
      \newcommand{\gcolor}[1]{{\color{red}#1}}

  \newenvironment{codexblock}{\begingroup\color{MidnightBlue}}{\endgroup}
\else

  \newcommand{\gcolor}[1]{{#1}}

\fi

\ifnum\withnotes=1
  \newcommand{\gnote}[1]{\par\gcolor{\textbf{G: }\sf #1}} \newcommand{\gfootnote}[1]{\footnote{{\bf \gcolor{Gautam}}: {#1}}}

\else
  \newcommand{\gnote}[1]{}
  \newcommand{\gfootnote}[1]{}

\fi
\newcommand{\ignore}[1]{\leavevmode\unskip} 

\newcommand{\paren}[1]{(#1)}
\newcommand{\Paren}[1]{\left(#1\right)}

\newcommand{\Brac}[1]{\left[#1\right]}

\newcommand{\floor}[1]{\lfloor#1\rfloor}

\newcommand{\ceil}[1]{\lceil#1\rceil}

\newcommand{\abs}[1]{\lvert#1\rvert}
\newcommand{\Abs}[1]{\left\lvert#1\right\rvert}

\newcommand{\set}[1]{\{#1\}}
\newcommand{\Set}[1]{\left\{#1\right\}}

\newcommand{\norm}[1]{\lVert#1\rVert}
\newcommand{\Norm}[1]{\left\lVert#1\right\rVert}

\newcommand{\normt}[1]{\norm{#1}_2}

\newcommand{\normo}[1]{\norm{#1}_1}

\newcommand{\iprod}[1]{\langle#1\rangle}

\newcommand{\normsch}[1]{\vert\kern-0.25ex\vert\kern-0.25ex\vert #1
  \vert\kern-0.25ex\vert\kern-0.25ex\vert}
\newcommand{\Normsch}[1]{\left\vert\kern-0.25ex\left\vert\kern-0.25ex\left\vert #1
  \right\vert\kern-0.25ex\right\vert\kern-0.25ex\right\vert}
\newcommand{\Esymb}{\mathbb{E}}
\newcommand{\pE}{\tilde{\mathbb{E}}}
\newcommand{\Psymb}{\mathbb{P}}

\DeclareMathOperator*{\E}{\Esymb}

\DeclareMathOperator*{\ProbOp}{\Psymb}
\renewcommand{\Pr}{\ProbOp}

\newcommand{\suchthat}{\;\middle\vert\;}
\newcommand{\tensor}{\otimes}

\newcommand{\mper}{\,.}
\newcommand{\mcom}{\,,}
\newcommand\bdot\bullet

\DeclareMathOperator{\poly}{poly}

\DeclareMathOperator{\supp}{supp}

\DeclareMathOperator{\sign}{sign}

\newcommand{\N}{\mathbb N}
\newcommand{\R}{\mathbb R}

\newcommand{\cA}{\mathcal A}
\newcommand{\cB}{\mathcal B}

\newcommand{\cL}{\mathcal L}

\newcommand{\cR}{\mathcal R}
\newcommand{\cS}{\mathcal S}

\newcommand{\cY}{\mathcal Y}

\renewcommand{\leq}{\leqslant}
\renewcommand{\le}{\leqslant}
\renewcommand{\geq}{\geqslant}
\renewcommand{\ge}{\geqslant}
\let\epsilon=\varepsilon
\newcommand\MYcurrentlabel{xxx}
\newcommand{\MYstore}[2]{\global\expandafter \def \csname MYMEMORY #1 \endcsname{#2}}
\newcommand{\MYload}[1]{\csname MYMEMORY #1 \endcsname }
\newcommand{\MYnewlabel}[1]{\renewcommand\MYcurrentlabel{#1}\MYoldlabel{#1}}
\newcommand{\MYdummylabel}[1]{}
\newcommand{\torestate}[1]{\let\MYoldlabel\label \let\label\MYnewlabel #1\MYstore{\MYcurrentlabel}{#1}\let\label\MYoldlabel }
\newcommand{\restatetheorem}[1]{\let\MYoldlabel\label
  \let\label\MYdummylabel
  \begin{theorem*}[Restatement of \Cref{#1}]
    \MYload{#1}
  \end{theorem*}
  \let\label\MYoldlabel
}
\newcommand{\restatelemma}[1]{\let\MYoldlabel\label
  \let\label\MYdummylabel
  \begin{lemma*}[Restatement of \cref{#1}]
    \MYload{#1}
  \end{lemma*}
  \let\label\MYoldlabel
}
\newcommand{\restateprop}[1]{\let\MYoldlabel\label
  \let\label\MYdummylabel
  \begin{proposition*}[Restatement of \cref{#1}]
    \MYload{#1}
  \end{proposition*}
  \let\label\MYoldlabel
}
\newcommand{\restatefact}[1]{\let\MYoldlabel\label
  \let\label\MYdummylabel
  \begin{fact*}[Restatement of \cref{#1}]
    \MYload{#1}
  \end{fact*}
  \let\label\MYoldlabel
}
\newcommand{\restate}[1]{\let\MYoldlabel\label
  \let\label\MYdummylabel
  \MYload{#1}
  \let\label\MYoldlabel
}

\newcommand{\eps}{\epsilon}

\allowdisplaybreaks
\newcommand*{\normtv}[1]{\norm{#1}_{\mathrm{TV}}}

\newcommand*{\normf}[1]{\Norm{#1}_{\mathrm{F}}
}
\newcommand*{\normop}[1]{\norm{#1}_{\mathrm{op}}
}

\DeclareMathOperator{\diag}{diag}

\newcommand*{\transpose}[1]{{#1}{}^{\mkern-1.5mu\mathsf{T}}}

\makeatletter
\newcommand{\distas}[1]{\mathbin{\overset{#1}{\kern\z@\sim}}}\newsavebox{\mybox}\newsavebox{\mysim}
\newcommand{\distras}[1]{\savebox{\mybox}{\hbox{\kern3pt$\scriptstyle#1$\kern3pt}}\savebox{\mysim}{\hbox{$\sim$}}\mathbin{\overset{#1}{\kern\z@\resizebox{\wd\mybox}{\ht\mysim}{$\sim$}}}}
\makeatother

\newcommand{\RNum}[1]{\uppercase\expandafter{\romannumeral #1\relax}}

\let\originalleft\left
\let\originalright\right
\renewcommand{\left}{\mathopen{}\mathclose\bgroup\originalleft}
\renewcommand{\right}{\aftergroup\egroup\originalright}

\usepackage[strict]{changepage}

\usepackage{framed}

\definecolor{formalshade}{rgb}{0.95,0.95,1}
\definecolor{darkblue}{rgb}{0.0, 0.0, 0.55}

\newcommand*\tageq{\refstepcounter{equation}\tag{\theequation}}

\DeclareMathOperator{\Res}{Res}
\DeclareMathOperator{\Ext}{Ext}
\DeclareMathOperator{\Proj}{Proj}

\newcommand{\etaabs}{\eta_{\mathrm{abs}}}
\newcommand{\etapriv}{\eta_{\mathrm{priv}}}

\hypersetup{pdftitle={Efficient and Optimal Learning of Discrete Distributions with Person-Level Privacy},pdfauthor={Gautam Kamath, Mahbod Majid, Ankur Moitra, Argyris Mouzakis, Jonathan Ullman}}

\title{Efficient and Optimal Learning of Discrete \\ Distributions with Person-Level Privacy}
\author{Gautam Kamath\thanks{New York University (\email{g@csail.mit.edu}). Supported by an NSERC Discovery Grant and an Ontario Early Researcher Award.}
  \hspace{3em}Mahbod Majid\thanks{Massachusetts Institute of Technology (\email{mahbod@mit.edu}).}
  \hspace{3em}Ankur Moitra\thanks{Massachusetts Institute of Technology (\email{moitra@mit.edu}). Supported in part by DARPA expMath, a Microsoft Trustworthy AI Grant, NSF-CCF 2430381, an ONR grant, and a David and Lucile Packard Fellowship. }
  \\[0.8em]
  Argyris Mouzakis\thanks{University of Texas at Austin and University of Waterloo (\email{argymouz@cs.utexas.edu}). Supported by an Ontario Early Researcher Award and a David R. Cheriton Graduate Scholarship.}
  \hspace{3em}Jonathan Ullman\thanks{Northeastern University (\email{jullman@ccs.neu.edu}). Supported by NSF CNS-2232692 and CNS-2247484.}}
\date{}

\begin{document}
  \maketitle

  \begin{abstract}
    We study learning a discrete distribution under person-level differential privacy in the batched setting. We are given $n$ users, each contributing a batch of $m$ i.i.d.\ samples from an unknown distribution $\mu$ over $[k]$, and seek an estimate of $\mu$ up to total-variation error $\alpha$. We give polynomial-time algorithms under both pure and approximate differential privacy. Under pure $(\varepsilon,0)$-DP, we achieve
    \begin{equation*}
      n = \widetilde O\!\left(\frac{k}{\alpha^2m}+\frac{k}{\alpha\sqrt m\,\varepsilon}+\frac{k}{\varepsilon}\right),
    \end{equation*}
    the first polynomial-time pure-DP guarantee for this problem that matches the information-theoretic optimum up to logarithmic factors. Under approximate $(\varepsilon,\delta)$-DP, we achieve
    \begin{equation*}
      n = \widetilde O\!\left(\frac{k}{\alpha^2m}+\frac{k}{\alpha\sqrt m\,\varepsilon}+k+\frac{\log(1/\delta)}{\varepsilon}\right).
    \end{equation*}
    At $\delta=2^{-\Theta(k)}$, the bound remains $\smash[t]{\widetilde O}(k)$ for fixed $\alpha,\varepsilon,m$, and failure probability. Our algorithms use a new degree-$6$ sum-of-squares estimator for learning from untrusted batches. Its key ingredient is a trace-bounded diagonal majorizer for the difference between the empirical covariance of the repaired batch averages and the multinomial covariance predicted by their mean. We prove that clean data admit this certificate through semidefinite-programming duality and a new application of the stronger symmetric form of Grothendieck's inequality due to Friedland and Lim. The certificate preserves the mean-dependent variance of multinomial batch averages, yielding the near-optimal robust rate $O(\eta\smash[t]{\sqrt{\log(1/\eta)/m}})$ and providing the score required by the privacy reduction. We also prove a statistical query lower bound that rules out a uniformly smaller order of error under fixed polynomial SQ resource bounds.
  \end{abstract}

  \newpage
  \tableofcontents
  \newpage

  \section{Introduction}
Federated learning~\cite{McMahanMRHY17,KairouzMABBBBCCCDEEEGGGGGGHHHHHHJJJKKKKKLLMMNOPQRRRSSSSSTVWXXYYYZ21} is motivated by the tension between two important goals in modern machine learning:
\begin{enumerate}
  \item to train machine learning models from large and diverse data sets, and
  \item while doing so, to preserve the data privacy, security, and ownership of the data, since data is increasingly distributed across a network of devices.
\end{enumerate}
Motivated by such considerations, Qiao and Valiant~\cite{QiaoV18} introduced the following problem which they called \emph{learning from untrusted batches}. There exists an unknown discrete distribution $\mu$ on $[k]$ and data is drawn in $n$ batches, each of size $m$. When all batches are drawn i.i.d.\ from $\mu$ and can be used without restriction, the empirical distribution achieves total-variation error $\alpha$, provided that the number of batches satisfies $$n = \Theta\left(\frac{k + \log(1/\beta)}{\alpha^2 m}\right),$$
where $\beta$ is the failure probability. This follows from the classical bound for $nm$ independent samples~\cite{Canonne20a}.

However, in federated learning, we can rarely assume that all the batches are sampled i.i.d.\ from $\mu$, and instead must consider the possibility that some are corrupted, possibly in an adversarial way.  To address this question, Qiao and Valiant studied the problem with the twist that an $\eta$-fraction of the batches are adversarially corrupted. Again, it is helpful to compare to a naive baseline. If we ignore the batch structure and treat the data as a single dataset of $nm$ samples, then it is easy to see that we cannot learn the distribution to accuracy better than $\eta$ in total variation distance, no matter how much data we have.  Qiao and Valiant demonstrated that when we exploit the batched structure of the data, it is possible to do much better. In particular, they showed that it is information-theoretically possible to achieve error ${O}(\eta/\sqrt{m})$, and that this is the best possible.  Batch information makes it possible to robustly distinguish distributions at much finer scale.  Without batch information, we cannot distinguish two distributions $\mu$ and $\tilde{\mu}$ that are $\gamma$-close in total variation distance in the face of an adversary, because they can convert one distribution to the other by moving just a $\gamma$ fraction of the probability mass.  But when data is batched and the adversary's corruptions must respect the batch structure, the total variation distance between $\mu^{\otimes m}$ and $\tilde{\mu}^{\otimes m}$ grows at least as $\Omega(\min\{1,\gamma \sqrt{m}\})$, meaning the adversary has to corrupt a much larger fraction of batches to make these distributions indistinguishable.

Qiao and Valiant gave two algorithms that solve the learning from untrusted batches problem, running in time exponential in $k$ or $m$, respectively. Jain and Orlitsky~\cite{JainO20a} gave a semidefinite-programming algorithm based on Grothendieck's inequality that roughly matches the information-theoretic bounds and runs in polynomial time. In independent and concurrent work, Chen, Li, and Moitra~\cite{ChenLM20a} gave similar algorithms and explored extensions to more complex settings, including improving sample complexity when there are shape constraints on $\mu$.

Qiao and Valiant motivated their model by the need to learn from batches supplied by potentially malicious agents. Privacy, communication, and computational costs are also central concerns in federated learning. Keeping raw data local does not by itself ensure privacy: shared gradients can still leak private training data~\cite{ZhuLH19}.

In this work, we revisit the learning from batches problem using the lens of differential privacy~\cite{DworkMNS06}.  There is an unknown discrete distribution $\mu$ on $[k]$ and data is drawn in $n$ batches of size $m$.  As is common in federated learning, we think of each batch of data as belonging to one person (e.g.\ a batch consists of the data on one person's device), and want to design a \emph{batch-level} differentially private\footnote{This notion of privacy is also frequently called \emph{person-level} or \emph{user-level} differential privacy in the literature.} algorithm that, informally, hides the information of any one batch of data. Formally, two datasets are adjacent if they differ in one person's entire batch. Batch-level privacy protects a person's entire contribution, whereas sample-level privacy protects only one data item. Differential privacy bounds how much the output distribution can change when one person's data is replaced.

We pursue three goals. First, we require the strong guarantee of pure batch-level differential privacy, formally $(\eps,0)$-DP.\footnote{For readers unfamiliar with the formalism of differential privacy, $\eps > 0$ quantifies privacy loss, and pure differential privacy is the most stringent standard variant.} Second, we seek the information-theoretically optimal user complexity. Prior work established that, up to logarithmic factors, the optimum for learning to total-variation error $\alpha$ is~\cite[Theorem~43]{GhaziKKMMZ23b}
\begin{equation*}
  n = \widetilde O\left(\frac{k}{\alpha^2 m} + \frac{k}{\alpha \sqrt{m} \varepsilon} + \frac{k}{\varepsilon}\right).
\end{equation*}
Our final goal is to attain this bound in polynomial time.

Several groups have considered private learning from batches, but none give an algorithm satisfying all three goals.  Liu et al.~\cite{LiuSYKR20} introduced the problem, giving an efficient algorithm achieving the weaker notion of \emph{approximate} differential privacy, or, formally, $(\varepsilon, \delta)$-differential privacy. One can strengthen the privacy guarantee of their algorithm to $(\varepsilon, 0)$-differentially private by appealing to basic composition of pure differential privacy instead of advanced composition of approximate differential privacy~\cite{DworkRV10}, but doing so would increase the required number of batches by a $\mathrm{poly}(k)$ factor and would not achieve optimal sample complexity.  Ghazi, Kamath, Kumar, Manurangsi, Meka, and Zhang~\cite{GhaziKKMMZ23b} presented a pure differentially private algorithm with the information-theoretically optimal sample complexity, but their stated construction applies private hypothesis selection over an exponentially large cover and is not polynomial-time.

We use the efficient robustness-to-privacy reduction of Hopkins, Kamath, Majid, and Narayanan~\cite{HopkinsKMN23}. Our contribution is to identify the robust certificate that makes this reduction work for discrete batches. It is a single degree-$6$ sum-of-squares (SoS) certificate that represents discarded batches, mean feasibility, and the mean-dependent covariance of multinomial batch averages. This certificate simultaneously yields a robust estimator and the score used for privacy.

\Needspace{6\baselineskip}
\subsection{Results}
Our certificate is satisfied by clean data at the right scale, forces every feasible pseudo-mean to be close to the truth, remains stable when batches are discarded, and admits polynomial-time optimization. It yields the robust estimator in \Cref{thm:main-robust-intro} and, through the reduction of Hopkins et al.~\cite{HopkinsKMN23}, the private algorithms in \Cref{thm:main-pure-private-intro}.

\begin{theorem}[Robust SoS Estimation from Discrete Batches]
  \label{thm:main-robust-intro}
  Let $\mu \in \Delta^{k-1}$, $\beta \in (0,1)$, and $m \in \mathbb{N}$.
  Let $\eta>0$ be smaller than a sufficiently small absolute constant.
  Then there exists a polynomial-time algorithm, obtained by solving a degree-$6$ SoS relaxation, that takes in $n$ user batches of size $m$ forming an $\eta$-corruption of independent clean batches
  \begin{equation*}
    B_i^\star = (U_{i,1}^\star,\ldots,U_{i,m}^\star),
    \qquad
    U_{i,1}^\star,\ldots,U_{i,m}^\star \stackrel{\mathrm{iid}}{\sim} \mu,
  \end{equation*}
  and outputs an estimate $\hat{\mu}$ such that, with probability at least $1-\beta$,
  \begin{equation*}
    \normtv{\hat{\mu} - \mu} \le O\Paren{\eta\sqrt{\log(1/\eta)/m}},
  \end{equation*}
  provided that
  \begin{equation*}
    n = \Omega\Paren{\frac{k + \log(1/\beta)}{\alpha^2 m}}
    \qquad \text{for } \alpha = \Theta\Paren{\eta\sqrt{\log(1/\eta)/m}} \mper
  \end{equation*}
\end{theorem}

\begin{remark}[Robustness to per-batch distribution shift]
  Our robust estimator also tolerates mild heterogeneity across the clean batches, with the modified covariance allowance in \Cref{cor:robust-heterogeneous}. Namely, in addition to the $\eta$-corruption model in \Cref{thm:main-robust-intro}, one may allow the samples in each clean batch $i$ to be drawn i.i.d.\ from a distribution $\mu_i$ satisfying $\normtv{\mu_i-\mu} \le \omega$. In that case the guarantee worsens by only an additive $O(\omega)$ term, giving
  \begin{equation*}
    \normtv{\hat{\mu}-\mu} \le O\Paren{\eta\sqrt{\log(1/\eta)/m} + \omega} \mper
  \end{equation*}
\end{remark}

\medskip
The diagonal-majorizer certificate defines a low-sensitivity, quasi-convex score (\Cref{def:score-function}), so the efficient robustness-to-privacy reduction gives the private guarantees stated next.

\begin{sodarevision}
  \begin{theorem}[Main Result, Private Estimation from Discrete Batches]
    \label{thm:main-pure-private-intro}
    \label{thm:main-approx-private-intro}
    Let $\mu \in \Delta^{k-1}$ be a discrete distribution over $[k]$.
    Let $\alpha,\beta\in(0,1)$, $\eps>0$, and $m\in\N$. Suppose that the batches are independent across users and each of the $n$ users contributes a batch
    \begin{equation*}
      B_i=(U_{i,1},\ldots,U_{i,m}),
      \qquad
      U_{i,1},\ldots,U_{i,m}\stackrel{\mathrm{iid}}{\sim}\mu.
    \end{equation*}
    Then the following guarantees hold.
    \begin{enumerate}[label=(\roman*)]
      \item There is a polynomial-time $(\eps,0)$-DP algorithm that outputs $\hat\mu$ satisfying $\normtv{\hat\mu-\mu}\le\alpha$ with probability at least $1-\beta$, provided
        \begin{equation*}
          n=\widetilde O\Paren{\frac{k+\log(1/\beta)}{\alpha^2m}+\frac{k+\log(1/\beta)}{\alpha\sqrt m\,\eps}+\frac{k+\log(1/\beta)}{\eps}} \mper
        \end{equation*}
      \item For every $\delta\in(0,1)$, there is a polynomial-time $(\eps,\delta)$-DP algorithm with the same accuracy and failure probability, provided
        \begin{equation*}
          n=\widetilde O\Paren{k+\frac{k+\log(1/\beta)}{\alpha^2m}+\frac{k+\log(1/\beta)}{\alpha\sqrt m\,\eps}+\frac{\log(1/\delta)+\log(1/\beta)}{\eps}} \mper
        \end{equation*}
    \end{enumerate}
  \end{theorem}

  The $\smash[t]{\widetilde O}$ notation suppresses logarithmic factors in $1/\alpha$; Remark~\ref{rem:explicit-private-logs} gives explicit bounds, including the square-root logarithm in the approximate-DP privacy term. Our pure-DP rate matches known lower bounds up to logarithmic factors~\cite{LiuSYKR20,GhaziKKMMZ23b}; \Cref{thm:dp-lower-bound} gives the dependence on the failure probability explicitly.

  Prior polynomial-time approximate-DP results cover different user regimes. Suppressing logarithmic factors other than those displayed below, the relevant privacy terms and user-count requirements are
  \begin{equation*}
    \begin{aligned}
      \text{Liu et al.~\cite{LiuSYKR20}:}\quad
      &\frac{k\sqrt{\log(1/\delta)}}{\alpha\eps\sqrt m},
      & n&\gtrsim\frac{\sqrt{k\log(1/\delta)}}{\eps},\\
      \text{(few users) Esfandiari et al.~\cite{NarayananME22}:}\quad
      &\frac{k\log(k/\delta)\sqrt{\log(1/\delta)\log(kn/\beta)}}{\alpha\eps\sqrt m},
      & n&\gtrsim\frac{\log(1/(\beta\delta))}{\eps}.
    \end{aligned}
  \end{equation*}
  Esfandiari et al.'s few-user guarantee removes the dimension-dependent user floor required by Liu et al. Their all-regime result combines this guarantee with earlier algorithms for the high-user regime; its larger logarithmic factor is the cost of reaching the few-user regime. For negligible approximate-privacy parameters $\delta=2^{-k^{\Omega(1)}}$ and fixed $\alpha,\eps,m$, the privacy terms of both Liu et al. and Esfandiari et al. are superlinear in $k$. Our bound has additive dependence on $\log(1/\delta)$ and includes a $\smash[t]{\widetilde O}(k)$ user term for clean-data feasibility. At $\delta=2^{-\Theta(k)}$, our bound is $\smash[t]{\widetilde O}(k)$ for fixed $\alpha,\eps,m$, and failure probability.

  \paragraph{Dependence on the batch size.}
  For constant failure probability and $0<\eps\le1$, when $\alpha\sqrt m\le1$, the term $k/(\alpha\sqrt m\,\eps)$ absorbs $k/\eps$, so the pure-DP theorem already gives the approximate-DP rate up to logarithmic factors. When $\alpha\sqrt m\ge1$, we cap the internal corruption parameter at a sufficiently small constant. The robust error is then at most $\alpha$, giving a pure-DP rate of $\smash[t]{\widetilde O}(k/\eps)$.

  The additional $k$ in the approximate-DP bound is inherent in using our robust estimator at its stated accuracy. For any fixed sufficiently small corruption fraction $\eta$, \Cref{thm:main-robust-intro} guarantees error $O(\eta\smash[t]{\sqrt{\log(1/\eta)/m}})=O(1/\smash[t]{\sqrt m})$, which requires $\Omega(k)$ users even without corruptions or privacy. Indeed, the nonprivate sample bound $\Omega(k/(\alpha^2m))$ becomes $\Omega(k)$ at this accuracy. This requirement is already implicit in the robust theorem through $\alpha=\Theta(\eta\smash[t]{\sqrt{\log(1/\eta)/m}})$, and is absorbed by $k/\eps$ for pure DP. It appears explicitly for approximate DP because the other terms can be smaller than $k$.

  For larger batches, the target error $\alpha$ can be much larger than $1/\sqrt m$, and approximate DP can use fewer than $k$ users~\cite{NarayananME22}. Exploiting this regime through robustness-to-privacy requires a robust guarantee adapted to the coarser target accuracy. Remark~\ref{rem:approx-batch-splitting} describes how splitting batches reduces the user requirement while retaining additive dependence on $\log(1/\delta)$, and how to combine our estimator with that of Esfandiari et al.\ using only the public parameters.

  \paragraph{On optimality under efficient computation.}
  Our polynomial-time robust algorithm achieves error $O(\eta\smash[t]{\sqrt{\log(1/\eta)/m}})$, while the information-theoretic optimum is $O(\eta/\sqrt m)$~\cite{QiaoV18}. Such an extra $\smash[t]{\sqrt{\log(1/\eta)}}$ factor has precedent in statistical query (SQ) lower bounds for robust estimation of binary product distributions~\cite{DiakonikolasKS22} and Gaussians~\cite{DiakonikolasKS17}. In \Cref{sec:sq-lower-bound}, we prove a lower bound directly for categorical batches. For every fixed integer $t\ge1$, we give a query--tolerance tradeoff at error a constant $c_t>0$ times
  \begin{equation*}
    \eta\min\Set{1,\sqrt{\frac{\log(1/\eta)}{m}}} \mper
  \end{equation*}
  For all sufficiently large $k$ depending on $t$ and every batch size, achieving this error requires either a query of tolerance less than $2k^{-(t+1)/8}$ or $\exp(\Omega(\smash[t]{\sqrt k}))$ queries. Consequently, no SQ estimator with fixed polynomial resource bounds in $k,m,1/\eta$ achieves uniformly smaller order of error. This rate matches the dependence on $m$ and $\eta$ of our robust estimator when $m\ge\log(1/\eta)$. For smaller batches, the contaminated empirical distribution attains error $O(\eta)$ with sufficiently many samples. Our construction uses moment matching as in~\cite{DiakonikolasKS22}, with a discrete polynomial correction and a balanced subset of the original alphabet. If this SQ lower bound captures the computational threshold for robust estimation, then the privacy-to-robustness connection of Georgiev and Hopkins~\cite{GeorgievH22} predicts a corresponding threshold for private estimation with high success probability. For $0<\eps\le1$ and $\beta=\exp(-\Theta(k))$, group privacy gives robustness to an $\eta=\Theta(k/(n\eps))$ fraction of corrupted batches, provided this fraction is sufficiently small. When $\alpha\sqrt m$ is small and $m\ge\log(1/\eta)$, substituting this value into the robust threshold predicts a privacy term of order
  \begin{equation*}
    \frac{k\sqrt{\log(1/(\alpha\sqrt m))}}{\alpha\sqrt m\,\eps}.
  \end{equation*}
  This agrees with the explicit privacy term in Remark~\ref{rem:explicit-private-logs}, including the square-root logarithm. Georgiev and Hopkins derive an analogous conditional square-root logarithmic gap for Gaussian mean estimation~\cite[full version, Corollary~4.2]{GeorgievH22}. This interpretation also applies to approximate DP for $\delta\le\eps\exp(-Ck)$ with a sufficiently large constant $C$.
\end{sodarevision}

\begin{remark}[Robustness of the private estimators to corruptions]
  For $0<\eps\le1$, both private estimators tolerate an $\eta$-fraction of fully corrupted user batches, under their respective sample bounds in \Cref{thm:main-pure-private-intro}, provided
  \begin{equation*}
    \alpha\ge C\eta\sqrt{\log(1/\eta)/m}
  \end{equation*}
  for a sufficiently large absolute constant $C$ and sufficiently small $\eta$. Both algorithms also tolerate clean-user distributions $\mu_i$ satisfying $\normtv{\mu_i-\mu}\le\omega$, provided the algorithm is given $\omega$ and the accuracy floor includes an additive $C\omega$ term. For approximate DP, a sufficient sample bound includes an additional $\smash[t]{\widetilde O}(k\omega/(\alpha\eps))$ term, which is absorbed by the stated bound when $\alpha\sqrt m\le1$ or when $\eps$ is constant. We prove these claims in \Cref{cor:private-robustness}.
\end{remark}

\subsection{Techniques}
\label{sec:techniques}

\begin{sodarevision}
  We first describe the subset-variance condition that identifies the unknown distribution, then encode it by a polynomial-size semidefinite constraint. We define the score used by the privacy reduction as the minimum number of batches that must be discarded to satisfy this constraint.

  \paragraph{Subset-variance control.}
  Let $y_i\in\Delta^{k-1}$ denote the empirical distribution of the $i$th batch. Fix a subset $S\subseteq[k]$ and let $s\in\Set{0,1}^k$ be its indicator. Then $\iprod{y_i,s}$ is simply the fraction of the $i$th batch that falls in $S$, and has variance $\mu(S)(1-\mu(S))/m$. We write this variance as $V(\nu,s)\coloneq\iprod{\nu,s}(1-\iprod{\nu,s})/m$ when the underlying distribution is $\nu$.

  \begin{samepage}
    A statistically natural goal is to find a set $T\subseteq[n]$ with $\abs{T}\ge(1-\eta)n$ such that, for $\bar\mu_T\coloneq\abs{T}^{-1}\sum_{i\in T}y_i$,
    \begin{equation}
      \label{eq:intro-subset-variance}
      \left|
      \frac1{\abs{T}}\sum_{i\in T}\iprod{y_i-\bar\mu_T,s}^2
      -V(\bar\mu_T,s)
      \right|
      \le O\Paren{\frac{\eta\log(1/\eta)}{m}}
      \qquad\text{for every }s\in\Set{0,1}^k.
    \end{equation}
    This condition says that, in every subset direction, the empirical variance should agree with the multinomial variance predicted by the empirical mean $\bar\mu_T$. 
  \end{samepage}

  Representing this condition poses a problem: there are exponentially many subset directions, and listing all of them does not give the succinct, low-sensitivity certificate needed for privacy. Earlier robust algorithms do not directly provide such a certificate. Jain and Orlitsky~\cite{JainO20a} obtain the right rate through filtering and randomized rounding. Their guarantee does not directly supply the efficiently computable, sensitivity-one, quasi-convex score needed by the privacy reduction. Chen, Li, and Moitra~\cite{ChenLM20a} give an SoS formulation. Bounding only second moments in their argument gives $O(\smash[t]{\sqrt{\eta/m}})$ error, polynomially worse in $\eta$ than our rate. Their higher-moment formulation gives error $O(\eta^{1-1/t}\smash[t]{\sqrt{t/m}})$ using degree-$O(t)$ constraints; obtaining $O(\eta\smash[t]{\sqrt{\log(1/\eta)/m}})$ by this approach uses $t=\Theta(\log(1/\eta))$ and hence quasipolynomial resources.

  \paragraph{Compressing the subset constraints.}
  The subset-variance condition \eqref{eq:intro-subset-variance} contains one inequality for each of the $2^k$ subsets of $[k]$, so it cannot be inserted directly into a polynomial-size program. We replace this family by one semidefinite condition: the difference between the empirical covariance and the multinomial covariance predicted by the empirical mean must lie below a diagonal matrix with small trace. We call this diagonal matrix a majorizer.

  Two facts justify the replacement. First, clean data admit such a majorizer; proving this uses semidefinite duality and the stronger symmetric Grothendieck inequality of Friedland and Lim~\cite{FriedlandL20}. Second, testing any feasible majorizer on sign vectors recovers the required upper variance bound in every subset direction. Thus the majorizer is a polynomial-size encoding of the one-sided subset control we need. 

  \paragraph{Robust SoS program.}
  We define the relaxation in \Cref{subsec:sos-based-robust-algo}. Its constraints have the following schematic interpretation: retain almost all batches, repair the discarded ones, let $\mu'$ be the mean of the repaired data, and require the repaired covariance to lie below the multinomial covariance plus a low-trace diagonal matrix. We write $\E_i z_i\coloneq n^{-1}\sum_{i=1}^n z_i$ and search over selectors $w_i$, repaired batch averages $y_i'$, a candidate mean $\mu'$, a diagonal vector $\lambda$, and a matrix $M$ satisfying
  \begin{equation*}
    \begin{aligned}
      &w_i^2=w_i,\qquad w_i y_i'=w_i y_i,\qquad \sum_{i=1}^n w_i\ge(1-\eta)n,\\
      &y_i'\in\Delta^{k-1}\text{ for all }i,\qquad \mu'\in\Delta^{k-1},\qquad \mu'=\E_i y_i',\\
      &\widehat\Sigma'+MM^\top=\Sigma(\mu')+\operatorname{Diag}(\lambda),
      \qquad \sum_{j=1}^k\lambda_j\le\tau_{\mathrm{cov}},
    \end{aligned}
  \end{equation*}
  where
  \begin{equation*}
    \widehat\Sigma'\coloneq\E_i(y_i'-\mu')\transpose{(y_i'-\mu')},
    \qquad
    \Sigma(\nu)\coloneq\frac1m\Paren{\diag(\nu)-\nu\transpose{\nu}},
    \qquad
    \tau_{\mathrm{cov}}\asymp\frac{\eta\log(1/\eta)}{m}.
  \end{equation*}
  We denote this system by $\mathcal A_\eta$. In the first line, we select the retained batches and enforce consistency with the observed data; in the second, we require the repaired batch averages and their mean to be distributions. For every subset indicator $s$, we have $s^\top\Sigma(\nu)s=V(\nu,s)$, so the covariance line encodes the upper bound in \eqref{eq:intro-subset-variance} for the repaired data. Here $M$ is an auxiliary slack variable: writing the positive-semidefinite slack as $MM^\top$ encodes the inequality $\smash[t]{\widehat\Sigma}'-\Sigma(\mu')\preccurlyeq\operatorname{Diag}(\lambda)$ as a polynomial equality.

  \begin{samepage}
    The program repairs discarded batches rather than removing and renormalizing them. On the clean witness, a discarded batch is filled in by its latent clean value; in the soundness proof, the influence of all repaired batches is controlled through their second moment. This represents the upper bound in \eqref{eq:intro-subset-variance} by a polynomial system over all $n$ batches.
  \end{samepage}

  We solve the degree-$6$ relaxation and use its pseudo-mean $\pE\mu'$ as the estimate, with an $\ell_1$ projection onto the simplex to account for numerical error. There is no filtering or randomized rounding afterward. It remains to justify the covariance constraint: clean data must admit a low-trace majorizer, and every feasible majorizer must recover the required subset bounds. We prove these statements in this order.

  \paragraph{Clean-data feasibility.}
  We assign the selector, repaired batches, and candidate mean directly from the latent clean data. Let $r_i=\mathbb{1}[y_i=y_i^*]$ and $\bar\mu=\E_i y_i^*$, and set
  \begin{equation*}
    w_i=r_i,\qquad y_i'=y_i^*,\qquad \mu'=\bar\mu.
  \end{equation*}
  These assignments satisfy the selector and simplex constraints by inspection. It remains to construct $\lambda$. Define
  \begin{equation*}
    \widehat\Sigma^\star\coloneq\E_i(y_i^*-\bar\mu)\transpose{(y_i^*-\bar\mu)},
    \qquad
    \Delta_{\mathrm{cov}}^\star\coloneq\widehat\Sigma^\star-\Sigma(\bar\mu).
  \end{equation*}
  If $\Delta_{\mathrm{cov}}^\star$ has a diagonal majorizer with trace at most $\tau_{\mathrm{cov}}$, then we choose $MM^\top=\operatorname{Diag}(\lambda)-\Delta_{\mathrm{cov}}^\star$ and the clean witness is complete.

  The only probabilistic input is the $\eta$-good clean-sample event. Under the sample-size assumption in \Cref{thm:main-robust-intro}, \Cref{cor:eta-good} guarantees this event with the required probability. For feasibility, we use only the following two consequences:
  \begin{equation*}
    \sup_{s\in\Set{0,1}^k}\Abs{s^\top\Delta_{\mathrm{cov}}^\star s}
    \le O\Paren{\frac{\eta\log(1/\eta)}{m}},
    \qquad
    \Delta_{\mathrm{cov}}^\star\bm{1}=0.
  \end{equation*}
  We prove the subset-discrepancy bound in \Cref{lem:clean-cov-gap-subset}, and the identity $\Delta_{\mathrm{cov}}^\star\bm1=0$ in \Cref{lem:clean-cov-gap-nullspace}. From this point on, the feasibility argument is deterministic.

  \begin{samepage}
    \paragraph{From subset bounds to a diagonal majorizer.}
    Let $\Delta$ be any symmetric matrix satisfying $\Delta\bm{1}=0$. Combining diagonal-majorizer duality from \Cref{cor:diagonal-majorizer-sdp-duality}, the Friedland--Lim inequality from \Cref{cor:friedland-lim-correlation}~\cite{FriedlandL20}, and the substitution $\sigma=2s-\bm{1}$ gives
    \begin{equation*}
      \begin{aligned}
        \min_{\lambda\in\R^k:\ \operatorname{Diag}(\lambda)\succcurlyeq\Delta}
        \sum_{j=1}^k\lambda_j
        &=
        \max_{\substack{X\succcurlyeq0\\\diag(X)=\bm{1}}}\iprod{\Delta,X}\\
        &\le
        \max_{\substack{X\succcurlyeq0\\\diag(X)=\bm{1}}}\Abs{\iprod{\Delta,X}}\\
        &\le C_{\mathrm{FL}}
        \max_{\sigma\in\Set{\pm1}^k}\Abs{\sigma^\top\Delta\sigma}
        =4C_{\mathrm{FL}}
        \max_{s\in\Set{0,1}^k}\Abs{s^\top\Delta s},
      \end{aligned}
    \end{equation*}
  \end{samepage}
  where $C_{\mathrm{FL}}$ is a universal constant. The final equality follows by expanding $\sigma=2s-\bm{1}$ and using $\Delta\bm{1}=0$. This explains why the nullspace condition matters: it turns every sign-vector quadratic form into four times a subset quadratic form. Friedland--Lim passes from sign vectors to the correlation SDP, and duality turns the SDP bound into the diagonal vector $\lambda$ used by the program. See \Cref{lem:subset-to-sign-nullspace,cor:friedland-lim-correlation,cor:diagonal-majorizer-sdp-duality} for the details.

  Applying this chain to $\Delta_{\mathrm{cov}}^\star$ produces a majorizer of trace $O(\eta\log(1/\eta)/m)$. Choosing the constant in $\tau_{\mathrm{cov}}$ sufficiently large therefore completes the clean witness. We next recover the one-sided subset bound needed for identifiability.

  \paragraph{Recovering the subset bounds.}
  To recover the subset bounds, let $\Delta_{\mathrm{cov}}'\coloneq\smash[t]{\widehat\Sigma}'-\Sigma(\mu')$. Both terms annihilate $\bm{1}$, so $\Delta_{\mathrm{cov}}'\bm{1}=0$. For any fixed $s\in\Set{0,1}^k$, set $\sigma=2s-\bm{1}$. Testing the majorizer on $\sigma$ gives
  \begin{equation*}
    4s^\top\Delta_{\mathrm{cov}}'s
    =\sigma^\top\Delta_{\mathrm{cov}}'\sigma
    \le\sigma^\top\operatorname{Diag}(\lambda)\sigma
    =\sum_{j=1}^k\lambda_j.
  \end{equation*}
  Thus testing the majorizer on the sign vector associated with $s$ recovers the desired bound in that subset direction. Since $s^\top\Sigma(\mu')s=V(\mu',s)$, we obtain the following degree-$6$ SoS upper bound on the repaired directional variance:
  \begin{equation*}
    \mathcal A_\eta\sststile{6}{}
    \Set{
      \E_i\iprod{y_i'-\mu',s}^2
      \le V(\mu',s)+O\Paren{\frac{\eta\log(1/\eta)}{m}}
    }.
  \end{equation*}
  This is the upper bound in \eqref{eq:intro-subset-variance}, applied to the repaired data. The two-sided subset-discrepancy bound in \Cref{lem:clean-cov-gap-subset} was needed to construct the majorizer. For identifiability, the program uses only the repaired-variance upper bound, together with the clean-retained second-moment lower bound in \eqref{eq:missing-clean-second-moment}.

  \paragraph{Low-degree SoS identifiability proof.}
  We argue in one dimension. Fix a subset direction $s\in\Set{0,1}^k$. Readers unfamiliar with SoS may read $\mathcal A_\eta\sststile{6}{}\Set{\cdots}$ as ``the program certifies the inequality inside the braces.'' The degree only records that the certificate remains valid under the degree-$6$ pseudoexpectation; statistically, this is the usual directional robust-mean argument.

  On the $\eta$-good clean-sample event, at most a $2\eta$ fraction of the batches fail to be both clean and retained. Weighted Cauchy--Schwarz reduces their contribution to a comparison between the repaired and clean-retained second moments. Since $\mu'=\E_i y_i'$, expanding $y_i'-\mu=(y_i'-\mu')+(\mu'-\mu)$ makes the cross term vanish. The majorizer then upper-bounds the repaired second moment through $V(\mu',s)$, while the clean-sample event lower-bounds the clean-retained second moment through $V(\mu,s)$, giving
  \begin{equation}
    \label{eq:intro-directional-identifiability}
    \mathcal A_\eta\sststile{6}{}
    \Set{
      \iprod{\mu'-\mu,s}^2
      \le O\Paren{\frac{\eta^2\log(1/\eta)}{m}}
      +4\eta\Paren{V(\mu',s)-V(\mu,s)}
      +4\eta\iprod{\mu'-\mu,s}^2
    }.
  \end{equation}
  We give the full calculation in \Cref{subsec:proof-of-main-sos-lemma}. In \eqref{eq:intro-directional-identifiability}, the $O(\eta^2\log(1/\eta)/m)$ term has the desired size, and the term $4\eta\iprod{\mu'-\mu,s}^2$ can be absorbed into the left-hand side. It remains to control the variance difference $V(\mu',s)-V(\mu,s)$.

  To control the variance difference, we keep the covariance predicted by $\mu'$ rather than use a universal upper bound. For every distribution $\nu$, we have $s^\top\Sigma(\nu)s=V(\nu,s)=\iprod{\nu,s}(1-\iprod{\nu,s})/m$. It follows that $V(\mu',s)-V(\mu,s)$ is controlled by the directional mean error $\iprod{\mu'-\mu,s}$. By contrast, using only $V(\mu',s)\le 1/(4m)$ would leave an $O(\eta/m)$ term and yield only $O(\smash[t]{\sqrt{\eta/m}})$ directional error. Bounding second moments in the formulation of Chen, Li, and Moitra~\cite{ChenLM20a} encounters this limitation; their higher-moment constraints improve the rate.

  To absorb the term $4\eta(V(\mu',s)-V(\mu,s))$ in \eqref{eq:intro-directional-identifiability}, \Cref{lem:sos-cov} bounds the square of the variance difference by the squared mean error, as shown in the first line below. Apply Young's inequality from \Cref{fact:sos-young-inequality} with $a=V(\mu',s)-V(\mu,s)$, $b=4\eta$, and $\lambda=m^2/16$. This gives the second line, and substituting the first line gives the third:
  \begin{equation*}
    \mathcal A_\eta\sststile{6}{}
    \Set{
      \begin{aligned}
        \Paren{V(\mu',s)-V(\mu,s)}^2
        &\le\frac{4}{m^2}\iprod{\mu'-\mu,s}^2,\\
        4\eta\Paren{V(\mu',s)-V(\mu,s)}
        &\le\frac{m^2}{16}\Paren{V(\mu',s)-V(\mu,s)}^2
        +\frac{64\eta^2}{m^2}\\
        &\le\frac14\iprod{\mu'-\mu,s}^2+\frac{64\eta^2}{m^2}.
      \end{aligned}
    }
  \end{equation*}
  Substituting into \eqref{eq:intro-directional-identifiability} leaves $(1/4+4\eta)\iprod{\mu'-\mu,s}^2$ on the right. For sufficiently small $\eta$ this term is absorbed, while $64\eta^2/m^2=O(\eta^2\log(1/\eta)/m)$. Hence
  \begin{equation*}
    \mathcal A_\eta\sststile{6}{}
    \Set{
      \iprod{\mu'-\mu,s}^2
      \le O\Paren{\frac{\eta^2\log(1/\eta)}{m}}
    }.
  \end{equation*}

  Finally, pseudoexpectation Cauchy--Schwarz and total-variation duality turn this directional certificate into $\normtv{\hat{\mu}-\mu}\le O(\eta\smash[t]{\sqrt{\log(1/\eta)/m}})$. Here $s$ is chosen, only in the analysis, to indicate the coordinates on which $\hat\mu-\mu$ is positive.

  \paragraph{From the robust certificate to a private score.}
  Clean data admit a polynomial-size certificate, and every certificate with a small discard budget is accurate. For privacy, we retain these two facts and allow the discard budget $T$ to vary. The score of a candidate $\nu$ is the smallest $T$ for which $\mathcal A_\eta$, with the count constraint replaced by $\sum_iw_i\ge n-T$, admits a degree-$6$ pseudoexpectation whose pseudo-mean lies within $\ell_1$-distance $\alpha$ of $\nu$. Thus a low score says that the data admit a covariance-aware certificate near $\nu$ after only a small number of deletions.

  \begin{samepage}
    The score satisfies the conditions of the robustness-to-privacy reduction: changing one batch costs one additional discard, so sensitivity is one; combining certificates gives quasi-convexity; and the polynomial-size semidefinite formulation gives efficient computation. One further ingredient is needed: the low-score sets must have controlled volume.

    \paragraph{$\ell_1$ geometry of low-score sets.}
    To compare the volumes of low-score sets, we use an inner and an outer ball. An exact clean certificate certifies the full $\ell_1$-ball of radius $\alpha$ around its mean. For a sufficiently small absolute constant $\eta^*$, robust soundness places all points of score at most $\eta^*n$ inside the concentric ball of radius $\alpha+O(1/\sqrt m)$ (\Cref{lem:vol}). Since the volume of an $\ell_1$-ball scales as the $k$th power of its radius, these two balls bound the ratio used by the pure-DP reduction of Hopkins, Kamath, Majid, and Narayanan~\cite{HopkinsKMN23}. Part~3 of \Cref{lem:vol} gives a sharper radius bound at each score level.

    Approximate DP requires the analogous comparison at score levels $0.8\eta^*n$ and $\eta^*n$ for an arbitrary dataset, where there may be no true mean to use as a center. We obtain the needed outer ball by adapting the one-dimensional obstruction argument of~\cite{HopkinsKMN23} to sign directions. We use signs because of exact $\ell_1$ duality:
    \begin{equation*}
      \normo{x}=\max_{\sigma\in\Set{\pm1}^k}\iprod{x,\sigma}.
    \end{equation*}
    If the pseudo-means of two certificates with discard budget at most $\eta^*n$ were farther than $O(1/\sqrt m)$ in $\ell_1$, some sign vector would separate them by more than $O(1/\sqrt m)$, contradicting their directional variance bounds. This gives the required $\ell_1$-diameter bound directly, without an $\ell_2$-to-$\ell_1$ conversion and its extra $\smash[t]{\sqrt{k}}$ loss. Adding the $\alpha$-radius candidate neighborhoods gives an intermediate log-volume ratio of $O(k\log(1+C/(\alpha\sqrt m)))$. Keeping the inner radius close to $\alpha$, as in \Cref{lem:approx_dp_vol_lb}, is needed when $\alpha\sqrt m$ is large. The clean-data certificate still requires $n=\smash[t]{\widetilde\Omega}(k)$ in this regime, which accounts for the extra user term in the approximate-DP theorem. We prove the volume comparison in \Cref{cor:approx_dp_vol_ratio}, and give the diameter argument in \Cref{app:approx-dp}.

  \end{samepage}
\end{sodarevision}

\subsection{Related Work}

Private learning discrete distributions from batches was introduced by Liu et al.~\cite{LiuSYKR20}, and was subsequently studied by others~\cite{NarayananME22,GhaziKKMMZ23b}. This problem also fits into the broader literature on user-level privacy~\cite{LevySAKKMS21}, which studies adjacent tasks including mean estimation~\cite{GirgisDD22,NarayananME22,CummingsFMT22,GeorgeRST24,ZhaoLSLWL24,AgarwalKMMSU25}, PAC learning~\cite{GhaziKM21}, and convex optimization~\cite{GhaziKKMMZ23a,BassilyS23,LiuA24,GhaziKLM25}. There is likewise an extensive literature on private estimation in the item-level setting, corresponding to $m=1$; see~\cite{KamathU20} for a partial survey. In that setting, estimating a discrete distribution over $[k]$ under item-level $(\varepsilon,0)$-differential privacy is straightforward: applying the Laplace mechanism to the empirical distribution achieves the optimal sample complexity~\cite{DiakonikolasHS15}. This estimator, followed by $\ell_1$ projection onto the simplex, tolerates replacement of an $\eta$-fraction of the items: with constant probability its total-variation error is $O(\eta+\smash[t]{\sqrt{k/n}}+k/(n\eps))$. Indeed, the replacements change the empirical distribution by at most $\eta$ in total variation, and projection increases its distance from the true distribution by at most a factor of two.

Qiao and Valiant~\cite{QiaoV18} introduced the batch-robust setting, which was subsequently studied by several groups~\cite{JainO20a,ChenLM20a}, as well as in related settings including structured distributions~\cite{JainO20b,JainO21,ChenLM20b}, linear regression~\cite{DasJKS23,JainSKDO24,DiakonikolasKKLP25}, and local differential privacy~\cite{ChhorS23,AcharyaLS23}. Our work, together with many of these results, also fits into the broader line of algorithmic robust statistics, which began with work on robust mean estimation~\cite{DiakonikolasKKLMS16,LaiRV16} and has since expanded substantially; see~\cite{DiakonikolasK22} for a textbook treatment. As in the private setting, the interesting algorithmic issues for learning discrete distributions arise because of the batch structure: when $m=1$, the information-theoretic optimum is already achieved by the contaminated empirical distribution.

Our private algorithm uses the robustness-to-privacy framework introduced by Hopkins, Kamath, Majid, and Narayanan~\cite{HopkinsKMN23} and independently by Asi, Ullman, and Zakynthinou~\cite{AsiUZ23}, which converts robust estimators into private ones. Efficient versions of this framework have already yielded near-optimal private and robust algorithms for Gaussian estimation~\cite{HopkinsKMN23}, parameter estimation in random graphs~\cite{ChenDHS24}, and linear regression~\cite{AndersonBMT25}. Related sum-of-squares exponential mechanisms also give private algorithms for graphon estimation~\cite{ChenDDHLS24}. Robustness-to-privacy methods also underlie private Bayesian estimation~\cite{ChenDMM26} and improved sample bounds for privately learning Gaussian mixtures~\cite{AshtianiMN24}. To our knowledge, this is the first polynomial-time application of this framework that converts batch-level robustness into user-level privacy for learning discrete distributions. Georgiev and Hopkins~\cite{GeorgievH22} prove a partial converse, showing settings in which privacy automatically yields robustness.

Starting from the original work of Dwork and Lei~\cite{DworkL09}, there has been substantial work on connections between robustness and privacy~\cite{BunKSW19,KamathSU20,LiuKO22,BrownGSUZ21,RamsayC21,KothariMV22,HopkinsKM22,AlabiKTVZ23}. Some of these papers give alternative general methods for converting robust estimators into private ones~\cite{LiuKO22,KothariMV22}. Other work from the statistics community studies these connections from a different perspective~\cite{AvellaMedina20,AvellaMedina21,RamsayJC22,SlavkovicM22}, including in the local model~\cite{LiBY22}. A recent line of work studies estimators that are simultaneously private and robust for high-dimensional problems~\cite{BunKSW19,GhaziKMN21,LiuKKO21,NarayananME22,AshtianiL22,KothariMV22,TsfadiaCKMS22,LiuKO22,LiuJKOS23}. See the survey of Kamath~\cite{Kamath24} for a broader discussion of the robustness--privacy interface.

\paragraph{Paper roadmap.}
\begin{sodarevision}
  \Cref{sec:prelim} shows how subset control produces a diagonal majorizer. \Cref{sec:robust-algorithm} uses the majorizer in both directions---first for clean-data feasibility and then for degree-$6$ SoS identifiability. \Cref{sec:private-algorithm} turns this certificate into a score function and proves the categorical private theorems. \Cref{sec:sq-lower-bound} proves the SQ lower bound for robust learning from discrete batches. The appendices give the approximate-DP $\ell_1$ obstruction and the matching pure-DP lower bound.
\end{sodarevision}

  \section{Preliminaries}
\label{sec:prelim}

\begin{problem}[Robust Learning from Discrete Batches]
  \label{prob:robust-learning}
  Let $\mu \in \Delta^{k-1}$ be an unknown discrete distribution over $[k]$.
  Let $m \in \N$ be the batch size, and $\eta$ be the corruption fraction.
  Let $\set{B_i^*}_{i \in [n]}$ be $n$ batches of samples, where each clean batch
  \begin{equation*}
    B_i^* = \set{U_{i,1}^*, \ldots, U_{i,m}^*}
  \end{equation*}
  and each $U_{i,j}^* \sim \mu$ is drawn independently from $\mu$.
  An $\eta$-fraction of these batches are then adversarially corrupted to form the observed
  batches $\set{B_i}_{i \in [n]}$, meaning that there exists a set
  $S \subseteq [n]$ with $|S| \leq \eta n$ such that $B_i = B_i^*$ for every $i \notin S$.
  We seek an algorithm that, given the observed batches $\set{B_i}_{i \in [n]}$, outputs a distribution $\hat{\mu} \in \Delta^{k-1}$ such that with high probability,
  \begin{equation*}
    \normtv{\hat{\mu} - \mu} \leq \alpha \mcom
  \end{equation*}
  where $\alpha \ge \eta\smash[t]{\sqrt{\log(1/\eta)/m}}$.
\end{problem}

\subsection{Data Model and Notation}
\label{subsec:data-model-notation}

We use calligraphic letters such as $\mathcal{B}$ and $\mathcal{Y}$ for datasets. A star marks a latent clean batch or average, and a prime marks a repaired point or candidate mean.

For a batch $B_i = \set{U_{i,1},\ldots,U_{i,m}}$ of symbols in $[k]$, define its empirical average vector by
\begin{equation*}
  y_i \coloneq \frac{1}{m}\sum_{j=1}^m e_{U_{i,j}} \in \Delta^{k-1},
\end{equation*}
where $e_1,\dots,e_k$ are the standard basis vectors in $\R^k$.
Likewise, each clean batch $B_i^*$ induces a clean batch-average vector
\begin{equation*}
  y_i^* \coloneq \frac{1}{m}\sum_{j=1}^m e_{U_{i,j}^*} \in \Delta^{k-1}.
\end{equation*}
We write
\begin{equation*}
  \mathcal{B}^* = \set{B_i^*}_{i \in [n]},
  \qquad
  \mathcal{B} = \set{B_i}_{i \in [n]},
  \qquad
  \mathcal{Y}^* = \set{y_i^*}_{i \in [n]},
  \qquad
  \mathcal{Y} = \set{y_i}_{i \in [n]} \mper
\end{equation*}
From now on, we work with the averaged vectors $\mathcal{Y}^*$ and $\mathcal{Y}$. The robust program uses repaired points $y_i'$ and their mean $\mu'$; $\hat\mu$ denotes the output estimate, and $\nu$ denotes a candidate for the private score.

\subsection{Probability Simplex}
We will repeatedly make use of the probability simplex of distributions over $[k]$.

\begin{definition}[Probability
  Simplex $\Delta^{k-1}$]
  \label{def:probability-simplex}
  Let $\Delta^{k-1}$ denote
  \begin{equation*}
    \Delta^{k-1} \coloneq \Set{x \suchthat \iprod{\bm{1}, x} = 1, x \in \R^{k}_{\ge 0}} \mper
  \end{equation*}
\end{definition}

\begin{remark}[Total Variation Convention]
  \label{rem:tv-convention}
  Whenever $x \in \R^k$ satisfies $\iprod{\bm{1},x}=0$, we write
  \begin{equation*}
    \normtv{x} \coloneq \frac12 \normo{x} = \max_{s \in \Set{0,1}^k} \iprod{x,s} \mper
  \end{equation*}
  In particular, for probability vectors $\mu,\nu \in \Delta^{k-1}$, $\normtv{\mu-\nu}$ denotes the usual total variation distance.
\end{remark}

\begin{fact}[Diagonal minus rank-one is PSD on the simplex]\label{fact:diag-minus-xxT-psd}
  Let $x\in\R_{\ge 0}^k$ with $s:=\sum_{i=1}^k x_i\le 1$ (in particular, for $x\in\Delta^{k-1}$ one has $s=1$). Define
  \[
    M(x)\;:=\;\diag(x)\;-\;x x^\top \;\in\R^{k\times k}.
  \]
  Then $M(x)\succeq 0$. Moreover, on $\supp(x)$,
  \[
    M(x)\ \succeq\ (1-s)\,\diag(x).
  \]
\end{fact}

\begin{proof}
  Let $D:=\diag(x)$ and $y:=D^{1/2}\bm{1}$. Then $M(x)$ factors as
  \[
    M(x) \;=\; D^{1/2}\bigl(I-yy^\top\bigr)D^{1/2},\qquad \|y\|^2=\bm{1}^\top D\,\bm{1}=s\le 1.
  \]
  Here $I-yy^\top$ has eigenvalues $1$ (with multiplicity $k-1$ on $y^\perp$) and $1-\|y\|^2=1-s\ge 0$ (in the direction of $y$). Hence $I-yy^\top\succeq 0$, and conjugation by $D^{1/2}$ preserves positive semidefiniteness, yielding $M(x)\succeq 0$. To prove $M(x)\succeq(1-s)\diag(x)$, observe that
  \[
    u^\top M(x)u \;=\; \|D^{1/2}u\|^2-\langle D^{1/2}u, y\rangle^2
    \;\ge\; (1-\|y\|^2)\,\|D^{1/2}u\|^2
    \;=\; (1-s)\,u^\top D\,u,
  \]
  by Cauchy–Schwarz. In the special case $s=1$ (the standard simplex), $I-yy^\top$ has a simple zero eigenvalue, so $M(x)$ has rank at most $k-1$ with $\bm{1}$ in its nullspace; equivalently
  \[
    u^\top M(x)u \;=\; \tfrac12\sum_{i,j} x_i x_j\,(u_i-u_j)^2 \;\ge 0 \mper
  \]
\end{proof}

\begin{fact}[Mean and Covariance of a Clean Batch Average]
  \label{fact:batch-average-mean-cov}
  Let $\mu \in \Delta^{k-1}$ be a discrete distribution over $[k]$, let $m \in \N$, let
  \begin{equation*}
    U_1,\dots,U_m \stackrel{\mathrm{iid}}{\sim} \mu,
    \qquad
    Y \coloneq \frac1m \sum_{\ell=1}^m e_{U_\ell} \in \Delta^{k-1},
  \end{equation*}
  and write
  \begin{equation*}
    \Sigma(\mu) \coloneq \frac{1}{m} \Paren{\diag(\mu) - \mu \transpose{\mu}} \mper
  \end{equation*}
  Then
  \begin{align*}
    \E[Y] &= \mu \mcom \\
    \E[Y\transpose{Y}] - \E[Y]\transpose{\E[Y]} &= \Sigma(\mu) \mper
  \end{align*}
\end{fact}
\begin{proof}
  Linearity of expectation gives the mean.
  For the second moment, we have
  \begin{align*}
    \E[Y\transpose{Y}] &= \E\Brac{\frac{1}{m^2} \sum_{\ell,t \in [m]} e_{U_\ell} \transpose{e_{U_t}}} \\
    &= \frac{1}{m^2} \Paren{\sum_{\ell \in [m]} \E[e_{U_\ell} \transpose{e_{U_\ell}}] + \sum_{\ell \neq t} \E[e_{U_\ell}] \transpose{\E[e_{U_t}]}} \\
    &= \frac{1}{m^2} \Paren{m \diag(\mu) + m(m-1) \mu \transpose{\mu}} \mper
  \end{align*}
  Thus,
  \begin{align*}
    \E[Y\transpose{Y}] - \E[Y]\transpose{\E[Y]} &= \frac{1}{m^2} \Paren{m \diag(\mu) + m(m-1) \mu \transpose{\mu}} - \mu \transpose{\mu} \\
    &= \frac{1}{m} \Paren{\diag(\mu) - \mu \transpose{\mu}} \mper
  \end{align*}
\end{proof}
We will also repeatedly use the variance of a clean batch average along a direction $s \in \set{0,1}^k$, whose formula is immediate from \Cref{fact:batch-average-mean-cov}.
\begin{definition}[Subset-Direction Mean and Variance]
  \label{def:variance}
  Let $\mu \in \Delta^{k-1}$, and let $s$ be a $k$-dimensional vector in $\set{0,1}^k$. We define the mean and variance of a clean batch average along direction $s$ by
  \begin{align*}
    \mu(s) &\coloneq \iprod{s, \mu} \mcom \\
    V(\mu, s) &\coloneq \frac{\iprod{s, \mu} (1 - \iprod{s, \mu})}{m} \mper
  \end{align*}
\end{definition}

\subsection{Concentration Bounds}
We use a version of the clean-sample condition in \cite{JainO20a} with an explicit failure probability.

\begin{fact}[cf.\ Lemma~3 in \cite{JainO20a}]
  \label{fact:main-resilience-lemma}
  Let $\mu\in\Delta^{k-1}$, $\eta,\beta\in(0,1/2)$, and $m\in\N$. Let
  \begin{equation*}
    B_i^* = (U_{i,1}^*,\ldots,U_{i,m}^*),
    \qquad
    U_{i,1}^*,\ldots,U_{i,m}^* \stackrel{\mathrm{iid}}{\sim} \mu,
    \qquad
    y_i^* \coloneq \frac1m \sum_{j=1}^m e_{U_{i,j}^*}
  \end{equation*}
  for each $i \in [n]$.
  Then as long as
  $$n \ge \Omega\Paren{\frac{k + \log(1/\beta)}{\eta^2\log(1/\eta)}} \mcom $$
  with probability $1-\beta$, for all subsets of indices $T \subseteq [n]$ with $\abs{T} \ge (1-\eta) n$, and all vectors $s \in \Set{0,1}^k$, we have
  \begin{equation*}
    \Abs{\frac{1}{\abs{T}} \sum_{i \in T} \iprod{y_i^*, s} - \iprod{\mu, s}} \le O\Paren{\eta\sqrt{\log (1/\eta)/m}}\mcom
  \end{equation*}
  and
  \begin{equation*}
    \Abs{\frac{1}{\abs{T}} \sum_{i \in T} \iprod{s, y_i^* - \mu}^2 - V(\mu, s)} \leq O\Paren{\eta \log(1/\eta) / m} \mper
  \end{equation*}
\end{fact}
\begin{proof}
  Write $\xi_i(s)=\iprod{y_i^*-\mu,s}$.
  For each fixed $s$, Hoeffding's inequality gives
  $\Pr[|\xi_i(s)|>t]\le 2\exp(-2mt^2)$.
  Thus $\xi_i(s)^2$ has subexponential norm $O(1/m)$. Applying the corresponding mean and second-moment concentration bounds and taking a union bound over the $2^k$ choices of $s$ gives, with probability at least $1-\beta/2$,
  \begin{align*}
    \Abs{\E_i\xi_i(s)}&\le O\Paren{\sqrt{\frac{k+\log(2/\beta)}{nm}}},\\
    \Abs{\E_i\xi_i(s)^2-V(\mu,s)}
    &\le \frac{O(1)}m\Paren{\sqrt{\frac{k+\log(2/\beta)}n}+\frac{k+\log(2/\beta)}n}
  \end{align*}
  simultaneously for all $s$.
  The assumed sample size ensures $\eta n\ge1$.
  For a fixed set $U$ of size $\lceil\eta n\rceil$, the exponential-moment bound gives
  $\sum_{i\in U}\xi_i(s)^2\le O((\lceil\eta n\rceil+t)/m)$ except with probability $e^{-t}$.
  A union bound over $s$ and the $\binom n{\lceil\eta n\rceil}\le(en/\lceil\eta n\rceil)^{\lceil\eta n\rceil}$ such sets therefore gives, with probability at least $1-\beta/2$,
  \begin{equation*}
    \max_{|U|\le\eta n}\frac1n\sum_{i\in U}\xi_i(s)^2
    \le \frac{O(1)}m\Paren{\eta\log(1/\eta)+\frac{k+\log(2/\beta)}n}
    \le O\Paren{\frac{\eta\log(1/\eta)}m}
  \end{equation*}
  for every $s$.
  For $T=[n]\setminus U$, Cauchy--Schwarz bounds
  $|n^{-1}\sum_{i\in U}\xi_i(s)|$ by
  $\smash[t]{\sqrt{(|U|/n)(n^{-1}\sum_{i\in U}\xi_i(s)^2)}}$.
  Subtracting the contribution of $U$ from the full mean and second moment, and dividing by $|T|/n\ge1/2$, proves the mean and second-moment bounds for $T$. For the second moment, the normalization contributes at most $2\eta V(\mu,s)\le\eta/(2m)$.
\end{proof}
\subsection{Sum-of-Squares}

\begin{fact}[Triangle inequality]
  \label{fact:sos-triangle-inequality}
  Let $a$ and $b$ be indeterminates, then
  \begin{equation*}
    \sststile{2}{a, b} \Set{(a+b)^2 \le 2a^2 + 2b^2} \mper
  \end{equation*}
\end{fact}
\begin{sodarevision}
  \begin{fact}[Young's inequality]
    \label{fact:sos-young-inequality}
    Let $a$ and $b$ be indeterminates and let $\lambda>0$ be a fixed scalar. Then
    \begin{equation*}
      \sststile{2}{a,b}
      \Set{ab\le \lambda a^2+\frac{b^2}{4\lambda}} \mper
    \end{equation*}
  \end{fact}
  \begin{proof}
    We can write
    \begin{equation*}
      \sststile{2}{a,b}
      \Set{
        \begin{aligned}
          \lambda a^2+\frac{b^2}{4\lambda}-ab
          &=\Paren{\sqrt{\lambda}a-\frac{b}{2\sqrt{\lambda}}}^2 \\
          &\ge 0
        \end{aligned}
      } \mcom
    \end{equation*}
    as desired.
  \end{proof}
\end{sodarevision}
\begin{lemma}[SoS Selector Lemma \cite{AndersonBMT25}]
  \label{lem:sos-selector-lemma}
  Let $B, a_1, \dots, a_n \in \R$ and $k \in \mathbb{N}$. Suppose $\sum_{i \in T} a_i \le B$ for every $T \subseteq [n]$ with $\abs{T} \le k$. Then
  \begin{equation*}
    \Set{\forall i \in [n]: 0 \le b_i \le 1, \sum_{i=1}^n b_i = k} \sststile{2}{b_1, \dots, b_n} \Set{\sum_{i=1}^n b_i a_i \le B} \mper
  \end{equation*}
  Additionally, if the $a_i$'s are non-negative, the constraint $\sum_{i \in [n]} b_i = k$ may be replaced by $\sum_{i \in [n]} b_i \le k$, retaining $0\le b_i\le1$ for every $i$.
  More generally, let $t\in[0,n]$ be real, let $a_i\ge0$, and suppose $\sum_{i\in T}a_i\le B$ whenever $|T|\le t$. Then
  \begin{equation*}
    \Set{\forall i:\ b_i^2=b_i,\ \sum_i b_i\le t}
    \sststile{4}{}\Set{\sum_i a_i b_i\le 2B}.
  \end{equation*}
\end{lemma}
\begin{proof}[Proof of the real-budget extension]
  For $t\ge1$, order the coefficients so that $a_1\ge\cdots\ge a_n$. By the Boolean constraints, $0\le b_i\le1$, and hence
  \begin{align*}
    \sum_i a_i b_i
    &\le a_{\lfloor t\rfloor}\sum_i b_i+\sum_{i\le\lfloor t\rfloor}(a_i-a_{\lfloor t\rfloor}) \\
    &\le \sum_{i\le\lfloor t\rfloor}a_i+(t-\lfloor t\rfloor)a_{\lfloor t\rfloor} \\
    &\le \frac{t}{\lfloor t\rfloor}\sum_{i\le\lfloor t\rfloor}a_i
    \le2B.
  \end{align*}
  These are degree-$2$ SoS inequalities. If $t<1$, the Boolean identities give the degree-$4$ certificate
  \begin{equation*}
    (t-1)b_i=b_i^2\Paren{t-\sum_j b_j}+\sum_{j\ne i}(b_i b_j)^2\ge0.
  \end{equation*}
  Thus $b_i\le0$, and $\sum_i a_i b_i\le0\le2B$.
\end{proof}

\begin{definition}[Approximately Satisfying Linear Functionals]
  \label{def:approx-sos-satisfaction}
  Let $n \in \mathbb N$, $T \in [0, n]$.
  Let $\mathcal{A}_T$ be a system of polynomial constraints parameterized by $T$ in indeterminates $w = (w_1,\dots, w_n)$ and $z$ (of arbitrary size) such that
  \begin{equation*}
    \cA_T
    =
    \Set{\forall i \in [m_{\ge}]: q_i(z, w) \ge 0}
    \cup
    \Set{\forall j \in [m_{\mathrm{eq}}]: r_j(z, w) = 0}
    \cup
    \Set{\sum_{i \in [n]} w_i \ge n- T} \mper
  \end{equation*}
  Let $\cR$ denote the operator that maps a polynomial to a vector representing its coefficients.
  Let $\tau \ge 0$. We say that a linear functional $\cL$ $\tau$-approximately satisfies $\cA_T$ at degree $D$ if
  \begin{enumerate}
    \item $\cL 1 = 1$.
    \item For all degree at most $D/2$ polynomials $p$, such that $\normt{\cR(p)}\le 1$, we have that $\cL p^2 \ge -\tau T$.
    \item For all $i \in [m_{\ge}]$ and polynomials $p$ such that $\deg(p^2 q_i) \le D$, and $\normt{\cR(p)}\le 1$, we have
      \begin{equation*}
        \cL p^2 q_i \ge -\tau T \mper
      \end{equation*}
    \item For all $j \in [m_{\mathrm{eq}}]$ and polynomials $p$ such that $\deg(p^2 r_j) \le D$, and $\normt{\cR(p)}\le 1$, we have
      \begin{equation*}
        \Abs{\cL p^2 r_j} \le \tau T \mper
      \end{equation*}
    \item For all polynomials $p$ such that $\deg(p^2(\sum_{i} w_i -n + T)) \le D$, and $\normt{\cR(p)} \le 1$, we have that $\cL \Brac{ p^2(\sum_i w_i - n +T)} \ge -5 \tau T n$.
  \end{enumerate}
\end{definition}

\subsection{Subset-Variance Control and the Diagonal-Majorizer SDP}
\label{subsec:grothendieck-duality-bridge}

We isolate the abstract bridge from subset-direction control of a covariance-gap matrix to a trace-bounded diagonal majorizer. Let $\Delta \in \R^{k\times k}$ be symmetric. We consider the Boolean quadratic $s^\top \Delta s$ over subset directions $s \in \Set{0,1}^k$.

\begin{lemma}[Subset-to-sign reduction under a nullspace condition]
  \label{lem:subset-to-sign-nullspace}
  Let $\Delta \in \R^{k\times k}$ be symmetric and satisfy $\Delta \bm{1} = 0$. Then for every $s \in \Set{0,1}^k$ and $\sigma = 2s-\bm{1}$,
  \begin{equation*}
    \Abs{\sigma^\top \Delta \sigma} = 4 \Abs{s^\top \Delta s} \mper
  \end{equation*}
\end{lemma}
\begin{proof}
  \begin{equation*}
    \sigma^\top \Delta \sigma
    =
    (2s-\bm{1})^\top \Delta (2s-\bm{1})
    =
    4 s^\top \Delta s
  \end{equation*}
  because $\Delta \bm{1} = 0$.
\end{proof}

\begin{definition}[Elliptope]
  \label{def:elliptope}
  Define
  \begin{equation*}
    \mathcal E_k \coloneq \Set{X \in \R^{k \times k} \suchthat X \succcurlyeq 0,\ \diag(X)=\bm{1}} \mper
  \end{equation*}
  Each rank-one matrix $\sigma \sigma^\top$ with $\sigma \in \Set{\pm1}^k$ belongs to $\mathcal E_k$.
\end{definition}

\begin{theorem}[Symmetric Grothendieck inequality, real $\gamma$-case of Friedland--Lim~\cite{FriedlandL20}]
  \label{thm:friedland-lim-symmetric-grothendieck}
  Let $A = (a_{ij}) \in \R^{k \times k}$ be symmetric, and let $d \ge k$.
  Then
  \begin{equation*}
    \max_{\substack{x_1,\dots,x_k \in \R^d \\ \normt{x_i}=1\ \forall i}}
    \Abs{\sum_{i,j=1}^k a_{ij}\iprod{x_i,x_j}}
    \le
    \sinh(\pi/2)
    \max_{\sigma \in \Set{\pm1}^k}
    \Abs{\sum_{i,j=1}^k a_{ij}\sigma_i\sigma_j} \mper
  \end{equation*}
  Equivalently, this is the real symmetric $\gamma$-seminorm inequality from Friedland--Lim, Theorem~4.3, equation~(55), together with the bound on the constant from equation~(57).
\end{theorem}

\begin{corollary}[Correlation-matrix form]
  \label{cor:friedland-lim-correlation}
  Let $A \in \R^{k \times k}$ be symmetric.
  Then
  \begin{equation*}
    \max_{X \succcurlyeq 0,\ \diag(X)=\bm{1}} \Abs{\iprod{A,X}}
    \le
    \sinh(\pi/2)
    \max_{\sigma \in \Set{\pm1}^k} \Abs{\sigma^\top A\sigma} \mper
  \end{equation*}
\end{corollary}
\begin{proof}
  Every feasible $X$ is a correlation matrix, hence a Gram matrix: there exist unit vectors $x_1,\dots,x_k \in \R^r$ for some $r \le k$ such that
  \begin{equation*}
    X_{ij} = \iprod{x_i,x_j}
    \qquad\text{for all } i,j \in [k] \mper
  \end{equation*}
  Therefore
  \begin{equation*}
    \iprod{A,X} = \sum_{i,j=1}^k a_{ij} X_{ij} = \sum_{i,j=1}^k a_{ij}\iprod{x_i,x_j},
  \end{equation*}
  and the claim follows from \Cref{thm:friedland-lim-symmetric-grothendieck}.
\end{proof}

We use the real symmetric indefinite absolute-value form from Friedland--Lim because $\Delta$ is not assumed positive semidefinite.

\begin{fact}[Semidefinite duality with equality constraints]
  \label{fact:sdp-duality-equality-form}
  Let $C,A_1,\dots,A_m \in \R^{d\times d}_{\mathrm{sym}}$ and let $b \in \R^m$. Consider the primal--dual pair
  \begin{align*}
    \textnormal{(P)}\qquad
    \max_{X \succcurlyeq 0}\ &\iprod{C,X}
    \qquad \textnormal{such that } \iprod{A_t,X}=b_t \textnormal{ for all } t\in[m],\\
    \textnormal{(D)}\qquad
    \min_{y \in \R^m}\ &\sum_{t=1}^m b_t y_t
    \qquad \textnormal{such that } \sum_{t=1}^m y_t A_t - C \succcurlyeq 0.
  \end{align*}
  If both programs are strictly feasible, then strong duality holds and the two optimal values are equal.
\end{fact}

\begin{corollary}[Diagonal-majorizer SDP duality]
  \label{cor:diagonal-majorizer-sdp-duality}
  Let $\Delta \in \R^{k\times k}$ be symmetric. Then
  \begin{equation*}
    \max_{X \succcurlyeq 0,\ \diag(X)=\bm{1}} \iprod{\Delta,X}
    =
    \min_{\lambda \in \R^k}
    \Set{
      \sum_{j \in [k]} \lambda_j
      \suchthat
      \operatorname{Diag}(\lambda)-\Delta \succcurlyeq 0
    } \mper
  \end{equation*}
  Moreover, both programs are strictly feasible.
\end{corollary}
\begin{proof}
  Apply \Cref{fact:sdp-duality-equality-form} with $A_j = e_j e_j^\top$, $b_j=1$, and $C=\Delta$. Strict primal feasibility is witnessed by $X = I_k$, and strict dual feasibility is witnessed by $\lambda = C\bm{1}$ for any sufficiently large scalar $C$.
\end{proof}

\begin{proposition}[Subset Control Implies Diagonal-Majorizer Feasibility]
  \label{prop:subset-control-majorizer}
  Let $\Delta \in \R^{k \times k}$ be symmetric.
  Assume that
  \begin{equation*}
    \Delta \bm{1} = 0,
    \qquad
    \sup_{s \in \Set{0,1}^k} \Abs{s^\top \Delta s} \le \rho \mper
  \end{equation*}
  Then there exists $\lambda \in \R^k$ such that
  \begin{equation*}
    \operatorname{Diag}(\lambda) - \Delta \succcurlyeq 0,
    \qquad
    \sum_{j \in [k]} \lambda_j \le O(\rho) \mper
  \end{equation*}
\end{proposition}
\begin{proof}
  By \Cref{lem:subset-to-sign-nullspace},
  \begin{equation*}
    \max_{\sigma \in \Set{\pm1}^k} \Abs{\sigma^\top \Delta \sigma}
    \le 4\rho \mper
  \end{equation*}
  By \Cref{cor:friedland-lim-correlation},
  \begin{equation*}
    \max_{X \succcurlyeq 0,\, \diag(X)=\bm{1}} \iprod{\Delta, X}
    \le
    \max_{X \succcurlyeq 0,\, \diag(X)=\bm{1}} \Abs{\iprod{\Delta, X}}
    \le
    \sinh(\pi/2)
    \max_{\sigma \in \Set{\pm1}^k} \Abs{\sigma^\top \Delta \sigma}
    \le O(\rho) \mper
  \end{equation*}
  Applying \Cref{cor:diagonal-majorizer-sdp-duality} yields the claim.
\end{proof}

\subsection{Robustness to Privacy}

We apply the following theorem in order to transform robust algorithms into private ones.

\begin{theorem}[\cite{HopkinsKMN23}]
  \label{thm:pure_dp_reduction}
  Let $\eps > 0$, $\beta \in (0,1)$, and $0 < \eta, r < 1 < R$ be fixed parameters.
  Suppose we have a score function $\cS(\theta, \mathcal{Y}) \in [0, n]$ that takes as input a dataset $\mathcal{Y} = \{y_1, \dots, y_n\}$ and a parameter $\theta \in \Theta \subset \mathbb{B}(R)^d$ (where $\Theta$ is convex and contained in a ball of radius $R$), with the following properties:
  \begin{itemize}
    \item (Bounded Sensitivity) For any two adjacent datasets $\mathcal{Y}, \mathcal{Y}'$ and any $\theta \in \Theta$, $|\cS(\theta, \mathcal{Y})-\cS(\theta, \mathcal{Y}')| \le 1.$
    \item (Quasi-Convexity) For any fixed dataset $\mathcal{Y}$, any $\theta, \theta' \in \Theta$, and any $0 \le \lambda \le 1$, we have that $\cS(\lambda \theta + (1-\lambda) \theta', \mathcal{Y}) \le \max(\cS(\theta, \mathcal{Y}), \cS(\theta', \mathcal{Y}))$.
    \item (Efficiently Computable) For any given $\theta \in \Theta$ and dataset $\mathcal{Y}$, we can compute $\cS(\theta, \mathcal{Y})$ up to error $\gamma$ in $\poly(n, d, \log \frac{R}{r}, \log \gamma^{-1})$ time for any $\gamma > 0$.
    \item (Robust algorithm finds low-scoring point) For a given dataset $\mathcal{Y}$, let $T = \min_{\theta_0 \in \Theta} \cS(\theta_0, \mathcal{Y})$. Then, we can find some point $\theta$ such that for all $\theta'$ within distance $r$ of $\theta$, $\cS(\theta', \mathcal{Y}) \le T+1$, in time $\poly(n, d, \log \frac{R}{r})$.
    \item (Volume) For any given dataset $\mathcal{Y}$ and $\eta' \ge \eta$, let $V_{\eta'}(\mathcal{Y})$ represent the $d$-dimensional volume of points $\theta \in \Theta \subset \R^d$ with score at most $\eta' n$. (Note that $V_1(\mathcal{Y})$ is the full volume of $\Theta$).
  \end{itemize}
  Then,
  we have a pure $\eps$-DP algorithm $\mathcal{A}$ on datasets of size $n$, that runs in $\poly(n, d, \log \frac{R}{r})$ time, with the following property. For any dataset $\mathcal{Y}$, if there exists $\theta$ with $\cS(\theta, \mathcal{Y}) \le \eta n$ and if
  \begin{equation*}
    n \ge \Omega\left(\max_{\eta': \eta \le \eta' \le 1}\frac{\log(V_{\eta'}(\mathcal{Y})/V_{\eta}(\mathcal{Y})) + \log (1/(\beta \cdot \eta))}{\eps \cdot \eta'}\right),
  \end{equation*}
  then $\mathcal{A}(\mathcal{Y})$ outputs some $\theta \in \Theta$ of score at most $2\eta n$ with probability $1-\beta$.
\end{theorem}

For context we also record the corresponding approximate-DP reduction theorem. Relative to \Cref{thm:pure_dp_reduction}, the only extra input is the intermediate volume-ratio condition at levels $0.8\eta^*$ and $\eta^*$.

\begin{theorem}[\cite{HopkinsKMN23}]
  \label{thm:approx-dp-reduction}
  Let $\eps > 0$, $\delta, \beta \in (0,1)$, $0 < \eta < 0.1$, and $0<r<1<R$ be fixed parameters. Suppose we have a score function $\cS(\theta, \cY) \in [0, \infty)$ that takes as input a dataset $\cY = \{y_1, \ldots , y_n\}$ and a parameter $\theta \in \Theta \subset \mathbb{R}^d$ (where $\Theta$ is convex and contained in a ball of radius $R$), satisfying the properties in~\cref{thm:pure_dp_reduction}. In addition, fix some parameter $\eta^* \in [10\eta, 1]$. Suppose that
  \begin{equation}
    \label{eq:approx-reduction-privacy}
    n \geq \Omega\Paren{\frac{\log (1/\delta) + \log(V_{\eta^*}(\cY)/V_{0.8\eta^*}(\cY))}{\eps \cdot \eta^*}}
  \end{equation}
  for all $\cY$ such that there exists $\theta$ with $\cS(\theta, \cY) \leq 0.7\eta^*n$.
  Then, there exists an $(\eps, \delta)$-DP algorithm $\cA$, that runs in $\poly(n, d, \log (\tfrac R r))$ time, such that for any dataset $\cY$, if there exists $\theta$ with $\cS(\theta, \cY) \leq \eta n$ and if
  \begin{equation}
    \label{eq:approx-reduction-utility}
    n \geq \Omega \left(\max_{\eta':\eta \leq \eta'\leq \eta^*} \frac{\log (V_{\eta'}(\cY)/V_{\eta}(\cY)) + \log (1/(\beta \cdot \eta))}{\eps \cdot \eta'}\right),
  \end{equation}
  then $\cA(\cY)$ outputs some $\theta \in \Theta$ satisfying $\cS(\theta, \cY) \leq 2\eta n$ with probability at least $1-\beta$.
\end{theorem}

\begin{lemma}[A refined utility bound]
  \label{lem:refined-private-utility}
  Suppose $0<\eps\le1$ and the score-$\eta n$ set contains a Euclidean ball of radius $r$. In the sample-size condition of \Cref{thm:pure_dp_reduction} and in \eqref{eq:approx-reduction-utility}, the term $\log(1/(\beta\eta))$ can be replaced by $\log(2/\beta)$. All other hypotheses remain required, including the worst-case volume condition \eqref{eq:approx-reduction-privacy} for approximate DP.
\end{lemma}
\begin{proof}
  We refine the tail calculation for the samplers of~\cite[full version, Lemmas~4.7 and~4.8]{HopkinsKMN23}. These samplers assign weights proportional to $\exp(-c\eps\cS)$ up to constant multiplicative factors; the approximate-DP sampler also truncates at score $0.9\eta^*n+O(1)$. Their grid-point counts agree with sublevel-set volumes up to constant factors, by~\cite[Fact~4.6]{HopkinsKMN23}. The inner-ball assumption supplies this comparison at level $\eta$.

  Each nonempty shell with score in $[2^j\eta n,2^{j+1}\eta n)$, $j\ge1$, has probability at most
  \begin{equation*}
    A\exp\Paren{-c\eps n(2^j-1)\eta}\,
    \begin{cases}
      \dfrac{V_{\min\{2^{j+1}\eta,1\}}(\cY)}{V_\eta(\cY)}, & \text{for pure DP},\\[6pt]
      \dfrac{V_{\min\{2^{j+1}\eta,\eta^*\}}(\cY)}{V_\eta(\cY)}, & \text{for approximate DP},
    \end{cases}
  \end{equation*}
  for absolute constants $A,c>0$. Additive $O(1)$ score errors affect only $A$, since $\eps\le1$. If
  \begin{equation}
    \label{eq:refined-utility-sample-size}
    n\ge
    \begin{cases}
      C\displaystyle\max_{\eta\le u\le1}
      \dfrac{\log(V_u(\cY)/V_\eta(\cY))+\log(2/\beta)}{\eps u}, & \text{for pure DP},\\[10pt]
      C\displaystyle\max_{\eta\le u\le\eta^*}
      \dfrac{\log(V_u(\cY)/V_\eta(\cY))+\log(2/\beta)}{\eps u}, & \text{for approximate DP},
    \end{cases}
  \end{equation}
  then the capped volume levels are at most $2^{j+1}\eta\le4(2^j-1)\eta$, which lets us absorb the volume ratio into half of the exponential decay by choosing $C$ sufficiently large. Taking $u=\eta$ in \eqref{eq:refined-utility-sample-size} also gives $\eps n\eta\ge C\log(2/\beta)$. Consequently the sum of the shell probabilities is at most
  \begin{equation*}
    A\sum_{j\ge1}\exp\Paren{-\tfrac c2\eps n(2^j-1)\eta}\le\beta/2.
  \end{equation*}
  For approximate DP, the reduction's test based on the minimum score accepts whenever that minimum is at most $\eta n$: indeed $\eta^*\ge10\eta$, and \eqref{eq:refined-utility-sample-size} gives $\eta n\ge C\log(2/\beta)/\eps$, which exceeds a sufficiently large constant. Its sampler has support below $\eta^*n$, since $0.9\eta^*n+O(1)<\eta^*n$. Set its numerical failure probability to at most $\min\{\beta/2,c_0\delta\}$ for a sufficiently small absolute constant $c_0>0$. Rescaling the privacy and failure parameters by absolute constants, as in the cited reductions, proves the claim.
\end{proof}

  \section{Robust Estimation for Discrete Batches}
\label{sec:robust-algorithm}

In this section, we design a robust estimation algorithm.
In \Cref{sec:private-algorithm} we will show how to use this robust algorithm (\Cref{thm:main-robust}) and \Cref{thm:pure_dp_reduction} to obtain a private algorithm.
For the technical arguments in this section, fix a sufficiently small absolute constant $\etaabs \in (0,1/20)$.

\begin{theorem}[Robust Estimation from Discrete Batches]
  \label{thm:main-robust}
  Let $\mu \in \Delta^{k-1}$, $\eta \in (0,\etaabs)$, $\beta \in (0,1)$, and $m \in \mathbb{N}$. Then there exists a polynomial-time algorithm\sodared{, obtained by solving a degree-$6$ SoS relaxation,} that takes in $n$ user batches of size $m$ forming an $\eta$-corruption of independent clean batches
  \begin{equation*}
    B_i^\star = (U_{i,1}^\star,\ldots,U_{i,m}^\star),
    \qquad
    U_{i,1}^\star,\ldots,U_{i,m}^\star \stackrel{\mathrm{iid}}{\sim} \mu,
  \end{equation*}
  and outputs an estimate $\hat{\mu}$ such that $\normtv{\hat{\mu} - \mu} \leq \alpha$, with probability at least $1 - \beta$, where $\alpha = O\paren{\eta\smash[t]{\sqrt{\log(1/\eta)/m}}}$, provided that
  \begin{equation*}
    n = \Omega\Paren{\frac{k + \log(1/\beta)}{\alpha^2 m}} \mper
  \end{equation*}
\end{theorem}

We will view the robust constraints through the variance discrepancy for a kept subset with empirical mean $\bar{\mu}_T$ and covariance $\smash[t]{\widehat{\Sigma}}_T$:
\begin{equation*}
  \langle ss^\top, \widehat{\Sigma}_T \rangle - \frac{1}{m}\Paren{\iprod{\bar{\mu}_T, s} - \iprod{\bar{\mu}_T, s}^2},
  \qquad
  s \in \Set{0,1}^k \mper
\end{equation*}
This is the object optimized by the SDP-based algorithms of Jain and Orlitsky~\cite{JainO20a} and Chen, Li, and Moitra~\cite{ChenLM20b}. Our succinct SoS certificate will be the dual object: a trace-bounded diagonal majorizer.

We isolate the deterministic properties of the clean samples needed in the robust analysis.
\begin{definition}[$\eta$-Good Samples]
  \label{def:eta-good}
  Let $\mu \in \Delta^{k-1}$, $\eta \in (0,1/2)$, and $m \in \mathbb{N}$.
  Let $y_1^*, \dots, y_n^* \in \R^k$ be any points.
  For $s \in \Set{0,1}^k$, write
  \begin{equation*}
    \xi_i(s) \coloneq \iprod{y_i^* - \mu, s} \mper
  \end{equation*}
  We say that $y_1^*, \dots, y_n^*$ are \emph{$\eta$-good} if for every $s \in \Set{0,1}^k$ we simultaneously have
  \begin{enumerate}
    \item $\abs{\E_i \xi_i(s)} \le O(\eta\smash[t]{\sqrt{\log(1/\eta)/m}})$.
    \item $\Abs{\E_i \xi_i(s)^2 - V(\mu, s)} \le O\Paren{\eta \log(1/\eta) / m}$.
    \item For every subset $T \subseteq [n]$ with $\abs{T} \le 2\eta n$,
      \begin{equation*}
        \frac{1}{n} \sum_{i \in T} \xi_i(s)^2 \le O\Paren{\eta \log(1/\eta) / m} \mper
      \end{equation*}
  \end{enumerate}
\end{definition}

Clean batch-average vectors are $\eta$-good with high probability, as shown next.
\begin{corollary}[True Samples Are $\eta$-Good with High Probability]
  \label{cor:eta-good}
  Let $\mu \in \Delta^{k-1}$, $\eta \in (0,1/4)$, $\beta \in (0,1/2)$, and $m \in \mathbb{N}$.
  Let
  \begin{equation*}
    B_i^* = (U_{i,1}^*,\ldots,U_{i,m}^*),
    \qquad
    U_{i,1}^*,\ldots,U_{i,m}^* \stackrel{\mathrm{iid}}{\sim} \mu,
    \qquad
    y_i^* \coloneq \frac1m \sum_{j=1}^m e_{U_{i,j}^*}
  \end{equation*}
  for each $i \in [n]$.
  Then as long as $n \ge \Omega((k + \log(1/\beta))/(\eta^2 \log(1/\eta)))$, with probability $1-\beta$, the samples $y_1^*, \dots, y_n^*$ are $\eta$-good in the sense of \Cref{def:eta-good}.
\end{corollary}
\begin{proof}
  Apply \Cref{fact:main-resilience-lemma} with parameter $2\eta$. Its sample-size requirement is satisfied by the hypothesis of \Cref{cor:eta-good}, up to absolute constants, so the concentration event holds with probability $1-\beta$.
  Taking the subset $T = [n]$ gives Properties~1 and~2 of \Cref{def:eta-good}.

  For the third property, fix a subset $T \subseteq [n]$ with $\abs{T} \le 2\eta n$, and let $S \coloneq [n] \setminus T$.
  Then $\abs{S} \ge (1-2\eta)n$, and
  \begin{equation*}
    \frac{1}{n} \sum_{i \in T} \xi_i(s)^2
    =
    \E_i \xi_i(s)^2
    - \frac{\abs{S}}{n} \cdot \frac{1}{\abs{S}} \sum_{i \in S} \xi_i(s)^2 \mper
  \end{equation*}
  Therefore,
  \begin{align*}
    \frac{1}{n} \sum_{i \in T} \xi_i(s)^2
    &\le
    \Abs{\E_i \xi_i(s)^2 - V(\mu, s)}
    + \frac{\abs{S}}{n}
    \Abs{\frac{1}{\abs{S}} \sum_{i \in S} \xi_i(s)^2 - V(\mu, s)}
    + \frac{\abs{T}}{n} V(\mu, s) \\
    &\le
    O\Paren{\frac{\eta \log(1/\eta)}{m}}
    + O\Paren{\frac{\eta \log(1/\eta)}{m}}
    + 2\eta \cdot \frac{1}{4m} \\
    &\le
    O\Paren{\frac{\eta \log(1/\eta)}{m}} \mper
  \end{align*}
  Here we used $V(\mu, s) \le 1/(4m)$.
  This proves Property~3 of \Cref{def:eta-good}, and hence the samples are $\eta$-good.
\end{proof}
\subsection{Sum-of-Squares Based Robust Algorithm}
\label{subsec:sos-based-robust-algo}
Let
\begin{equation*}
  B_i^* = (U_{i,1}^*,\ldots,U_{i,m}^*),
  \qquad
  U_{i,1}^*,\ldots,U_{i,m}^* \stackrel{\mathrm{iid}}{\sim} \mu,
  \qquad
  y_i^* \coloneq \frac1m \sum_{j=1}^m e_{U_{i,j}^*}
\end{equation*}
for each $i \in [n]$. Let $\Set{y_i}_{i=1}^n$ be an $\eta$-corrupted version of these clean batch-average vectors, and let $\Set{y_i'}_{i=1}^n$ be the repaired batch-average indeterminates.
Write
\begin{equation*}
  \widehat{\Sigma}' \coloneq \E_i (y_i' - \mu') \transpose{(y_i' - \mu')}, \qquad
  \Sigma(\nu) \coloneq \frac1m\Paren{\diag(\nu) - \nu \transpose{\nu}}, \qquad
  \tau_{\mathrm{cov}} \coloneq C \frac{\eta\log(1/\eta)}{m},
\end{equation*}
for a sufficiently large universal constant $C$.
With variables $\lambda\in\R^k$ and $M\in\R^{k\times k}$, consider the following constraint system:

\begingroup
\small
\begin{align*}
  \mathcal{A}_\eta
  =
  \left\{
  \begin{aligned}
    &\textbf{selector/data consistency:}
    &&\forall i \in [n]:\ w_i^2 = w_i,\quad w_i y_i' = w_i y_i,\quad \sum_{i=1}^n w_i \ge (1-\eta)n, \\
    &\textbf{simplex constraints:}
    &&\forall i,j \in [n] \times [k]:\ y_{i,j}' \ge 0,\quad \forall i \in [n]:\ \sum_{j \in [k]} y_{i,j}' = 1, \\
    &&&\forall j \in [k]:\ \mu_j' \ge 0,\quad \sum_{j \in [k]} \mu'_j = 1, \\
    &\textbf{covariance-majorizer constraints:}
    &&\widehat{\Sigma}' + MM^\top = \Sigma(\mu') + \operatorname{Diag}(\lambda),\quad \sum_{j \in [k]} \lambda_j \le \tau_{\mathrm{cov}}, \\
    &&&\text{where } \mu' = \E_i y_i' \mper
  \end{aligned}
  \right\}
\end{align*}
\endgroup
Selector/data consistency determines which observed batch averages are kept and ties the repaired points to the observed data on those indices. Simplex constraints keep both the repaired points and the pseudo-mean inside the simplex. To replace all subset-direction variance checks succinctly, the covariance-majorizer block uses a diagonal slack that upper-bounds the difference between the empirical and multinomial covariance simultaneously in every subset direction.

On the $\eta$-good event, the clean data provide a witness for this system. If $r_i = \mathbb{1}[y_i = y_i^*]$ and $\bar{\mu} = \E_i y_i^*$, set $w_i = r_i$, $y_i' = y_i^*$, and $\mu' = \bar{\mu}$. These choices satisfy the selector/data-consistency and simplex constraints. By \Cref{prop:subset-control-majorizer}, feasibility of the covariance-majorizer constraint follows once we verify subset discrepancy and the nullspace condition for the clean covariance-gap matrix $\Delta_{\mathrm{cov}}^\star$.

\begin{lemma}[Clean Covariance Gap Has Small Subset Discrepancy]
  \label{lem:clean-cov-gap-subset}
  Let $y_1^*, \dots, y_n^*$ be clean batch-average vectors, and suppose they are $\eta$-good in the sense of \Cref{def:eta-good}. Let
  \begin{equation*}
    \bar{\mu} \coloneq \E_i y_i^*, \qquad
    \Delta_{\mathrm{cov}}^\star \coloneq \E_i (y_i^* - \bar{\mu}) \transpose{(y_i^* - \bar{\mu})}
    - \frac1m \Paren{\diag(\bar{\mu}) - \bar{\mu} \transpose{\bar{\mu}}} \mper
  \end{equation*}
  Then for every $s \in \Set{0,1}^k$,
  \begin{equation*}
    \Abs{s^\top \Delta_{\mathrm{cov}}^\star s} \le O\Paren{\frac{\eta \log(1/\eta)}{m}} \mper
  \end{equation*}
\end{lemma}
\begin{proof}
  Let $h \coloneq \bar{\mu} - \mu$.
  Fix $s \in \Set{0,1}^k$, and write $\xi_i(s) \coloneq \iprod{y_i^* - \mu, s}$.
  Then
  \begin{align*}
    s^\top \Delta_{\mathrm{cov}}^\star s
    &=
    \E_i \iprod{y_i^* - \bar{\mu}, s}^2 - V(\bar{\mu}, s) \\
    &=
    \E_i \Paren{\xi_i(s) - \iprod{h, s}}^2 - V(\bar{\mu}, s) \\
    &=
    \E_i \xi_i(s)^2 - \iprod{h, s}^2 - V(\bar{\mu}, s) \\
    &=
    \Paren{\E_i \xi_i(s)^2 - V(\mu, s)} - \iprod{h, s}^2 + \Paren{V(\mu, s)-V(\bar{\mu}, s)} \mper
  \end{align*}
  By Property~1 of \Cref{def:eta-good},
  \begin{equation*}
    \Abs{\iprod{h, s}} = \Abs{\E_i \xi_i(s)} \le O\Paren{\eta\sqrt{\frac{\log(1/\eta)}{m}}} \mper
  \end{equation*}
  Since $x \mapsto x(1-x)$ is $1$-Lipschitz on $[0,1]$, we obtain
  \begin{equation*}
    \Abs{V(\mu, s) - V(\bar{\mu}, s)}
    \le
    \frac{\Abs{\iprod{h, s}}}{m}
    \le
    O\Paren{\frac{\eta\sqrt{\log(1/\eta)}}{m^{3/2}}}
    \le
    O\Paren{\frac{\eta \log(1/\eta)}{m}} \mper
  \end{equation*}
  Therefore
  \begin{align*}
    \Abs{s^\top \Delta_{\mathrm{cov}}^\star s}
    &\le
    \Abs{\E_i \xi_i(s)^2 - V(\mu, s)}
    +
    \Abs{\iprod{h, s}}^2
    +
    \Abs{V(\mu, s)-V(\bar{\mu}, s)} \\
    &\le
    O\Paren{\frac{\eta \log(1/\eta)}{m}}
    +
    O\Paren{\frac{\eta^2 \log(1/\eta)}{m}}
    +
    O\Paren{\frac{\eta \log(1/\eta)}{m}} \\
    &\le
    O\Paren{\frac{\eta \log(1/\eta)}{m}} \mper
  \end{align*}
\end{proof}

\begin{lemma}[Clean Covariance Gap Kills the All-Ones Direction]
  \label{lem:clean-cov-gap-nullspace}
  Let $y_1^*, \dots, y_n^*$ be clean batch-average vectors, and define
  \begin{equation*}
    \bar{\mu} \coloneq \E_i y_i^*, \qquad
    \Delta_{\mathrm{cov}}^\star \coloneq \E_i (y_i^* - \bar{\mu}) \transpose{(y_i^* - \bar{\mu})}
    - \frac1m \Paren{\diag(\bar{\mu}) - \bar{\mu} \transpose{\bar{\mu}}} \mper
  \end{equation*}
  Then
  \begin{equation*}
    \Delta_{\mathrm{cov}}^\star \bm{1} = 0 \mper
  \end{equation*}
\end{lemma}
\begin{proof}
  Since each $y_i^*$ and $\bar{\mu}$ lie in the simplex,
  \begin{equation*}
    \iprod{y_i^* - \bar{\mu}, \bm{1}} = 0
    \qquad\text{for all } i \in [n],
  \end{equation*}
  and hence
  \begin{equation*}
    \E_i (y_i^* - \bar{\mu}) \transpose{(y_i^* - \bar{\mu})} \bm{1} = 0 \mper
  \end{equation*}
  Also,
  \begin{equation*}
    \Paren{\diag(\bar{\mu}) - \bar{\mu} \transpose{\bar{\mu}}}\bm{1}
    =
    \bar{\mu} - \bar{\mu}\iprod{\bar{\mu}, \bm{1}}
    =
    0 \mper
  \end{equation*}
\end{proof}

\begin{corollary}[Diagonal-Majorizer Feasibility]
  \label{lem:covariance-feasibility}
  Let $y_1^*, \dots, y_n^*$ be clean batch-average vectors, and suppose these vectors are $\eta$-good in the sense of \Cref{def:eta-good}. Let
  \begin{equation*}
    \bar{\mu} \coloneq \E_i y_i^*, \qquad
    \Delta_{\mathrm{cov}}^\star \coloneq \E_i (y_i^* - \bar{\mu}) \transpose{(y_i^* - \bar{\mu})}
    - \frac1m \Paren{\diag(\bar{\mu}) - \bar{\mu} \transpose{\bar{\mu}}} \mper
  \end{equation*}
  Then there exists $\lambda \in \R^k$ such that
  \begin{equation*}
    \operatorname{Diag}(\lambda) - \Delta_{\mathrm{cov}}^\star \succcurlyeq 0,
    \qquad
    \sum_{j \in [k]} \lambda_j \le O\Paren{\frac{\eta \log(1/\eta)}{m}} \mper
  \end{equation*}
\end{corollary}
\begin{proof}
  By \Cref{prop:subset-control-majorizer}, it remains only to verify the subset-discrepancy and nullspace hypotheses. Those are exactly \Cref{lem:clean-cov-gap-subset,lem:clean-cov-gap-nullspace} applied to $\Delta = \Delta_{\mathrm{cov}}^\star$.
\end{proof}
This proves feasibility of the covariance-majorizer block for the clean witness.

If $\Delta' \bm{1}=0$ and $\operatorname{Diag}(\lambda)-\Delta' \succcurlyeq 0$, then for every subset direction $s \in \Set{0,1}^k$ and $\sigma = 2s-\bm{1}$,
\begin{equation*}
  4 s^\top \Delta' s
  =
  \sigma^\top \Delta' \sigma
  \le
  \sigma^\top \operatorname{Diag}(\lambda)\sigma
  =
  \sum_{j \in [k]} \lambda_j \mper
\end{equation*}
Thus one matrix majorizer simultaneously controls every subset direction. We next give the SoS version of this calculation for $\Delta_{\mathrm{cov}}' = \smash[t]{\widehat{\Sigma}}' - \Sigma(\mu')$.

\begin{lemma}[Diagonal Majorizer Implies Program-Side Subset-Direction Variance Control]
  \label{lem:diag-majorizer-subset}
  Under $\mathcal{A}_\eta$ and the Boolean constraints $\{\forall j \in [k]: s_j^2 = s_j\}$, we have
  \begin{equation*}
    \mathcal{A}_\eta \cup \Set{\forall j \in [k]: s_j^2 = s_j}
    \sststile{\sodared{6}}{}
    \Set{
      s^\top \Paren{\widehat{\Sigma}' - \Sigma(\mu')} s \le \frac14 \sum_{j \in [k]} \lambda_j
    } \mper
  \end{equation*}
  Consequently,
  \begin{equation*}
    \mathcal{A}_\eta \cup \Set{\forall j \in [k]: s_j^2 = s_j}
    \sststile{\sodared{6}}{}
    \Set{
      \E_i \iprod{y_i' - \mu', s}^2
      \le
      \frac{\iprod{\mu', s}\Paren{1-\iprod{\mu', s}}}{m}
      +
      O\Paren{\frac{\eta \log(1/\eta)}{m}}
    } \mper
  \end{equation*}
\end{lemma}
\begin{proof}
  Write
  \begin{equation*}
    \Delta_{\mathrm{cov}}' \coloneq \widehat{\Sigma}' - \Sigma(\mu'), \qquad \sigma \coloneq 2s - \bm{1} \mper
  \end{equation*}
  We first show that $\mathcal{A}_\eta \sststile{\sodared{6}}{} \Set{\Delta_{\mathrm{cov}}' \bm{1} = 0}$. Indeed, for every $i \in [n]$,
  \begin{equation*}
    \mathcal{A}_\eta \sststile{\sodared{6}}{} \Set{\iprod{y_i' - \mu', \bm{1}} = \sum_{j \in [k]} y_{i,j}' - \sum_{j \in [k]} \mu_j' = 0} \mper
  \end{equation*}
  Therefore
  \begin{equation*}
    \mathcal{A}_\eta \sststile{\sodared{6}}{} \Set{\widehat{\Sigma}' \bm{1} = \E_i (y_i' - \mu') \iprod{y_i' - \mu', \bm{1}} = 0} \mper
  \end{equation*}
  Also,
  \begin{equation*}
    \mathcal{A}_\eta \sststile{\sodared{6}}{} \Set{\Sigma(\mu') \bm{1} = \frac{1}{m}\Paren{\mu' - \mu' \iprod{\mu', \bm{1}}} = 0} \mper
  \end{equation*}
  and hence $\Delta_{\mathrm{cov}}' \bm{1} = 0$.

  Using this identity, we have
  \begin{equation*}
    \mathcal{A}_\eta \cup \Set{\forall j \in [k]: s_j^2 = s_j}
    \sststile{\sodared{6}}{}
    \Set{
      \sigma^\top \Delta_{\mathrm{cov}}' \sigma = (2s-\bm{1})^\top \Delta_{\mathrm{cov}}' (2s-\bm{1}) = 4 s^\top \Delta_{\mathrm{cov}}' s
    } \mper
  \end{equation*}
  Now the diagonal-majorizer constraint in $\mathcal{A}_\eta$ gives
  \begin{equation*}
    \mathcal{A}_\eta \cup \Set{\forall j \in [k]: s_j^2 = s_j}
    \sststile{\sodared{6}}{}
    \Set{
      \sigma^\top \Delta_{\mathrm{cov}}' \sigma \le \sigma^\top \operatorname{Diag}(\lambda) \sigma
    } \mper
  \end{equation*}
  since
  \begin{equation*}
    \sigma^\top \Paren{\operatorname{Diag}(\lambda) - \Delta_{\mathrm{cov}}'} \sigma = \normt{M^\top \sigma}^2 \ge 0 \mper
  \end{equation*}
  Under the Boolean constraints on $s$, we have
  \begin{equation*}
    \Set{\forall j \in [k]: s_j^2 = s_j} \sststile{\sodared{6}}{} \Set{\sigma_j^2 = (2s_j-1)^2 = 1}
  \end{equation*}
  for every $j \in [k]$, and therefore
  \begin{equation*}
    \mathcal{A}_\eta \cup \Set{\forall j \in [k]: s_j^2 = s_j}
    \sststile{\sodared{6}}{}
    \Set{
      \sigma^\top \operatorname{Diag}(\lambda) \sigma = \sum_{j \in [k]} \lambda_j
    } \mper
  \end{equation*}
  Substituting $\sigma^\top\Delta_{\mathrm{cov}}'\sigma=4s^\top\Delta_{\mathrm{cov}}'s$ and $\sigma^\top\operatorname{Diag}(\lambda)\sigma=\sum_j\lambda_j$ into the diagonal-majorizer inequality proves the first claim.

  For the variance bound, observe that
  \begin{equation*}
    \Set{\forall j \in [k]: s_j^2 = s_j}
    \sststile{\sodared{6}}{}
    \Set{
      s^\top \widehat{\Sigma}' s = \E_i \iprod{y_i' - \mu', s}^2
    }
  \end{equation*}
  and
  \begin{equation*}
    \mathcal{A}_\eta \cup \Set{\forall j \in [k]: s_j^2 = s_j}
    \sststile{\sodared{6}}{}
    \Set{
      s^\top \Sigma(\mu') s
      =
      \frac{\iprod{\mu', s} - \iprod{\mu', s}^2}{m}
      =
      \frac{\iprod{\mu', s}\Paren{1-\iprod{\mu', s}}}{m}
    } \mper
  \end{equation*}
  Using the bound $s^\top\Delta_{\mathrm{cov}}'s\le\tfrac14\sum_j\lambda_j$ together with $\sum_j \lambda_j \le \tau_{\mathrm{cov}}$ gives
  \begin{equation*}
    \mathcal{A}_\eta \cup \Set{\forall j \in [k]: s_j^2 = s_j}
    \sststile{\sodared{6}}{}
    \Set{
      \E_i \iprod{y_i' - \mu', s}^2
      \le
      \frac{\iprod{\mu', s}\Paren{1-\iprod{\mu', s}}}{m}
      +
      \frac14 \tau_{\mathrm{cov}}
    } \mper
  \end{equation*}
  which is the desired bound.
\end{proof}

We can now state our main technical lemma.

\begin{lemma}[Main SoS Lemma]
  \label{lem:main-sos-lemma}
  Suppose $\mathcal{Y}^* = \set{y_i^*}_{i \in [n]}$ is the clean batch-average dataset and that it is $\eta$-good (as in \Cref{def:eta-good}). Then for any fixed binary vector $s \in \Set{0,1}^k$ the following SoS proof exists.
  \begin{equation*}
    \mathcal{A}_\eta \sststile{\sodared{6}}{} \Set{\iprod{\mu' - \mu, s}^2 \le O\Paren{\Paren{\eta\sqrt{\frac{\log(1/\eta)}{m}}}^2}}
  \end{equation*}
\end{lemma}
\Cref{alg:robust-discrete-batches} summarizes the robust estimator promised by \Cref{thm:main-robust}.

\begin{boxedalgorithm}{Robust Estimation from Discrete Batches}
  \label{alg:robust-discrete-batches}
  \alginput An $\eta$-corrupted batch dataset $\set{B_i}_{i \in [n]}$, where each batch has size $m$.
  \algoutput An estimate $\hat{\mu}$ of the underlying discrete distribution $\mu$.
  \algoperations
  \begin{algsteps}
    \item Convert each observed batch to its average vector
      \begin{equation*}
        y_i \coloneq \frac{1}{m}\sum_{j=1}^m e_{U_{i,j}} \in \Delta^{k-1}
        \qquad \text{for all } i \in [n] \mper
      \end{equation*}
    \item Form the SoS constraint system $\cA_\eta$ from \Cref{subsec:sos-based-robust-algo}, with variables $(w_i)_{i \in [n]}$, $(y_i')_{i \in [n]}$, $\mu'$, $\lambda$, and $M$.
    \item Find a \sodared{degree-$6$ pseudoexpectation} $\pE$ satisfying $\cA_\eta$.
    \item Output an $\ell_1$ projection of $\pE\mu'$ onto $\Delta^{k-1}$. For an exact pseudoexpectation, this is $\hat\mu=\pE\mu'$.
  \end{algsteps}
\end{boxedalgorithm}
\begin{sodarevision}
  We use exact pseudoexpectation notation in the analysis. For computation, we may restrict the SDP to a ball of radius $\poly(n,k)$ in moment space. Indeed, the clean witness has $w_i,y_{i,j}',\mu_j'\in[0,1]$. Since $(\Delta_{\mathrm{cov}}^\star)_{jj}\ge-1/(4m)$, its diagonal majorizer satisfies
  \begin{equation*}
    -\frac1{4m}\le\lambda_j\le\tau_{\mathrm{cov}}+\frac{k-1}{4m},
    \qquad
    \normf{M}^2=\sum_j\lambda_j-\operatorname{tr}(\Delta_{\mathrm{cov}}^\star)
    \le\tau_{\mathrm{cov}}+\frac{k}{4m}.
  \end{equation*}
  Thus the point-evaluation functional at this witness has polynomially bounded degree-$6$ moments.
  In the standard bit model, solve the bounded SDP to an inverse-polynomial feasibility tolerance $\tau$ chosen so that
  \begin{equation*}
    \tau\,\poly(n,k,m)
    =o\Paren{\frac{\eta^2\log(1/\eta)}{m}}.
  \end{equation*}
  Since the degree-$6$ certificates in the proof of \Cref{lem:main-sos-lemma} have coefficient norms bounded by $\poly(n,k,m)$, approximate feasibility changes the squared-error bound by at most $\tau\poly(n,k,m)$. Projecting onto the simplex in $\ell_1$ changes the error by at most a constant factor. This preserves the error rate in \Cref{thm:main-robust} and polynomial running time, as in \Cref{lem:certificate-soundness}.
\end{sodarevision}

We now derive \Cref{thm:main-robust} from \Cref{lem:main-sos-lemma}.
\begin{proof}[Proof of \Cref{thm:main-robust}]
  Apply \Cref{cor:eta-good} with failure probability $\min\{\beta,1/4\}$. With probability at least $1-\beta$, the clean batch-average vectors satisfy the three properties in \Cref{def:eta-good}.
  Conditioned on this event, \Cref{lem:main-sos-lemma} applies.
  Let $\pE$ be a \sodared{degree-$6$ pseudoexpectation} satisfying $\mathcal{A}_\eta$, and let $\hat{\mu} = \pE \mu'$ be the output of the algorithm. Let $s \in \Set{0,1}^k$ be defined coordinate-wise by $s_i = \mathbb{1}[\hat{\mu}_i > \mu_i]$. Then we have
  \begin{align*}
    \normtv{\hat{\mu} - \mu}^2
    &=
    \iprod{\hat{\mu} - \mu, s}^2 \\
    &=
    \iprod{\pE \mu' - \mu, s}^2 \\
    &\le
    \pE \iprod{\mu' - \mu, s}^2  \\
    &\le O\Paren{\Paren{\eta \sqrt{\frac{\log(1/\eta)}{m}}}^2} \mper
  \end{align*}
  Therefore,
  \begin{equation*}
    \normtv{\hat{\mu} - \mu} \le O\Paren{\eta \sqrt{\frac{\log(1/\eta)}{m}}} \mcom
  \end{equation*}
  as desired.
\end{proof}
\subsection{Proof of the Main SoS Lemma}
\label{subsec:proof-of-main-sos-lemma}
\begin{proof}
  We proceed in three steps: first control the clean-kept and corrupted/discarded contributions in a fixed direction; then use the diagonal-majorizer certificate to upper bound the repaired second moment; finally control and absorb the remaining variance difference using the covariance formula $V(\mu,s)=\iprod{\mu,s}(1-\iprod{\mu,s})/m$.

  Let $r_i$ denote the indicator function of uncorrupted samples, i.e.\ $r_i = \mathbb{1}[y_i^* = y_i]$.
  Let $s \in \Set{0,1}^k$ be fixed, and write
  \begin{equation*}
    \Delta_s \coloneq \iprod{\mu' - \mu, s}, \qquad \xi_i \coloneq \iprod{y_i^* - \mu, s}, \qquad \xi_i' \coloneq \iprod{y_i' - \mu, s} \mper
  \end{equation*}
  Multiplying by $r_i w_i$ selects indices that are both clean and kept; $r_i(1-w_i)$ selects clean indices discarded by the selector; and $1-r_i w_i$ is the complement.
  \paragraph{Split into the clean-kept part and its complement.}
  Applying SoS triangle inequality we can write
  \begin{equation}
    \label{eq:directional-error-decomposition}
    \mathcal{A}_\eta
    \sststile{\sodared{6}}{} \Delta_s^2 \le 2 \underbrace{\Paren{\E_i r_i w_i \xi_i'}^2}_{\text{clean-kept contribution}} + 2\underbrace{\Paren{\E_i (1-r_i w_i) \xi_i'}^2}_{\text{corrupted-or-discarded contribution}} \mper
  \end{equation}
  \paragraph{Clean-kept contribution.}
  Observing that $r_iw_iy_i^* = r_i w_iy_i'$, we can write
  \begin{equation*}
    \mathcal{A}_\eta
    \sststile{\sodared{6}}{}
    \sodared{\Set{
        \begin{aligned}
          \Paren{\E_i r_i w_i \xi_i'}^2 &= \Paren{\E_i r_i w_i \xi_i}^2 \\
          &\le
          2 \Paren{\E_i r_i \xi_i}^2 + 2 \Paren{\E_i r_i \paren{1-w_i} \xi_i}^2 \\
        \end{aligned}
    }} \mper
  \end{equation*}
  Now
  \begin{align*}
    \Paren{\E_i r_i \xi_i}^2
    &=
    \Paren{\E_i \xi_i - \E_i (1-r_i)\xi_i}^2 \\
    &\le
    2 \Paren{\E_i \xi_i}^2 + 2 \Paren{\E_i (1-r_i)\xi_i}^2 \\
    &\le
    O\Paren{\frac{\eta^2 \log(1/\eta)}{m}}
    + 2 \E_i (1-r_i) \cdot \E_i (1-r_i) \xi_i^2 \\
    &\le
    O\Paren{\frac{\eta^2 \log(1/\eta)}{m}} \mper
  \end{align*}
  Here we used Property~1 of \Cref{def:eta-good} in the third line and Property~3 in the last line, since the corrupted indices form a subset of size at most $\eta n$.
  Therefore
  \begin{align*}
    \mathcal{A}_\eta
    \sststile{\sodared{6}}{}
    \Paren{\E_i r_i w_i \xi_i'}^2
    &\le
    O\Paren{\frac{\eta^2 \log(1/\eta)}{m}} + 2 \underbrace{\Paren{\E_i r_i \Paren{1-w_i} \xi_i}^2}_{\text{discarded-clean contribution}}
    \tageq\label{eq:discarded-clean-term} \mper
  \end{align*}
  We next bound the discarded-clean contribution. We have
  \begin{equation*}
    \mathcal{A}_\eta
    \sststile{\sodared{6}}{}
    \sodared{\Set{
        \begin{aligned}
          \Paren{\E_i r_i (1-w_i) \xi_i}^2
          &\le
          \Paren{\E_i r_i (1-w_i)} \cdot \Paren{\E_i r_i (1-w_i) \xi_i^2} \\
          &\le \eta \cdot \Paren{\E_i (1-w_i) \xi_i^2}
        \end{aligned}
    }} \mper
  \end{equation*}
  Property~3 of \Cref{def:eta-good} together with \Cref{lem:sos-selector-lemma} implies
  \begin{equation*}
    \mathcal{A}_\eta
    \sststile{\sodared{6}}{}
    \sodared{\Set{
        \begin{aligned}
          \Paren{\E_i r_i (1-w_i) \xi_i}^2
          &\le \eta \cdot O(\eta \log(1/\eta) / m) \\
          &= O(\eta^2 \log(1/\eta) / m)
        \end{aligned}
    }} \mper
  \end{equation*}
  Therefore,
  \begin{equation*}
    \mathcal{A}_\eta \sststile{\sodared{6}}{} \Set{\Paren{\E_i r_i w_i \xi_i'}^2 \le O\Paren{\frac{\eta^2 \log(1/\eta)}{m}}} \mper
  \end{equation*}

  \begin{samepage}
    \paragraph{Corrupted-or-discarded contribution.}
    We can write
    \begin{equation*}
      \mathcal{A}_\eta
      \sststile{\sodared{6}}{}
      \sodared{\Set{
          \begin{aligned}
            \Paren{\E_i \paren{1-r_i w_i} \xi_i'}^2
            &\le
            \Paren{\E_i \paren{1-r_i w_i}} \Paren{\E_i \paren{1-r_i w_i} \xi_i'^2} \\
            &\le
            2 \eta \Paren{\E_i \paren{1-r_i w_i} \xi_i'^2} \\
            &=
            2 \eta \paren{
              \underbrace{\E_i \xi_i'^2}_{\text{repaired second moment}}
              -
              \underbrace{\E_i r_i w_i \xi_i'^2}_{\text{clean-kept second moment}}
            }
          \end{aligned}
      }} \mper
      \tageq\label{eq:complement-second-moment-decomposition}
    \end{equation*}
  \end{samepage}
  \begin{sodarevision}
    For completeness, the weighted Cauchy--Schwarz step in \eqref{eq:complement-second-moment-decomposition} is the degree-$6$ bottleneck. With $b_i\coloneq 1-r_iw_i$ and $z_i\coloneq\xi_i'$, the Boolean identities $b_i^2=b_i$ give
    \begin{equation*}
      \mathcal{A}_\eta
      \sststile{6}{}
      \Set{
        \begin{aligned}
          \Paren{\E_i b_i}\Paren{\E_i b_i z_i^2}-\Paren{\E_i b_i z_i}^2
          &=\frac{1}{2n^2}\sum_{i,j}b_i b_j(z_i-z_j)^2 \\
          &=\frac{1}{2n^2}\sum_{i,j}\bigl(b_i b_j(z_i-z_j)\bigr)^2 \\
          &\ge 0
        \end{aligned}
      } \mper
    \end{equation*}
    Since the polynomials being squared have degree $3$, this certificate has degree $6$.
  \end{sodarevision}
  We first bound the repaired second moment. Since
  \begin{equation*}
    \E_i \xi_i'^2 = \E_i \iprod{y_i' - \mu', s}^2 + \Delta_s^2,
  \end{equation*}
  \Cref{lem:diag-majorizer-subset} with subset direction $s$ gives
  \begin{equation*}
    \mathcal{A}_\eta
    \sststile{\sodared{6}}{}
    \sodared{\Set{
        \begin{aligned}
          \E_i \xi_i'^2
          &\le
          \frac{\iprod{\mu', s}\Paren{1 - \iprod{\mu', s}}}{m} + O\Paren{\frac{\eta \log(1/\eta)}{m}} + \Delta_s^2 \\
          &=
          V(\mu', s) + O\Paren{\frac{\eta \log(1/\eta)}{m}} + \Delta_s^2
        \end{aligned}
    }} \mcom
  \end{equation*}
  Next we lower bound the clean-kept second moment. Since $r_i w_i \xi_i'^2 = r_i w_i \xi_i^2$, we have
  \begin{equation*}
    \mathcal{A}_\eta
    \sststile{\sodared{6}}{}
    \sodared{\Set{
        \begin{aligned}
          \E_i r_i w_i \xi_i'^2
          &= \E_i r_i w_i \xi_i^2 \\
          &= \E_i \xi_i^2 - \E_i (1-r_i w_i) \xi_i^2 \\
          &\ge V(\mu, s) - O(\eta \log(1/\eta) / m)
          - \underbrace{\E_i (1-r_i w_i) \xi_i^2}_{\text{missing clean second moment}}
        \end{aligned}
    }} \mcom
    \tageq\label{eq:missing-clean-second-moment}
  \end{equation*}
  Here the first inequality follows from Property~2 of \Cref{def:eta-good}, and the remaining error comes from the missing clean second moment. Since
  \begin{equation*}
    \sum_i (1-r_i w_i) \le \sum_i (1-r_i) + \sum_i r_i(1-w_i) \le 2\eta n,
  \end{equation*}
  Property~3 of \Cref{def:eta-good} together with \Cref{lem:sos-selector-lemma} gives
  \begin{align*}
    \mathcal{A}_\eta
    \sststile{\sodared{6}}{}
    \E_i (1-r_i w_i) \xi_i^2
    &\le O(\eta \log(1/\eta) / m) \mper
  \end{align*}
  \begin{sodarevision}
    \paragraph{Directional identifiability inequality.}
    Apply the clean-kept contribution bound to \eqref{eq:directional-error-decomposition}, and use \eqref{eq:complement-second-moment-decomposition} for the complementary contribution. Combining the upper bound on $\E_i\xi_i'^2$ with the lower bound on $\E_i r_iw_i\xi_i'^2$ then gives
    \begin{equation}
      \label{eq:robust-directional-identifiability}
      \mathcal{A}_\eta
      \sststile{6}{}
      \Set{
        \begin{aligned}
          \Delta_s^2
          &\le O\Paren{\frac{\eta^2 \log(1/\eta)}{m}}
          +4\eta\Paren{\E_i\xi_i'^2-\E_i r_iw_i\xi_i'^2} \\
          &\le O\Paren{\frac{\eta^2 \log(1/\eta)}{m}}
          +4\eta\Paren{V(\mu', s) - V(\mu, s)}
          +4\eta\Delta_s^2
        \end{aligned}
      } \mper
    \end{equation}

    \paragraph{Absorbing the variance difference.}
    The term $4\eta(V(\mu',s)-V(\mu,s))$ in \eqref{eq:robust-directional-identifiability} is linear in the variance difference, whereas \Cref{lem:sos-cov} controls the square of that difference. We use Young's inequality to pass from the linear term to the square. In \Cref{fact:sos-young-inequality}, take $a=V(\mu',s)-V(\mu,s)$, $b=4\eta$, and $\lambda=m^2/16$. This gives
    \begin{equation*}
      \mathcal{A}_\eta
      \sststile{4}{}
      \Set{
        \begin{aligned}
          4\eta\Paren{V(\mu',s)-V(\mu,s)}
          &\le \frac{m^2}{16}\Paren{V(\mu',s)-V(\mu,s)}^2
          +\frac{(4\eta)^2}{4(m^2/16)} \\
          &=\frac{m^2}{16}\Paren{V(\mu',s)-V(\mu,s)}^2
          +\frac{64\eta^2}{m^2}
        \end{aligned}
      } \mper
    \end{equation*}
    We chose $\lambda=m^2/16$ so that \Cref{lem:sos-cov} turns the first term on the right into $\Delta_s^2/4$. Indeed, by the definition of $V$,
    \begin{equation*}
      \mathcal{A}_\eta
      \sststile{6}{}
      \Set{
        \begin{aligned}
          \frac{m^2}{16}\Paren{V(\mu',s)-V(\mu,s)}^2
          &= \frac{1}{16}\Paren{
            \iprod{\mu',s}\Paren{1-\iprod{\mu',s}}
            -\iprod{\mu,s}\Paren{1-\iprod{\mu,s}}
          }^2
          \\
          &\le \frac14\Delta_s^2
        \end{aligned}
      } \mper
    \end{equation*}
    Combining these two inequalities, we obtain
    \begin{equation}
      \label{eq:robust-variance-difference-absorption}
      \mathcal{A}_\eta
      \sststile{6}{}
      \Set{
        4\eta\Paren{V(\mu',s)-V(\mu,s)}
        \le \frac14\Delta_s^2+\frac{64\eta^2}{m^2}
      } \mper
    \end{equation}
    Substituting \eqref{eq:robust-variance-difference-absorption} into \eqref{eq:robust-directional-identifiability} gives
    \begin{equation*}
      \mathcal{A}_\eta
      \sststile{6}{}
      \Set{
        \Delta_s^2
        \le O\Paren{\frac{\eta^2\log(1/\eta)}{m}}
        +\Paren{\frac14+4\eta}\Delta_s^2
        +\frac{64\eta^2}{m^2}
      } \mper
    \end{equation*}
    Since $\eta\le\etaabs$ and $\etaabs$ is a sufficiently small absolute constant, we may absorb the middle term into the left-hand side and conclude
    \begin{equation*}
      \mathcal{A}_\eta
      \sststile{6}{}
      \Set{
        \Delta_s^2
        \le O\Paren{\frac{\eta^2\log(1/\eta)}{m}}
        +O\Paren{\frac{\eta^2}{m^2}}
        =O\Paren{\frac{\eta^2\log(1/\eta)}{m}}
      } \mcom
    \end{equation*}
    as desired. The weighted Cauchy--Schwarz step in \eqref{eq:complement-second-moment-decomposition} has degree $6$, and every other step has degree at most $6$. Hence the entire proof has degree $6$.
  \end{sodarevision}
\end{proof}
\subsection{Covariance Closeness Guarantees}
We first prove the following lemma that will help us in bounding the differences of covariances by the differences of the means inside SoS.
\begin{lemma}
  \label{lem:sos-cov-prep}
  Let $x$ be an indeterminate and $a$ a scalar in $[0,1]$. Then
  \begin{equation*}
    \Set{0 \le x\le 1} \sststile{\sodared{6}}{x} \Set{\Paren{x(1-x) - a(1-a)}^2\le 4(x-a)^2}
  \end{equation*}
\end{lemma}
\begin{proof}
  \begin{sodarevision}
    We have
    \begin{align*}
      4(x-a)^2-\Paren{x(1-x)-a(1-a)}^2
      &=4(x-a)^2-(x-a)^2(x+a-1)^2 \\
      &=3(x-a)^2+(x-a)^2\Paren{1-(x+a-1)^2} \\
      &=3(x-a)^2
      +\Paren{1-(a-1)^2}(1-x)(x-a)^2
      +(1-a^2)x(x-a)^2 \\
      &\qquad+x\bigl[(x-a)(1-x)\bigr]^2
      +(1-x)\bigl[x(x-a)\bigr]^2 \\
      &\ge 0 \mper
    \end{align*}
    where every term in the last two lines is a square times either $1$, $x$, or $1-x$, with a nonnegative scalar coefficient, since $1-(a-1)^2\ge0$ and $1-a^2\ge0$. Thus,
    \begin{equation*}
      \Set{0\le x\le1}
      \sststile{6}{x}
      \Set{\Paren{x(1-x)-a(1-a)}^2\le4(x-a)^2} \mper
    \end{equation*}
  \end{sodarevision}
\end{proof}
Now using this lemma we can bound the difference between the covariances inside SoS.
\begin{lemma}
  \label{lem:sos-cov}
  Let $s$ be any fixed binary vector in $\Set{0,1}^k$. Then
  \begin{equation*}
    \mathcal{A}_\eta \sststile{\sodared{6}}{} \Set{
      \Paren{
        \iprod{\mu', s}\paren{1-\iprod{\mu', s}}
        -
        \iprod{\mu, s}\paren{1- \iprod{\mu, s}}
      }^2
      \le 4\iprod{\mu' - \mu, s}^2
    } \mper
  \end{equation*}
\end{lemma}
\begin{proof}
  Note that
  \begin{equation*}
    \mathcal{A}_\eta \sststile{\sodared{6}}{}
    \sodared{\Set{
        \begin{aligned}
          \iprod{\mu' , s} &= \sum_{j \in [k]} \mu'_j s_j
          \\
          &\ge 0
        \end{aligned}
    }} \mcom
  \end{equation*}
  where the inequality comes from $\mathcal{A}_\eta \sststile{}{} \set{\forall j \in [k]: \mu_j' \ge 0}$, and that $s_j$'s are fixed scalars larger than or equal to $0$.
  Moreover,
  \begin{equation*}
    \mathcal{A}_\eta \sststile{\sodared{6}}{}
    \sodared{\Set{
        \begin{aligned}
          \iprod{\mu' , s} &= \sum_{j \in [k]} \mu'_j s_j
          \\
          &\le 1
        \end{aligned}
    }} \mcom
  \end{equation*}
  where the inequality comes from the fact that $s_j$'s are fixed scalars in $\Set{0,1}$, and that $\mathcal{A}_\eta  \sststile{}{} \set{\forall j \in [k]: \mu_j' \ge 0, \sum_{j} \mu_j' = 1}$. Therefore,
  \begin{equation*}
    \mathcal{A}_\eta \sststile{\sodared{6}}{}
    \Set{0 \le \iprod{\mu' , s}\le 1} \mcom
  \end{equation*}
  \begin{sodarevision}
    Applying \Cref{lem:sos-cov-prep} with $x=\iprod{\mu',s}$ and the fixed scalar $a=\iprod{\mu,s}\in[0,1]$, we obtain
    \begin{equation*}
      \mathcal{A}_\eta \sststile{\sodared{6}}{} \Set{
        \Paren{
          \iprod{\mu', s}\paren{1-\iprod{\mu', s}}
          -
          \iprod{\mu, s}\paren{1- \iprod{\mu, s}}
        }^2
        \le 4\iprod{\mu' - \mu, s}^2
      } \mcom
    \end{equation*}
  \end{sodarevision}
  as desired.
\end{proof}

\subsection{Heterogeneous Clean Batches}
\begin{corollary}
  \label{cor:robust-heterogeneous}
  Suppose the clean batches are independent and batch $i$ consists of $m$ i.i.d.\ samples from $\mu_i\in\Delta^{k-1}$, where $\normtv{\mu_i-\mu}\le\omega$ for every $i$ and some known $\omega\in[0,1]$.
  Let $\eta\in(0,\etaabs)$ and $\beta\in(0,1)$.
  Replace the covariance allowance in \Cref{alg:robust-discrete-batches} by
  \begin{equation*}
    \tau_{\mathrm{cov}}=C\Paren{\frac{\eta\log(1/\eta)}m+\frac\omega m+\omega^2}
  \end{equation*}
  for a sufficiently large absolute constant $C$.
  If $n\ge\Omega((k+\log(1/\beta))/(\eta^2\log(1/\eta)))$, then on an $\eta$-corruption of these batches the algorithm returns, with probability at least $1-\beta$,
  \begin{equation*}
    \normtv{\hat\mu-\mu}\le O\Paren{\eta\sqrt{\frac{\log(1/\eta)}m}+\omega}.
  \end{equation*}
\end{corollary}
\begin{proof}
  For a fixed $s\in\{0,1\}^k$, write
  $d_i=\iprod{\mu_i-\mu,s}$ and $Z_i=\iprod{y_i^*-\mu_i,s}$, so that $|d_i|\le\omega$ and $\xi_i=Z_i+d_i$.
  Each $Z_i$ is centered and subgaussian with parameter $O(1/m)$.
  By independence, the proof of \Cref{fact:main-resilience-lemma} applies to these variables even though their distributions may differ.
  In addition, the weighted average $\E_i d_iZ_i$ has subgaussian parameter $O(\omega^2/(nm))$.
  A union bound over $s$, the inequality $|V(\mu_i,s)-V(\mu,s)|\le\omega/m$, and $(Z_i+d_i)^2\le2Z_i^2+2d_i^2$ give the following simultaneous bounds with probability at least $1-\beta$:
  \begin{align*}
    |\E_i\xi_i|&\le O\Paren{\eta\sqrt{\frac{\log(1/\eta)}m}+\omega},\\
    |\E_i\xi_i^2-V(\mu,s)|&\le O\Paren{\frac{\eta\log(1/\eta)}m+\frac\omega m+\omega^2},\\
    \max_{|T|\le2\eta n}\frac1n\sum_{i\in T}\xi_i^2
    &\le O\Paren{\frac{\eta\log(1/\eta)}m+\eta\omega^2}.
  \end{align*}
  Here the cross term in the second-moment estimate is at most
  $O(\omega\eta\smash[t]{\sqrt{\log(1/\eta)/m}})\le O(\omega^2+\eta^2\log(1/\eta)/m)$, which is absorbed by the $\omega^2$ and $\eta\log(1/\eta)/m$ terms.
  Substituting the mean and second-moment estimates in the proof of \Cref{lem:clean-cov-gap-subset} bounds the clean covariance discrepancy by $O(\eta\log(1/\eta)/m+\omega/m+\omega^2)$ in every subset direction. Moreover, the matrix $\Delta_{\mathrm{cov}}^\star$ defined in that lemma satisfies $\Delta_{\mathrm{cov}}^\star\bm1=0$ because each $y_i^*$ and $\bar\mu=\E_i y_i^*$ has coordinate sum one. Thus \Cref{prop:subset-control-majorizer} gives a diagonal majorizer with trace at most $\tau_{\mathrm{cov}}$, proving feasibility.

  For soundness, the clean-kept contribution in the proof of \Cref{lem:main-sos-lemma} is now $O(\eta^2\log(1/\eta)/m+\omega^2)$.
  Comparing the repaired and clean-kept second moments, as in \eqref{eq:robust-directional-identifiability}, gives
  \begin{equation*}
    \mathcal A_\eta\sststile{6}{}\Set{
      \Delta_s^2\le O\Paren{\frac{\eta^2\log(1/\eta)}m+\omega^2+\frac{\eta\omega}m}
      +4\eta\Paren{V(\mu',s)-V(\mu,s)}+4\eta\Delta_s^2}.
  \end{equation*}
  By Young's inequality, $\eta\omega/m$ is at most $(\omega^2+\eta^2/m^2)/2$.
  Applying \eqref{eq:robust-variance-difference-absorption} and absorbing $(1/4+4\eta)\Delta_s^2$ into the left-hand side gives
  $\Delta_s^2\le O(\eta^2\log(1/\eta)/m+\omega^2)$ at degree $6$.
  Taking pseudoexpectations and using total-variation duality completes the proof.
\end{proof}

\section{Private Estimation for Discrete Batches}
\label{sec:private-algorithm}

This section instantiates the reduction of Hopkins, Kamath, Majid, and Narayanan~\cite{HopkinsKMN23} on the batch-average dataset $\cY = \set{y_i}_{i \in [n]}$. It requires only a quasi-convex score on candidate means from \Cref{sec:robust-algorithm}, implemented here through low-degree pseudoexpectations for the diagonal-majorizer system.

\begin{samepage}
  \begin{theorem}[Pure-DP Estimation from Discrete Batches]
    \label{thm:main-pure-private}
    Let $\mu \in \Delta^{k-1}$.
    Let $\alpha, \beta \in (0, 1)$, $\eps > 0$, and $m \in \mathbb{N}$.
    Then there exists a polynomial-time $(\eps, 0)$-DP algorithm that takes in $n$ independent user batches, each of the form
    \begin{equation*}
      B_i = (U_{i,1},\ldots,U_{i,m}),
      \qquad
      U_{i,1},\ldots,U_{i,m} \stackrel{\mathrm{iid}}{\sim} \mu,
    \end{equation*}
    and outputs an estimate $\hat{\mu}$ such that $\normtv{\hat{\mu} - \mu} \le \alpha$, with probability at least $1-\beta$, provided that
    \begin{equation*}
      n \ge \tilde{O}\Paren{\frac{k + \log(1/\beta)}{\alpha^2 m} + \frac{k + \log(1/\beta)}{\alpha \sqrt{m} \eps} + \frac{k+ \log(1/\beta)}{\eps}} \mper
    \end{equation*}
  \end{theorem}
\end{samepage}

\begin{theorem}[Approximate-DP Estimation from Discrete Batches]
  \label{thm:main-approx-private}
  Let $\mu \in \Delta^{k-1}$.
  Let $\alpha, \beta, \delta \in (0, 1)$, $\eps > 0$, and $m \in \mathbb{N}$.
  Then there exists a polynomial-time $(\eps, \delta)$-DP algorithm that takes in $n$ independent user batches, each of the form
  \begin{equation*}
    B_i = (U_{i,1},\ldots,U_{i,m}),
    \qquad
    U_{i,1},\ldots,U_{i,m} \stackrel{\mathrm{iid}}{\sim} \mu,
  \end{equation*}
  and outputs an estimate $\hat{\mu}$ such that $\normtv{\hat{\mu} - \mu} \le \alpha$, with probability at least $1-\beta$, provided that
  \begin{equation*}
    n \ge \tilde{O}\Paren{k+\frac{k + \log(1/\beta)}{\alpha^2 m} + \frac{k + \log(1/\beta)}{\alpha \sqrt{m} \eps} + \frac{\log(1/\delta)+\log(1/\beta)}{\eps}} \mper
  \end{equation*}
\end{theorem}

Both privacy arguments use one score function. For the approximate-DP theorem, we also need a worst-case bound on the intermediate ratio $V_{\eta^*}(\cY)/V_{0.8\eta^*}(\cY)$; we obtain it from an $\ell_1$-diameter bound for the score-$\eta^* n$ set.

Our score function reuses the succinct diagonal-majorizer system introduced in \Cref{subsec:sos-based-robust-algo}. For a dataset $\cY = \set{y_i}_{i \in [n]}$ and trimming budget $T$, let $\cB_T(\cY)$ denote the system from \Cref{subsec:sos-based-robust-algo} on the observed batch averages $\cY$, with the selector/data-consistency, simplex, and covariance-majorizer blocks of $\cA_\eta$, but with the counting constraint changed to
\begin{equation*}
  \sum_{i=1}^n w_i \ge n-T \mper
\end{equation*}
Relative to $\cA_\eta$, this is the only structural change: the certificate may now discard up to $T$ observed batch averages.
Write
\begin{equation*}
  \widehat{\Sigma}' \coloneq \E_i (y_i' - \mu') \transpose{(y_i' - \mu')}, \qquad
  \Sigma(\nu) \coloneq \frac1m\Paren{\diag(\nu) - \nu \transpose{\nu}}, \qquad
  \tau_{\mathrm{cov}} \coloneq C \frac{\eta\log(1/\eta)}{m},
\end{equation*}
for the same sufficiently large universal constant $C$ as in \Cref{subsec:sos-based-robust-algo}.
The selector/data-consistency block decides which observed batch averages are kept and ties the repaired points to the observed data on those indices. Simplex constraints keep the repaired points and the pseudo-mean inside the simplex, while the covariance-majorizer block enforces the subset-direction variance control in \Cref{lem:diag-majorizer-subset}.

Given a query point $\nu$, we will consider \sodared{degree-$6$ pseudoexpectations} that satisfy
\begin{enumerate}
  \item $\pE \sdtstile{}{} \cB_T(\cY)$ at \sodared{degree $6$}.
  \item Closeness constraint: $\normo{\pE \mu' - \nu} \le \alpha$.
\end{enumerate}
Informally, the score of a candidate mean is the smallest trimming budget for which there exists a low-degree certificate consistent with that candidate. We first define certifiable means, then use them to define the score.

\newcommand{\scoredomain}{\Delta^{k-1} + \mathbb{B}_1^{k}(\tau n+ \alpha)}

\begin{definition}[$(\alpha, \tau, T)$-Certifiable Mean]
  \label{def:certifiable-mean}
  Let $\alpha, \tau \ge 0$, $k, n \in \mathbb{N}$, $T \in [0, n]$.
  Let $\cY = \Set{y_i}_{i \in [n]} \subseteq \Delta^{k-1}$.
  We call a point $\nu \in \scoredomain \subseteq \R^k$ an $(\alpha, \tau, T)$ certifiable mean for $\cY$ if there exists a linear functional $\cL$ which $\tau$-approximately (according to \Cref{def:approx-sos-satisfaction}) satisfies $\cB_T(\cY)$ \sodared{at degree $6$} such that, writing $\mu_{\cL} \coloneq \cL \mu'$,
  \begin{enumerate}
    \item $\normf{\cR(\cL)} \le R' + \tau$, where $R' = \poly(n, k)$ is sufficiently large and $\cR(\cL)$ is the matrix representation of the coefficients of $\cL$.
    \item The query point $\nu$ is close to the assignment given by the approximate pseudoexpectation $\cL$:
      \begin{equation*}
        \normo{\mu_{\cL} - \nu} \le \alpha + \tau T \mper
      \end{equation*}
  \end{enumerate}
  Moreover, we will call such a functional $\cL$ an $(\alpha, \tau, T)$-certificate for $(\cY, \nu)$.
\end{definition}

\begin{lemma}[Approximate Certificates Are Close to the Simplex]
  \label{lem:approx-certificate-simplex}
  Let $\cL$ be a linear functional that $\tau$-approximately satisfies $\cB_T(\cY)$, and let
  \begin{equation*}
    x \coloneq \cL \mu' \in \R^k \mper
  \end{equation*}
  Then for every $j \in [k]$,
  \begin{equation*}
    x_j \ge -\tau T,
    \qquad
    \Abs{\sum_{j=1}^k x_j - 1} \le \tau T \mper
  \end{equation*}
  Consequently, there exists a point $\Pi(x) \in \Delta^{k-1}$ such that
  \begin{equation*}
    \normo{x-\Pi(x)} \le C k \tau T
  \end{equation*}
  for an absolute constant $C$.
\end{lemma}
\begin{proof}
  For the coordinate lower bound, apply item~3 of \Cref{def:approx-sos-satisfaction} to the constraint $\mu_j' \ge 0$ with $p=1$. This gives
  \begin{equation*}
    x_j = \cL \mu_j' \ge -\tau T
  \end{equation*}
  for every $j \in [k]$.
  For the constraint $\sum_{j \in [k]} \mu_j' - 1 = 0$, item~4 of \Cref{def:approx-sos-satisfaction}, with $p=1$ and $\cL1=1$, gives
  \begin{equation*}
    \Abs{\sum_{j=1}^k x_j - 1}
    =
    \Abs{\cL\Brac{\sum_{j=1}^k \mu_j' - 1}}
    \le \tau T \mper
  \end{equation*}

  Set $y_j \coloneq \max\{x_j,0\}$ for every $j \in [k]$.
  Then
  \begin{equation*}
    \normo{x-y} = \sum_{j=1}^k \max\{-x_j,0\} \le k\tau T \mper
  \end{equation*}
  Also,
  \begin{equation*}
    1-\tau T \le \sum_{j=1}^k x_j \le \sum_{j=1}^k y_j \le \sum_{j=1}^k x_j + k\tau T \le 1 + (k+1)\tau T \mper
  \end{equation*}
  Let $s \coloneq \sum_{j=1}^k y_j$, and define
  \begin{equation*}
    \Pi(x) \coloneq
    \begin{cases}
      y/s & \text{if } s > 0, \\
      e_1 & \text{if } s = 0,
    \end{cases}
  \end{equation*}
  where $e_1$ is the first standard basis vector.
  Then $\Pi(x) \in \Delta^{k-1}$.
  If $s>0$, then
  \begin{equation*}
    \normo{y-\Pi(x)} = \sum_{j=1}^k y_j \Abs{1-\frac1s} = \Abs{s-1} \le (k+1)\tau T \mper
  \end{equation*}
  If $s=0$, then $y=0$ and hence $x_j \le 0$ for every $j$.
  Together with $\abs{\sum_j x_j - 1} \le \tau T$, this implies $\tau T \ge 1$, and therefore
  \begin{equation*}
    \normo{x-\Pi(x)} = \normo{x-e_1} \le 1 + k\tau T \le 2k\tau T \mper
  \end{equation*}
  For $s>0$,
  \begin{equation*}
    \normo{x-\Pi(x)}
    \le
    \normo{x-y} + \normo{y-\Pi(x)}
    \le
    (2k+1)\tau T
    \le
    3k\tau T \mper
  \end{equation*}
  So the claim holds with $C=3$.
\end{proof}

Henceforth choose
\begin{equation*}
  0<\tau \le \frac{\min\{\alpha,1/\sqrt m\}}{100 C k n} \mper
  \tageq\label{eq:private-tau-bound}
\end{equation*}
This bound permits an inverse-polynomial choice of $\tau$, which we may decrease further to meet the numerical error bound in \Cref{lem:certificate-soundness}. We now define the score using \Cref{def:certifiable-mean}.

\begin{definition}[Score Function]
  \label{def:score-function}
  Let $\Delta^{k-1}$ be the probability simplex (see \Cref{def:probability-simplex}).
  Consider $\scoredomain \subseteq \R^k$, the fattening of $\Delta^{k-1}$ by the $\ell_1$-ball of radius $\tau n + \alpha$ in $k$ dimensions.
  Let $\alpha, \tau \ge 0$, $\cY = \Set{y_i}_{i \in [n]} \subseteq \Delta^{k-1}$, and $\nu \in \scoredomain$. We define the score function $\cS: \scoredomain \to \R_{\ge 0}$ (viewed as a function of $\nu$) as

  \begin{equation*}
    \cS (\nu; \cY, \alpha, \tau) \coloneq \min_{T \ge 0} \text{ such that } \nu \text{ is a } (\alpha, \tau, T) \text{ certifiable mean for } \cY \mper
  \end{equation*}
\end{definition}
\noindent
For the rest of the section, we fix $\alpha$, $\tau$, and the dataset $\cY$, and abbreviate
\begin{equation*}
  \cS_{\cY}(\nu) \coloneq \cS(\nu;\cY,\alpha,\tau) \mper
\end{equation*}
We next verify the score properties required by Hopkins, Kamath, Majid, and Narayanan~\cite{HopkinsKMN23}: well-definedness, sensitivity, quasi-convexity, efficient computation, a low-scoring neighborhood, and bounds on sublevel-set volumes.

\Needspace{8\baselineskip}
\subsection{Properties of the Score}

\subsubsection{Well-Definedness of the Score Function}
\begin{lemma}[score function upper bound]
  \label{lem:score-function-upper-bound}
  Let $\alpha, \tau \ge 0$, $\cY = \Set{y_i}_{i \in [n]}$ be any set of inputs, and $\nu \in \scoredomain$. Then $\cS_{\cY}(\nu) \le n$, where $\cS$ is defined in \Cref{def:score-function}.
\end{lemma}
\begin{proof}
  Let $\nu$ be any point in $\scoredomain$.
  It suffices to find an $(\alpha, \tau, T)$ certificate $\cL$ for $(\cY, \nu)$, where $T = n$.

  Let $\mu_0 \coloneq \Proj_{\ell_1}(\nu; \Delta^{k-1})$, so that
  \begin{equation*}
    \normo{\nu - \mu_0} \le \alpha + \tau n = \alpha + \tau T
  \end{equation*}
  because $\nu \in \Delta^{k-1} + \mathbb{B}_1^{k}(\alpha + \tau n)$. Now let $\cL$ be the linear functional corresponding to the single-point distribution that sets
  \begin{equation*}
    \forall i \in [n]: \qquad w_i = 0, \qquad y_i' = \mu_0, \qquad \mu' = \mu_0, \qquad \lambda = 0 \mper
  \end{equation*}
  All constraints except the diagonal-majorizer equality are then immediate. Since every $y_i' = \mu_0$, we have $\smash[t]{\widehat{\Sigma}}' = 0$. Because $\mu_0 \in \Delta^{k-1}$, \Cref{fact:diag-minus-xxT-psd} implies that $\Sigma(\mu_0)$ is positive semidefinite, so we may choose $M$ such that
  \begin{equation*}
    MM^\top = \Sigma(\mu_0) \mper
  \end{equation*}
  Thus
  \begin{equation*}
    \widehat{\Sigma}' + MM^\top = \Sigma(\mu_0) = \Sigma(\mu') + \operatorname{Diag}(\lambda),
    \qquad
    \sum_{j \in [k]} \lambda_j = 0 \le \tau_{\mathrm{cov}} \mper
  \end{equation*}
  Since the entries of $\mu_0$ and $M$ have absolute value at most $1$, the degree-$6$ moment vector has polynomial norm and satisfies item~1 of \Cref{def:certifiable-mean} for sufficiently large $R'$. By construction, the closeness constraint holds, so $\cS_{\cY}(\nu) \le n$. Compactness of the feasible pairs $(T,\cL)$ ensures that the minimum in \Cref{def:score-function} is attained.
\end{proof}

\subsubsection{Bounded Sensitivity}
\begin{lemma}[Bounded Sensitivity]
  \label{lem:bounded_sensitivity}
  For every fixed $\nu \in \scoredomain$, the score function $\cY \mapsto \cS(\nu;\cY,\alpha,\tau)$ has sensitivity $1$ with respect to changing the input data $\cY = \Set{y_i}_{i \in [n]}$.
\end{lemma}
\begin{proof}
  Suppose $\cY = \Set{y_1, \dots, y_n}$ and $\cY' = \Set{\bar y_1, y_2, \dots, y_n}$ are two datasets that differ in a single batch average. Moreover, assume that for some $\nu \in \scoredomain$, $\cS_{\cY}(\nu) = T$. This means that there exists an $(\alpha, \tau, T)$ certificate $\cL$ for $(\cY, \nu)$. We will show that there exists a $(\alpha, \tau, T+1)$ certificate $\cL'$ for $(\cY', \nu)$.

  We define $\cL'$ as follows. For every monomial $p$ that does not involve $w_1$, we set $\cL'(p) = \cL(p)$. For monomials $p$ that involve $w_1$, we set $\cL'(p) = 0$.

  If $T+1\ge n$, the claim follows from \Cref{lem:score-function-upper-bound}, so suppose $T+1<n$. Substitution sets moment coordinates to zero, and hence cannot increase $\normf{\cR(\cL)}$. It preserves normalization and the pseudo-mean.
  For a polynomial $p$ with $\normt{\cR(p)}\le 1$, write $q=p|_{w_1=0}$, so $\normt{\cR(q)}\le1$.
  If a constraint polynomial $s$ does not involve $w_1$, then
  \begin{equation*}
    \cL'[s p^2]=\cL[s q^2].
  \end{equation*}
  Thus the constraints independent of $w_1$, square positivity, and the closeness constraint retain their error bounds at budget $T$. The Boolean and data-consistency constraints involving $w_1$ become identically zero. For the counting constraint in item~5 of \Cref{def:approx-sos-satisfaction},
  \begin{align*}
    \cL'\Brac{\Paren{\sum_{i=1}^n w_i-n+T+1}p^2}
    &=\cL\Brac{\Paren{\sum_{i\ne1}w_i-n+T+1}q^2} \\
    &=\cL\Brac{\Paren{\sum_{i=1}^n w_i-n+T}q^2}
    +\cL[(1-w_1)^2q^2]+\cL[w_1(1-w_1)q^2] \\
    &\ge -5\tau Tn-4\tau T-\tau T \\
    &\ge -5\tau(T+1)n.
  \end{align*}
  Here $\deg q\le2$, so every expression has degree at most $6$.
  The square term $\cL[(1-w_1)^2q^2]$ is at least $-4\tau T$ because
  $\normt{\cR((1-w_1)q)}\le2$; the term $\cL[w_1(1-w_1)q^2]$ is at least $-\tau T$ by item~4 of \Cref{def:approx-sos-satisfaction}, applied to $w_1^2-w_1=0$.
  Interchanging $\cY$ and $\cY'$ proves the sensitivity bound.
\end{proof}

\subsubsection{Quasi Convexity}
\begin{lemma}[Quasi Convexity]
  \label{lem:quasi_convexity}
  For every fixed dataset $\cY$, the score function $\nu \mapsto \cS_{\cY}(\nu)$ is quasi-convex on $\scoredomain$.
\end{lemma}
\begin{proof}
  Let $\nu_1, \nu_2 \in \scoredomain$ and $T_1, T_2 \in [0, n]$ with $\cS_{\cY}(\nu_1) = T_1$ and $\cS_{\cY}(\nu_2) = T_2$.
  For any $\lambda \in [0, 1]$ and $T = \cS_{\cY}(\lambda \nu_1 + (1 - \lambda) \nu_2)$, we want to show that $T \le \max\{T_1, T_2\}$.
  By \cref{def:score-function}, we get that $\nu_1$ is an $(\alpha, \tau, T_1)$-certifiable mean for $\cY$, and an analogous property holds for $\nu_2$.
  Choose certificates $\cL_1,\cL_2$ for $(\cY,\nu_1)$ and $(\cY,\nu_2)$ at budgets $T_1,T_2$, respectively, as in \Cref{def:certifiable-mean}.
  We construct $\cL = \lambda \cL_1 + (1 - \lambda) \cL_2$, and show that it certifies $\lambda\nu_1+(1-\lambda)\nu_2$ at budget $\max\{T_1,T_2\}$.
  The moment-norm bound, the closeness constraint, and all constraints independent of $T$ are convex in $\cL$ and their permitted errors increase with $T$. It remains to check the counting constraint in item~5 of \Cref{def:approx-sos-satisfaction}.
  Let $p$ be any polynomial with $\|\cR(p)\|_2 \le 1$.
  Assuming that $T_1 \le T_2$, we have that:
  \begin{align*}
    &\quad\ \cL\left[\left(\sum\limits_{i = 1}^n w_i + T_2 - n\right) p^2\right] \\
    &= \lambda \cL_1\left[\left[\sum\limits_{i = 1}^n w_i + T_1 + (T_2 - T_1) - n\right] p^2\right] + (1 - \lambda) \cL_2\left[\left(\sum\limits_{i = 1}^n w_i + T_2 - n\right) p^2\right] \\
    &= \lambda \cL_1\left[\left[\sum\limits_{i = 1}^n w_i + T_1 - n\right] p^2\right] + \lambda (T_2 - T_1) \cL_1 p^2 + (1 - \lambda) \cL_2\left[\left(\sum\limits_{i = 1}^n w_i + T_2 - n\right) p^2\right] \\
    &\geq \lambda (- 5 \tau T_1 n) + \lambda (T_2 - T_1) (- 5 \tau T_1) + (1 - \lambda) (- 5 \tau T_2 n) \\
    &\geq \lambda (- 5 \tau T_1 n) + \lambda [- 5 \tau n (T_2 - T_1)] + (1 - \lambda) (- 5 \tau T_2 n) \\
    &\geq - 5 \tau T_2 n,
  \end{align*}
  where we used that $T_1 \le n$ (\cref{lem:score-function-upper-bound}) in the second-to-last inequality.
\end{proof}

\subsubsection{Efficient Computability}
\begin{lemma}[Efficient Computation of Score Function]
  \label{lem:efficient_comp}
  Let $\nu$ be any point in $\scoredomain$. Let $T = \cS_{\cY}(\nu)$ be its score. Then we can compute $T$ to additive accuracy $\gamma$ in time polynomial in $n$, $k$, $\log(1/\gamma)$, $\log(1/\tau)$, and the input bit length.
\end{lemma}
\begin{proof}
  There are $\poly(n,k)$ moments of degree at most $6$. At fixed $T$, approximate positivity and the inequality constraints are positive-semidefinite constraints on moment and localizing matrices; each approximate equality is a pair of such constraints. Polynomial-time separation also applies to the moment-norm and $\ell_1$-closeness constraints. Feasible sets are nested in $T$, as follows from the counting-constraint calculation in \Cref{lem:quasi_convexity}.

  For completeness, after fixing the normalization $\cL1=1$, the positive tolerance supplies the interior needed for weak feasibility. Increasing $T$ by $\gamma$ increases the permitted error in each constraint by at least $\tau\gamma$; the counting constraint retains slack of order $\tau n\gamma$, using $T\le n$ and $\cL p^2\ge-\tau T$. If the moment-norm bound is tight, mix the functional with a sufficiently small multiple of the bounded point-mass functional from \Cref{lem:score-function-upper-bound}. Taking $R'$ larger than twice the norm of that functional creates norm slack, while changing every other constraint by at most a fixed polynomial times the mixing weight. A mixing weight of order $\tau\gamma/\poly(n,k)$ therefore leaves a ball of radius $\tau\gamma/\poly(n,k)$ in the feasible set at the enlarged budget. We apply the ellipsoid argument of Appendix~C in~\cite{HopkinsKMN23}, followed by binary search in $T$, to compute the score to additive error $\gamma$. Its bit complexity is polynomial also in the input bit length and $\log(1/\tau)$; throughout, $\tau$ is inverse-polynomial in the problem parameters.
\end{proof}

\Needspace{8\baselineskip}
\subsection{Problem-Specific Consequences}

\subsubsection{Efficiently Finding Low Scoring Points}
\begin{lemma}
  \label{lem:eff_low_score}
  There exists an algorithm, running in time polynomial in $n$, $k$, $\log(1/\tau)$, and the input bit length, that outputs $\hat{\mu}$ such that
  \begin{equation*}
    \cS_{\cY}(\hat{\mu}) \le \min_{\nu \in \scoredomain} \cS_{\cY}(\nu) + 1,
  \end{equation*}
  and every point within $\ell_1$ distance $\alpha/2$ of the returned $\hat{\mu}$ has score at most $\min_{\nu \in \scoredomain}\cS_{\cY}(\nu)+1$.
\end{lemma}
\begin{proof}
  Consider the certificate constraints in \Cref{def:certifiable-mean} with the closeness constraint in item~2 omitted.
  Without the closeness constraint, we may still apply the weak-feasibility argument in \Cref{lem:efficient_comp}. Binary search, with an additive budget allowance smaller than $1$, returns a feasible budget $T_{\min}$ and a corresponding functional $\cL$ with
  \begin{equation*}
    T_{\min}\le \min_{\nu\in\scoredomain}\cS_{\cY}(\nu)+1.
  \end{equation*}
  Indeed, dropping the closeness constraint can only decrease the optimum.
  Define
  \begin{equation*}
    \mu_{\cL} \coloneq \cL \mu',
    \qquad
    \hat{\mu} \coloneq \Pi(\mu_{\cL}) \mper
  \end{equation*}
  By \Cref{lem:approx-certificate-simplex} and \eqref{eq:private-tau-bound}, we have
  \begin{equation*}
    \normo{\hat{\mu} - \mu_{\cL}} \le Ck\tau T_{\min} \le Ck\tau n \le \alpha/100 \le \alpha/10 \mper
  \end{equation*}
  In particular, $\hat{\mu} \in \Delta^{k-1} \subseteq \scoredomain$.
  Thus $\cL$ certifies that $\hat{\mu}$ has score at most $T_{\min}$, since
  \begin{equation*}
    \normo{\hat{\mu} - \mu_{\cL}} \le \alpha/10 \le \alpha + \tau T_{\min} \mper
  \end{equation*}
  Now let $\nu$ be any point with $\normo{\nu - \hat{\mu}} \le \alpha/2$.
  Then
  \begin{equation*}
    \normo{\nu - \mu_{\cL}}
    \le
    \normo{\nu - \hat{\mu}} + \normo{\hat{\mu} - \mu_{\cL}}
    \le
    \frac{\alpha}{2} + \frac{\alpha}{10}
    < \alpha \mper
  \end{equation*}
  Hence $\cL$ also witnesses that $\nu$ has score at most $T_{\min}$.

  Both score bounds follow from $T_{\min}\le\min_{\nu\in\scoredomain}\cS_{\cY}(\nu)+1$.
\end{proof}

\begin{lemma}[Soundness of Low-Trimming Certificates]
  \label{lem:certificate-soundness}
  Let $\eta_0 \in (0,\etaabs)$.
  Assume $\tau_{\mathrm{cov}}\le C\eta_0\log(1/\eta_0)/m$. Suppose $\cY^* = \set{y_i^*}_{i \in [n]}$ is $\eta_0$-good, and $\cY = \set{y_i}_{i \in [n]}$ is an $\eta_0$-corruption of $\cY^*$.
  Let $\cL$ be an $(\alpha,\tau,T)$-certificate for $(\cY,\nu)$ with $T \le \eta_0 n$, and write $\mu_{\cL} \coloneq \cL \mu'$.
  Then
  \begin{equation*}
    \normo{\mu_{\cL} - \mu} \le O\Paren{\eta_0 \sqrt{\frac{\log(1/\eta_0)}{m}}} \mper
  \end{equation*}
\end{lemma}
\begin{proof}
  For exact certificates, this is exactly the deterministic soundness proof from \cref{lem:main-sos-lemma} and the proof of \Cref{thm:main-robust}, with trimming budget $T \le \eta_0 n$ in place of the robust theorem's corruption budget. \begin{sodarevision}For approximate certificates, each inference in this fixed degree-$6$ proof incurs at most $\tau\,\poly(n,k,m,\eta_0^{-1})T$ coefficient error. We choose the inverse-polynomial tolerance $\tau$ sufficiently small that
    \begin{equation*}
      \tau\,\poly(n,k,m,\eta_0^{-1})T
      =o\Paren{\frac{\eta_0^2\log(1/\eta_0)}{m}}.
    \end{equation*}
    Thus the perturbation is lower order than the squared directional-error bound in \Cref{lem:main-sos-lemma}. Applying that directional bound in both subset directions and using $|\sum_j(\mu_{\cL})_j-1|\le\tau T$ from \Cref{lem:approx-certificate-simplex} gives the claimed $\ell_1$ bound.\end{sodarevision}
\end{proof}

\subsubsection{Volume}

\begin{lemma}[Volume]
  \label{lem:vol}
  Assume $n \ge \Omega((k + \log(1/\beta))/(\eta^2 \log(1/\eta)))$.
  Let $\eta^* \ge \eta$ be a sufficiently small absolute constant.
  Let
  \begin{equation*}
    B_i^* = (U_{i,1}^*,\ldots,U_{i,m}^*),
    \qquad
    U_{i,1}^*,\ldots,U_{i,m}^* \stackrel{\mathrm{iid}}{\sim} \mu,
    \qquad
    y_i^* \coloneq \frac1m \sum_{j=1}^m e_{U_{i,j}^*}
  \end{equation*}
  for each $i \in [n]$, and let $\cY^* = \Set{y_1^*,\ldots, y_n^*}$.
  Let $\cY = \Set{y_i}_{i \in [n]}$ be any $\eta$-corruption of $\cY^*$.
  The following three statements hold with probability $1-\beta$ over the draw of $\cY^*$.
  \begin{enumerate}
    \item There exists a point $\mu_0$ such that every point at $\ell_1$ distance at most $\alpha$ from $\mu_0$ has score at most $\eta n$.
    \item All points which have score at most $\eta^* n$ are contained in an $\ell_1$-ball of radius $\alpha+O(1/\sqrt m)$ centered at the same point $\mu_0$.
    \item Simultaneously for every $\eta'\in[\eta,\eta^*]$, all points of score at most $\eta'n$ lie in the $\ell_1$-ball about $\mu_0$ of radius
      \begin{equation*}
        \alpha+O\Paren{\eta'\sqrt{\frac{\log(1/\eta')}{m}}}.
      \end{equation*}
  \end{enumerate}
\end{lemma}
\begin{proof}
  Take $\mu_0$ to be the empirical average of the true samples, $\bar{\mu}$. For each integer $0\le j\le\ceil{\log_2(\eta^*/\eta)}$, apply \Cref{cor:eta-good} at $\min\{2^j\eta,\eta^*\}$ with failure probability $\beta/2^{j+2}$. These events hold simultaneously with probability at least $1-\beta$. To check the sample size, observe that
  \begin{equation*}
    (2q)^2\log(1/(2q))\ge2q^2\log(1/q)
    \qquad (0<q\le1/8).
  \end{equation*}
  For the final, possibly shorter step, we therefore obtain
  \begin{equation*}
    \min\{2^j\eta,\eta^*\}^2\log\Paren{\frac1{\min\{2^j\eta,\eta^*\}}}\ge2^{j-1}\eta^2\log(1/\eta)
    \qquad(j\ge1).
  \end{equation*}
  Since $k+\log(2^{j+2}/\beta)=O(k+\log(1/\beta)+j)$ and $j+1=O(2^j)$, the sample-size hypothesis of this lemma supplies the good-sample event at every level $\min\{2^j\eta,\eta^*\}$. Work on their intersection. Let $r_i$ be the indicator of the uncorrupted samples. By \Cref{lem:covariance-feasibility}, there exists $\lambda$ such that
  \begin{equation*}
    \operatorname{Diag}(\lambda) - \Paren{
      \E_i (y_i^* - \bar{\mu}) \transpose{(y_i^* - \bar{\mu})}
      - \frac1m\Paren{\diag(\bar{\mu}) - \bar{\mu} \transpose{\bar{\mu}}}
    }
    \succcurlyeq 0
  \end{equation*}
  and $\sum_j \lambda_j \le O(\eta \log(1/\eta)/m) \le \tau_{\mathrm{cov}}$, where the last inequality holds by the choice of the constant in $\tau_{\mathrm{cov}}$. Therefore, taking
  \begin{equation*}
    w_i = r_i, \qquad y_i' = y_i^*, \qquad \mu' = \bar{\mu},
  \end{equation*}
  and letting $M$ be a square root of the positive semidefinite matrix
  \begin{equation*}
    \operatorname{Diag}(\lambda) - \Paren{
      \E_i (y_i^* - \bar{\mu}) \transpose{(y_i^* - \bar{\mu})}
      - \frac1m\Paren{\diag(\bar{\mu}) - \bar{\mu} \transpose{\bar{\mu}}}
    },
  \end{equation*}
  gives a feasible point-mass functional. The diagonal entries of the majorizer imply $\lambda_j\ge-1/(4m)$; together with its trace bound, this bounds each $\lambda_j$ and $\normf M$ by $\poly(k)$. Thus its degree-$6$ moments also satisfy item~1 of \Cref{def:certifiable-mean} for sufficiently large $R'$, and $\bar{\mu}$ is an $(\alpha,\tau,\eta n)$-certifiable mean. Hence every $\nu$ such that
  \begin{equation*}
    \normo{\nu - \bar{\mu}} \le \alpha
  \end{equation*}
  also has score at most $\eta n$, proving the first part.

  For the third part, fix $\eta'\in[\eta,\eta^*]$ and choose the smallest $j$ with $\min\{2^j\eta,\eta^*\}\ge\eta'$, so that $\min\{2^j\eta,\eta^*\}\le2\eta'$. The fixed covariance allowance obeys
  \begin{equation*}
    \tau_{\mathrm{cov}}=C\frac{\eta\log(1/\eta)}m
    \le C\frac{\min\{2^j\eta,\eta^*\}\log(1/\min\{2^j\eta,\eta^*\})}m.
  \end{equation*}
  The dataset is a $\min\{2^j\eta,\eta^*\}$-corruption of samples that are good at that level. Thus \Cref{lem:certificate-soundness} applies to every certificate of score at most $\eta'n$, giving
  \begin{align*}
    \normo{\cL\mu'-\mu}
    &\le O\Paren{\min\{2^j\eta,\eta^*\}\sqrt{\frac{\log(1/\min\{2^j\eta,\eta^*\})}m}}\\
    &\le O\Paren{\eta'\sqrt{\frac{\log(1/\eta')}m}}.
  \end{align*}
  Choose $\tau$ sufficiently small for soundness at every grid point and so that $\tau n\le\eta\smash[t]{\sqrt{\log(1/\eta)/m}}$. One inverse-polynomial choice suffices, since every level is at least $\eta$. Property~1 of the $\eta$-good event bounds $\normo{\bar\mu-\mu}$ by $O(\eta\smash[t]{\sqrt{\log(1/\eta)/m}})$. Adding these bounds and the certificate's closeness constraint proves the third part. Taking $\eta'=\eta^*$ gives the second part.
\end{proof}

\begin{lemma}[Lower Volume at Level $0.8\eta^* n$]
  \label{lem:approx_dp_vol_lb}
  Let $\cY=\Set{y_i}_{i\in[n]}$ and let $\eta^*$ be a sufficiently small absolute constant. Suppose $\cL_0$ is an $(\alpha,\tau,0.7\eta^*n)$-certificate. Set
  \begin{equation*}
    \mu^\circ=\Pi(\cL_0\mu').
  \end{equation*}
  Every point $\nu$ with $\normo{\nu-\mu^\circ}\le\alpha-Ck\tau n$ has score at most $0.8\eta^*n$, where $C$ is the constant in \Cref{lem:approx-certificate-simplex}.
\end{lemma}
\begin{proof}
  The ball lies in $\scoredomain$ because $\mu^\circ\in\Delta^{k-1}$.
  By \Cref{lem:approx-certificate-simplex},
  \begin{equation*}
    \normo{\nu-\cL_0\mu'}
    \le\normo{\nu-\mu^\circ}+\normo{\mu^\circ-\cL_0\mu'}
    \le\alpha-Ck\tau n+Ck\tau n=\alpha.
  \end{equation*}
  Thus this certificate gives score at most $0.7\eta^*n$, and hence at most $0.8\eta^*n$.
\end{proof}

\begin{lemma}[Worst-Case $\ell_1$ Diameter of the Score-$\eta^* n$ Set]
  \label{lem:approx_dp_vol_ub}
  Let $\eta^*$ be a sufficiently small absolute constant and let $\tau_{\mathrm{cov}}\ge0$.
  If $\nu_1, \nu_2 \in \scoredomain$ both have score at most $\eta^* n$, then
  \begin{equation*}
    \normo{\nu_1 - \nu_2} \le O\Paren{\alpha+\sqrt{\frac1m+\tau_{\mathrm{cov}}}}.
  \end{equation*}
  More precisely, any corresponding certificates $\cL_1,\cL_2$ satisfy
  \begin{equation*}
    \normo{\cL_1\mu'-\cL_2\mu'}\le O\Paren{\sqrt{\frac1m+\tau_{\mathrm{cov}}}}.
  \end{equation*}
\end{lemma}
\begin{proof}
  See \Cref{app:approx-dp} for the full proof.
\end{proof}

\begin{corollary}[Intermediate Volume Ratio]
  \label{cor:approx_dp_vol_ratio}
  Let $\cY=\Set{y_i}_{i\in[n]}$, let $\eta^*$ be a sufficiently small absolute constant, and let $\tau_{\mathrm{cov}}\ge0$. Suppose there is a point of score at most $0.7\eta^*n$. Then
  \begin{equation*}
    \log\Paren{\frac{V_{\eta^*}(\cY)}{V_{0.8\eta^*}(\cY)}}
    \le O\Paren{k\log\Paren{1+\frac{\sqrt{1/m+\tau_{\mathrm{cov}}}}{\alpha}}}.
  \end{equation*}
\end{corollary}
\begin{proof}
  Choose a certificate $\cL_0$ for a point of score at most $0.7\eta^*n$, and let $\mu^\circ=\Pi(\cL_0\mu')$. Let $C$ be the constant in \Cref{lem:approx-certificate-simplex}.
  By \Cref{lem:approx_dp_vol_lb}, the score-$0.8\eta^*n$ set contains the $\ell_1$ ball of radius $\alpha-Ck\tau n$ centered at $\mu^\circ$.
  For any point $\nu$ of score at most $\eta^*n$, choose a certificate $\cL$. The pseudo-mean diameter bound in \Cref{lem:approx_dp_vol_ub} gives
  \begin{align*}
    \normo{\nu-\mu^\circ}
    &\le\normo{\nu-\cL\mu'}+\normo{\cL\mu'-\cL_0\mu'}+\normo{\cL_0\mu'-\mu^\circ} \\
    &\le\alpha+\tau n+O\Paren{\sqrt{\frac1m+\tau_{\mathrm{cov}}}}+Ck\tau n.
  \end{align*}
  Choosing $\tau$ as in \eqref{eq:private-tau-bound} ensures $Ck\tau n\le\min\{\alpha/100,1/(100\sqrt m)\}$ and $\tau n\le Ck\tau n$. Comparing volumes of the two $\ell_1$ balls therefore gives
  \begin{equation*}
    \frac{V_{\eta^*}(\cY)}{V_{0.8\eta^*}(\cY)}
    \le\Paren{\frac{\alpha+O\Paren{\sqrt{\frac1m+\tau_{\mathrm{cov}}}}+2Ck\tau n}{\alpha-Ck\tau n}}^k
    \le\Paren{1+\frac{C'\sqrt{1/m+\tau_{\mathrm{cov}}}}{\alpha}}^k
  \end{equation*}
  for an absolute constant $C'$, proving the claim.
\end{proof}

\Cref{alg:private-discrete-batches} summarizes the private estimator obtained from the robustness-to-privacy reduction.

\begin{boxedalgorithm}{Private Estimation from Discrete Batches}
  \label{alg:private-discrete-batches}
  \alginput User batches $\set{B_i}_{i \in [n]}$, target accuracy $\alpha$, failure probability $\beta$, and privacy parameters $(\eps,0)$ or $(\eps,\delta)$.
  \algoutput A private estimate $\hat{\mu}$ of the underlying distribution.
  \algoperations
  \begin{algsteps}
    \item If $\eps>1$, return the estimate from the pure-DP construction in \Cref{subsec:pure-dp-proof}.
    \item If approximate DP is requested and $\delta>1-\beta/2$, return the empirical distribution of all $nm$ samples with probability $\delta$, and the fixed distribution $e_1$ otherwise.
    \item Replace $\alpha$ by $\alpha/2$ in the remaining steps.
    \item Convert each observed batch to its average vector
      \begin{equation*}
        y_i \coloneq \frac{1}{m}\sum_{j=1}^m e_{U_{i,j}},
        \qquad
        \cY \coloneq \Set{y_i}_{i \in [n]} \mper
      \end{equation*}
    \item Define the score function $\cS_{\cY}(\nu)$ from \Cref{def:score-function}: the smallest trimming budget $T$ for which $\nu$ admits an $(\alpha,\tau,T)$-certificate via the diagonal-majorizer system $\cB_T(\cY)$.
    \item Set the internal trimming parameter
      \begin{equation*}
        \etapriv \coloneq \min\Set{\etaabs/20,\ \tilde{\Theta}(\alpha\sqrt{m})}
      \end{equation*}
      and instantiate the score system with
      \begin{equation*}
        \tau_{\mathrm{cov}} \coloneq C \frac{\etapriv\log(1/\etapriv)}{m}.
      \end{equation*}
      Run the reduction from \Cref{thm:pure_dp_reduction} (for pure DP) or \Cref{thm:approx-dp-reduction} (for approximate DP) on the domain $\Theta = \scoredomain$ with score function $\cS_{\cY}$.
    \item Output a nearest point in $\Delta^{k-1}$ to the returned point in $\ell_1$ distance.
  \end{algsteps}
\end{boxedalgorithm}

For the proofs of \Cref{thm:main-pure-private,thm:main-approx-private}, fix $\eta^*=\etaabs/2$ and instantiate the score system with the capped internal parameter
\begin{equation*}
  \etapriv \coloneq \min\Set{\etaabs/20,\ \tilde{\Theta}(\alpha\sqrt{m})},
  \qquad
  \tau_{\mathrm{cov}} \coloneq C \frac{\etapriv\log(1/\etapriv)}{m},
\end{equation*}
so that $10\etapriv\le\eta^*$, and where the hidden constant in $\tilde{\Theta}(\alpha\sqrt{m})$ is chosen small enough that whenever $\etapriv < \etaabs/20$, we have
\begin{equation*}
  O\Paren{\etapriv \sqrt{\frac{\log(1/\etapriv)}{m}}} \le \alpha/10 \mper
\end{equation*}
Equivalently, after fixing the hidden constant once and for all, there is an absolute split threshold such that:
\begin{enumerate}
  \item if $\alpha\sqrt{m}$ is below that threshold, then $\etapriv = \tilde{\Theta}(\alpha\sqrt{m})$;
  \item if $\alpha\sqrt{m}$ is above that threshold, then $\etapriv = \etaabs/20$, so $2\etapriv < 1/10$ and
    \begin{equation*}
      O\Paren{\etapriv \sqrt{\frac{\log(1/\etapriv)}{m}}}
      =
      O\Paren{\frac{1}{\sqrt{m}}}
      \le
      \alpha/10 \mper
    \end{equation*}
\end{enumerate}

\subsubsection{Proof of the Pure-DP Theorem}
\label{subsec:pure-dp-proof}

First suppose $0<\eps\le1$. We allocate failure probability $\beta/3$ to the $2\etapriv$-good event, the volume event of \Cref{lem:vol}, and the privacy mechanism; this changes only absolute constants and logarithmic factors. We apply \Cref{thm:pure_dp_reduction}, using the score function defined in \Cref{def:score-function}, with $\Theta = \scoredomain$, internal score threshold $\etapriv$, and Euclidean radius $\alpha/(2\smash[t]{\sqrt{k}})$.
The $\ell_1$ ball in \Cref{lem:eff_low_score} contains a Euclidean ball of this radius. All score properties required by \Cref{thm:pure_dp_reduction} follow from \Cref{lem:score-function-upper-bound,lem:bounded_sensitivity,lem:quasi_convexity,lem:efficient_comp,lem:eff_low_score}. It remains to verify the low-score witness, bound the volume term, and invoke soundness.

Both the good-sample event from \Cref{cor:eta-good} at parameter $2\etapriv$ and the event in \Cref{lem:vol} at parameter $\etapriv$ hold once
\begin{equation*}
  n \ge \Omega\Paren{\frac{k + \log(1/\beta)}{\etapriv^2 \log(1/\etapriv)}} \mper
  \tageq\label{eq:private-sample-feasibility}
\end{equation*}
In the below-threshold regime $\etapriv = \tilde{\Theta}(\alpha\sqrt{m})$, the requirement in \eqref{eq:private-sample-feasibility} becomes
\begin{equation*}
  \tilde{O}\Paren{\frac{k + \log(1/\beta)}{\alpha^2 m}} \mper
\end{equation*}
In the above-threshold regime $\etapriv = \etaabs/20$, this requirement is only $O(k+\log(1/\beta))$, which is absorbed by the $\smash[t]{\widetilde O}((k+\log(1/\beta))/\eps)$ term when $\eps\le1$.

For the volume term, let $\mu_0$ be the point from the first part of \Cref{lem:vol} at parameter $\etapriv$.
Because the witnessing certificate there has pseudo-mean exactly $\mu_0$, every point within $\ell_1$ distance $\alpha$ of $\mu_0$ also has score at most $\etapriv n$.
We use \Cref{lem:refined-private-utility}; the $\alpha$-radius $\ell_1$ ball about $\mu_0$ contains a Euclidean ball of radius $\alpha/(2\smash[t]{\sqrt k})$. Split the maximization in \eqref{eq:refined-utility-sample-size} into $\eta'\in[\etapriv,\eta^*]$ and $\eta'\in[\eta^*,1]$. For the first interval, item~3 of \Cref{lem:vol} gives
\begin{equation*}
  \log\Paren{\frac{V_{\eta'}(\cY)}{V_{\etapriv}(\cY)}}
  \le k\log\Paren{1+C\frac{\eta'\sqrt{\log(1/\eta')}}{\alpha\sqrt m}}
  \le k\log\Paren{1+C'\frac{\eta'}{\etapriv}},
\end{equation*}
where the last step uses $\etapriv\sqrt{\log(1/\etapriv)}=O(\alpha\sqrt m)$. Using $\sup_{u\ge1}\log(1+C'u)/u=O(1)$, we obtain
\begin{equation*}
  \max_{\etapriv\le\eta'\le\eta^*}
  \frac{\log(V_{\eta'}(\cY)/V_{\etapriv}(\cY))+\log(2/\beta)}{\eps\eta'}
  \le O\Paren{\frac{k+\log(2/\beta)}{\eps\etapriv}}.
  \tageq\label{eq:private-low-score-utility}
\end{equation*}
The right-hand side of \eqref{eq:private-low-score-utility} is $\smash[t]{\widetilde O}((k+\log(2/\beta))/(\alpha\smash[t]{\sqrt m}\eps))$ below the split threshold, and $O((k+\log(2/\beta))/\eps)$ above it.
For $\eta'\ge\eta^*$, bounding $V_{\eta'}$ by the volume $V_1$ of the full score domain gives
\begin{equation*}
  \frac{\log (V_{1}(\cY) / V_{\etapriv}(\cY)) + \log(2/\beta)}{\eps \cdot \eta^*}
  \le
  \tilde{O}\Paren{\frac{k + \log(1/\beta)}{\eps}} \mper
\end{equation*}

Therefore, as long as
\begin{equation*}
  n \ge \tilde{O}\Paren{\frac{k + \log(1/\beta)}{\alpha^2 m}
    +
    \frac{k + \log(1/\beta)}{\alpha \sqrt{m} \eps}
    +
    \frac{k + \log(1/\beta)}{\eps}},
\end{equation*}
the sampler of \Cref{thm:pure_dp_reduction}, with utility analyzed in \Cref{lem:refined-private-utility}, returns a point $\nu$ with score at most $2\etapriv n$ with probability at least $1-\beta$.

Let $\cL$ be a corresponding $(\alpha, \tau, 2\etapriv n)$-certificate for $(\cY,\nu)$.
Since $2\etapriv \le \etaabs/10 < \etaabs$, \Cref{cor:eta-good} and \Cref{lem:certificate-soundness} with parameter $\eta_0 = 2\etapriv$ give
\begin{equation*}
  \normo{\cL \mu' - \mu}
  \le
  O\Paren{\etapriv \sqrt{\frac{\log(1/\etapriv)}{m}}}
  \le
  \alpha/10 \mper
\end{equation*}
Using the closeness constraint and taking $\tau$ sufficiently small,
\begin{equation*}
  \normo{\nu-\mu}
  \le\normo{\nu-\cL\mu'}+\normo{\cL\mu'-\mu}
  \le\alpha+2\tau\etapriv n+\alpha/10
  \le6\alpha/5.
\end{equation*}
If $\hat\mu$ is a nearest point to $\nu$ in $\Delta^{k-1}$ in $\ell_1$ distance, then
\begin{equation*}
  \normtv{\hat\mu-\mu}
  =\tfrac12\normo{\hat\mu-\mu}
  \le\normo{\nu-\mu}
  \le6\alpha/5.
  \tageq\label{eq:private-projected-error}
\end{equation*}
Running the construction with internal accuracy $\alpha/2$ gives the claimed error $\alpha$ and changes the sample bound by only an absolute factor. Projection is deterministic postprocessing and preserves privacy.
To cover $\eps>1$, first suppose $\eps\le m$. Partition each user's batch into $\ceil{\eps}$ disjoint subbatches of size $\floor{m/\ceil{\eps}}$, discarding any remainder, and run the construction just proved for privacy parameters at most $1$ on the $n\ceil{\eps}$ subbatches with privacy parameter $\eps/\ceil{\eps}\le1$. Changing one original batch changes at most $\ceil{\eps}$ subbatches, so group privacy gives $\eps$-DP. Since $\floor{m/\ceil{\eps}}\ge m/(2\ceil{\eps})$ and $\ceil{\eps}\le2\eps$, the required number of original users is
\begin{equation*}
  \widetilde O\Paren{
    \frac{k+\log(1/\beta)}{\alpha^2m}
    +\frac{k+\log(1/\beta)}{\alpha\sqrt{m\eps}}
    +\frac{k+\log(1/\beta)}{\eps}}.
\end{equation*}
By the arithmetic--geometric mean inequality, the middle term is at most the sum of the other two. If $\eps>m$, apply the Laplace mechanism to the pooled empirical distribution, whose $\ell_1$ sensitivity under replacement of one batch is $2/n$, and project to the simplex. This uses
\begin{equation*}
  \widetilde O\Paren{\frac{k+\log(1/\beta)}{\alpha^2m}
    +\frac{k+\log(1/\beta)}{\alpha\eps}}
\end{equation*}
users; its second term is absorbed by the first because $\eps>m$ and $\alpha<1$. This proves \Cref{thm:main-pure-private}.

\subsubsection{Proof of the Approximate-DP Theorem}
\label{subsec:approx_dp}

If $\delta>1-\beta/2$, release the empirical distribution with probability $\delta$ and a fixed distribution otherwise. This is $(0,\delta)$-DP. With $n=O((k+\log(1/\beta))/(\alpha^2m))$, the empirical distribution has error at most $\alpha$ with probability at least $1-\beta/2$, so the released estimate succeeds with probability at least $\delta(1-\beta/2)> (1-\beta/2)^2\ge1-\beta$. We may therefore assume $\delta\le1-\beta/2$. Then $\beta\delta\le1/2$, and the additive constant in $\log(2/(\beta\delta))$ is absorbed by $\log(1/\delta)+\log(1/\beta)$.

First suppose $0<\eps\le1$, and apply \Cref{thm:approx-dp-reduction} with the score, internal threshold $\etapriv$, and Euclidean radius $\alpha/(2\smash[t]{\sqrt k})$ from \Cref{subsec:pure-dp-proof}. That section verifies the oracle properties, and item~1 of \Cref{lem:vol} supplies the low-score witness.

To obtain the good-sample event used in \Cref{subsec:pure-dp-proof}, we need
\begin{equation*}
  n\ge\widetilde O\Paren{k+\log(1/\beta)+\frac{k+\log(1/\beta)}{\alpha^2m}},
  \tageq\label{eq:private-approx-feasibility}
\end{equation*}
because $\etapriv=\min\{\etaabs/20,\widetilde\Theta(\alpha\smash[t]{\sqrt m})\}$. For the utility condition \eqref{eq:refined-utility-sample-size} on this event, the balls in items~1 and~2 of \Cref{lem:vol} give
\begin{equation*}
  \log\Paren{\frac{V_{\eta'}(\cY)}{V_{\etapriv}(\cY)}}
  \le O\Paren{k\log\Paren{1+\frac1{\alpha\sqrt m}}}
  \qquad (\etapriv\le\eta'\le\eta^*).
  \tageq\label{eq:private-coarse-volume}
\end{equation*}
Apply \Cref{lem:refined-private-utility}, using the inner ball in item~1 of \Cref{lem:vol}. Below the fixed split threshold, \eqref{eq:private-low-score-utility} gives $O((k+\log(2/\beta))/(\eps\etapriv))$. Above that threshold, $\etapriv$ is constant and $\log(1+1/(\alpha\sqrt m))\le1/(\alpha\sqrt m)$, so \eqref{eq:private-coarse-volume} bounds the utility requirement by
\begin{equation*}
  \widetilde O\Paren{\frac{k}{\alpha\sqrt m\eps}+\frac{\log(2/\beta)}{\eps}}.
\end{equation*}
Thus \eqref{eq:refined-utility-sample-size} is implied by
\begin{equation*}
  n\ge\widetilde O\Paren{\frac{k+\log(1/\beta)}{\alpha\sqrt m\eps}+\frac{\log(2/\beta)}{\eps}}.
\end{equation*}

For the worst-case volume condition \eqref{eq:approx-reduction-privacy}, choose $\etaabs$ small enough that $m\tau_{\mathrm{cov}}\le1/2$. Then \Cref{cor:approx_dp_vol_ratio} gives
\begin{equation*}
  n\ge O\Paren{\frac{\log(1/\delta)+k\log(1+1/(\alpha\sqrt m))}{\eps}}
\end{equation*}
as a sufficient condition. Since $\log(1+1/(\alpha\sqrt m))\le1/(\alpha\sqrt m)$, this requirement is absorbed by the sample bound in \Cref{thm:main-approx-private}. Finally, $k+(\log(1/\delta)+\log(1/\beta))/\eps\ge k+\log(1/\beta)$ for $\eps\le1$, so that theorem's sample bound also supplies \eqref{eq:private-approx-feasibility}.

Allocating failure probability $\beta/3$ to the sample events and the reduction, we obtain a point $\nu$ with score at most $2\etapriv n$ with probability at least $1-\beta$. Applying \Cref{lem:certificate-soundness} at $2\etapriv$ and projecting as in \eqref{eq:private-projected-error} gives error at most $\alpha$, after running the construction with internal accuracy $\alpha/2$.

If $\eps>1$, use the pure-DP estimator from \Cref{thm:main-pure-private}. Its last sample term satisfies
\begin{equation*}
  \frac{k+\log(1/\beta)}{\eps}
  \le k+\frac{\log(1/\delta)+\log(1/\beta)}{\eps},
\end{equation*}
so it also satisfies the sample bound in \Cref{thm:main-approx-private} and is $(\eps,\delta)$-DP. This proves \Cref{thm:main-approx-private}.

\begin{remark}[Explicit logarithmic factors]
  \label{rem:explicit-private-logs}
  For $0<\eps\le1$, the proofs of \Cref{thm:main-pure-private,thm:main-approx-private} give sufficient sample bounds with ordinary $O$ notation. For pure DP, the bound is
  \begin{equation*}
    n=O\Paren{\frac{k+\log(2/\beta)}{\alpha^2m}
      +\frac{(k+\log(2/\beta))\sqrt{\log(e+1/(\alpha\sqrt m))}}
      {\alpha\sqrt m\eps}
      +\frac{k\log(e/\alpha)+\log(2/\beta)}\eps}.
  \end{equation*}
  For approximate DP, the bound is
  \begin{equation*}
    n=O\Paren{k+\frac{k+\log(2/\beta)}{\alpha^2m}
      +\frac{(k+\log(2/\beta))\sqrt{\log(e+1/(\alpha\sqrt m))}}
      {\alpha\sqrt m\eps}
      +\frac{\log(1/\delta)+\log(2/\beta)}\eps}.
  \end{equation*}
  Indeed, choose
  \begin{equation*}
    \etapriv=c\min\Set{1,
      \frac{\alpha\sqrt m}{\sqrt{\log(e+1/(\alpha\sqrt m))}}}
  \end{equation*}
  with $c>0$ a sufficiently small absolute constant. Then $\etapriv^2\log(1/\etapriv)=\Theta(\min\{1,\alpha^2m\})$, so \eqref{eq:private-sample-feasibility} requires $O(k+\log(2/\beta)+(k+\log(2/\beta))/(\alpha^2m))$ users. Below the fixed split threshold, \eqref{eq:private-low-score-utility} requires
  \begin{equation*}
    O\Paren{\frac{k+\log(2/\beta)}{\eps\etapriv}}
    =O\Paren{
      \frac{(k+\log(2/\beta))\sqrt{\log(e+1/(\alpha\sqrt m))}}
      {\alpha\sqrt m\eps}}.
  \end{equation*}
  Above the split threshold, \eqref{eq:private-coarse-volume} gives $O(k/(\alpha\sqrt m\eps)+\log(2/\beta)/\eps)$. Pure DP may also use the $O((k+\log(2/\beta))/\eps)$ bound, which is absorbed by $(k\log(e/\alpha)+\log(2/\beta))/\eps$. For pure DP, the full score domain lies in a fixed-radius $\ell_1$ ball, giving $\log(V_1/V_{\etapriv})=O(k\log(e/\alpha))$. For approximate DP, \eqref{eq:approx-reduction-privacy} contributes
  \begin{equation*}
    O\Paren{\frac{\log(1/\delta)+k\log(1+1/(\alpha\sqrt m))}\eps}.
  \end{equation*}
  Its dimension-dependent term is absorbed by
  \begin{equation*}
    \frac{(k+\log(2/\beta))\sqrt{\log(e+1/(\alpha\sqrt m))}}
    {\alpha\sqrt m\eps}.
  \end{equation*}
  Thus the privacy term for small $\alpha\sqrt m$ contains only a square-root logarithm.
\end{remark}

\begin{remark}[Larger batches]
  \label{rem:approx-batch-splitting}
  For clean data, splitting batches gives a refinement of \Cref{thm:main-approx-private}.
  Suppose $0<\eps\le1$, $\beta\le1/3$, and $\delta\le1/2$.
  There is a polynomial-time $(\eps,\delta)$-DP estimator with the same accuracy and failure probability using
  \begin{equation*}
    \begin{aligned}
      n=\widetilde O\bigg(&
      \frac{k+\log(1/\beta)}{\alpha^2m}
      +\frac{k+\log(1/\beta)}{\alpha\sqrt m\eps}\\
      &+\min\Set{k,\frac{k+\log(1/\beta)}{(\alpha\sqrt m\eps)^{2/3}}}\\
      &+\frac{\log(1/\delta)+\log(1/\beta)}{\eps}\bigg).
    \end{aligned}
  \end{equation*}
  When $\alpha\sqrt m\eps\le1$, this follows directly from \Cref{thm:main-approx-private}. Otherwise, partition each user's batch into $\ceil{(\alpha\sqrt m\eps)^{2/3}}$ disjoint blocks, each of size
  \begin{equation*}
    \floor{\frac{m}{\ceil{(\alpha\sqrt m\eps)^{2/3}}}},
  \end{equation*}
  discarding any remaining samples. The number of blocks per user is at most $m$ because $\alpha\eps\le1$.
  Apply \Cref{thm:main-approx-private} to the $n\ceil{(\alpha\sqrt m\eps)^{2/3}}$ independent blocks with privacy parameters
  \begin{equation*}
    \left(
    \frac{\eps}{\ceil{(\alpha\sqrt m\eps)^{2/3}}},\quad
    \frac{\delta}{e^\eps\ceil{(\alpha\sqrt m\eps)^{2/3}}}
    \right).
  \end{equation*}
  Changing one original user changes at most $\ceil{(\alpha\sqrt m\eps)^{2/3}}$ blocks, so group privacy gives $(\eps,\delta)$-DP: the accumulated additive privacy error is at most
  \begin{equation*}
    \frac{\delta}{e^\eps\ceil{(\alpha\sqrt m\eps)^{2/3}}}
    \sum_{j=0}^{\ceil{(\alpha\sqrt m\eps)^{2/3}}-1}
    \exp\Paren{\frac{j\eps}{\ceil{(\alpha\sqrt m\eps)^{2/3}}}}
    \le\delta.
  \end{equation*}
  Each block has size at least $m/(2\ceil{(\alpha\sqrt m\eps)^{2/3}})$. Dividing the required number of blocks by $\ceil{(\alpha\sqrt m\eps)^{2/3}}$ gives
  \begin{equation*}
    \begin{aligned}
      n=\widetilde O\bigg(&
      \frac{k}{\ceil{(\alpha\sqrt m\eps)^{2/3}}}
      +\frac{k+\log(1/\beta)}{\alpha^2m}\\
      &+\frac{(k+\log(1/\beta))\sqrt{\ceil{(\alpha\sqrt m\eps)^{2/3}}}}
      {\alpha\sqrt m\eps}\\
      &+\frac{\log(1/\delta)+\log(1/\beta)}\eps\\
      &+\frac{\log\ceil{(\alpha\sqrt m\eps)^{2/3}}}\eps\bigg).
    \end{aligned}
  \end{equation*}
  Since $\alpha\sqrt m\eps>1$, rounding the number of blocks up changes it by at most a factor of two. Taking the better of this bound and the bound of \Cref{thm:main-approx-private} applied to the original batches proves the claim.
  Here $\smash[t]{\widetilde O}$ suppresses the additional factor $\log\ceil{(\alpha\sqrt m\eps)^{2/3}}$ but no additional dependence on $\delta$.
  For constant $\beta$ and $\alpha\sqrt m\eps>1$, the fixed $k$ term is thus replaced by $k/(\alpha\sqrt m\eps)^{2/3}$.
  This argument does not preserve the guarantee of \Cref{cor:private-robustness}: splitting reduces the batch size that controls the corruption error.

  For clean data, we may also choose between this estimator and that of Esfandiari et al.~\cite{NarayananME22} using only the public parameters, obtaining the smaller of the two sample bounds without any privacy composition. In particular, for constant $\beta$ and $m\ge Ck^2/\alpha^2$, their ball-localization algorithm gives $n=O(\log(1/\delta)/\eps)$ users. Indeed, each user's empirical distribution satisfies $\E\normt{y_i-\mu}^2\le1/m$, so Markov's inequality and a Chernoff bound place at least two-thirds of the users within $O(1/\sqrt m)$ of $\mu$ in $\ell_2$. Applying~\cite[Theorem~3.3]{NarayananME22}, followed by Euclidean projection onto the simplex, gives $\ell_2$ error $O(\smash[t]{\sqrt{k/m}})$, hence total-variation error $O(k/\sqrt m)$. For general batch sizes, their few-user bound has the additional logarithmic dependence on $1/\delta$ stated in~\cite[Theorem~4.1]{NarayananME22}.
\end{remark}

\begin{corollary}[Robustness of the Private Estimators]
  \label{cor:private-robustness}
  There are absolute constants $c,C>0$ for which the following holds.
  Let $0<\eps\le1$, $0\le\eta\le c$, and $\alpha,\beta\in(0,1)$.
  Suppose the clean batches are independent and batch $i$ consists of $m$ i.i.d.\ samples from $\mu_i$, where
  $\normtv{\mu_i-\mu}\le\omega$ for every $i$ and a known bound $\omega\in[0,1]$.
  There is a polynomial-time, pure $\eps$-DP estimator that, on any $\eta$-corruption of these batches, returns $\hat\mu\in\Delta^{k-1}$ with
  $\normtv{\hat\mu-\mu}\le\alpha$ with probability at least $1-\beta$, provided
  \begin{equation*}
    \alpha\ge C\Paren{\eta\sqrt{\frac{\log(1/\eta)}m}+\omega}
  \end{equation*}
  and
  \begin{equation*}
    n\ge\widetilde O\Paren{
      \frac{k+\log(1/\beta)}{\alpha^2m}
      +\frac{k+\log(1/\beta)}{\alpha\sqrt m\,\eps}
      +\frac{k+\log(1/\beta)}\eps}.
  \end{equation*}
  The term involving $\eta$ is defined to be zero when $\eta=0$.
  For every $\delta\in(0,1)$, there is also a polynomial-time $(\eps,\delta)$-DP estimator with this accuracy guarantee, provided
  \begin{equation*}
    n\ge\widetilde O\Paren{
      k+\frac{k+\log(1/\beta)}{\alpha^2m}
      +\frac{k+\log(1/\beta)}{\alpha\sqrt m\,\eps}
      +\frac{k\omega/\alpha+\log(1/\delta)+\log(1/\beta)}\eps}.
  \end{equation*}
  This recovers the sample bound of \Cref{thm:main-approx-private} when $\omega=O(1/\sqrt m+\alpha\eps)$; in particular, it does so when $\alpha\sqrt m\le1$ or when $\eps$ is constant.
\end{corollary}
\begin{proof}
  We first prove the pure-DP claim. Use the score construction from \Cref{sec:private-algorithm}, with internal accuracy $c_1\alpha$ for a sufficiently small absolute $c_1>0$.
  Choose its internal trimming parameter $\etapriv\le\etaabs/20$ as in \Cref{subsec:pure-dp-proof}, so that
  \begin{equation*}
    \etapriv=\min\Set{\etaabs/20,\ \widetilde\Theta(\alpha\sqrt m)},
    \qquad
    O\Paren{\etapriv\sqrt{\frac{\log(1/\etapriv)}m}}\le\alpha/10.
  \end{equation*}
  The constants $C$ in the accuracy assumption and $c$ in $\eta\le c$ ensure $\eta\le\etapriv$.
  Replace the covariance allowance by
  \begin{equation*}
    \tau_{\mathrm{cov}}=
    C_0\Paren{\frac{\etapriv\log(1/\etapriv)}m+\frac\omega m+\omega^2},
    \tageq\label{eq:private-robust-covariance}
  \end{equation*}
  where $C_0$ is a sufficiently large absolute constant.
  All these parameters are public. The data-consistency constraints and the dependence of the score on its trimming budget are unchanged, so \Cref{lem:bounded_sensitivity,lem:quasi_convexity,lem:efficient_comp,lem:eff_low_score} apply without modification.
  The witness bounds remain polynomial because $\omega\le1$.

  We verify the statistical properties needed by the pure-DP reduction.
  The concentration estimates in the proof of \Cref{cor:robust-heterogeneous} hold simultaneously at parameters $\etapriv$, $2\etapriv$, and $\eta^*=\etaabs/2$, with failure probability at most $\beta/2$, once
  \begin{equation*}
    n\ge\Omega\Paren{\frac{k+\log(1/\beta)}{\etapriv^2\log(1/\etapriv)}}.
  \end{equation*}
  For the low-score witness, take $w_i$ to indicate the uncorrupted batches, set $y_i'=y_i^*$, and set $\mu'=\bar\mu=\E_i y_i^*$.
  By the proof of \Cref{cor:robust-heterogeneous}, we obtain a diagonal majorizer with allowance \eqref{eq:private-robust-covariance}.
  Since at most $\eta n\le\etapriv n$ batches are corrupted, this gives a score-$\etapriv n$ certificate at $\bar\mu$ and throughout the $\ell_1$ ball of radius $c_1\alpha$ around $\bar\mu$.

  For $\eta_0\in\{2\etapriv,\eta^*\}$, the allowance in \eqref{eq:private-robust-covariance} is at most
  $C_0(\eta_0\log(1/\eta_0)/m+\omega/m+\omega^2)$.
  Thus the degree-$6$ soundness argument in \Cref{cor:robust-heterogeneous}, with the numerical tolerance argument of \Cref{lem:certificate-soundness}, shows that every certificate with trimming budget at most $\eta_0n$ has
  \begin{equation*}
    \normo{\cL\mu'-\mu}
    \le O\Paren{\eta_0\sqrt{\frac{\log(1/\eta_0)}m}+\omega}.
    \tageq\label{eq:private-robust-soundness}
  \end{equation*}
  In particular, the score-$\eta^*n$ set is contained in an $\ell_1$ ball around $\mu$ of radius
  $O(c_1\alpha+1/\sqrt m+\omega)$.
  Because $\omega\le\alpha/C=O(c_1\alpha)$, this radius is $O(c_1\alpha+1/\sqrt m)$.
  Comparing this outer ball with the score-$\etapriv n$ ball of radius $c_1\alpha$ about $\bar\mu$ gives
  \begin{equation*}
    \log\Paren{\frac{V_{\eta'}(\cY)}{V_{\etapriv}(\cY)}}
    \le O\Paren{k\log\Paren{2+\frac1{c_1\alpha\sqrt m}}}
    \qquad(\etapriv\le\eta'\le\eta^*).
  \end{equation*}
  The full score domain is still
  $\Delta^{k-1}+\mathbb B_1^k(c_1\alpha+\tau n)$, so the bound on $V_1/V_{\etapriv}$ from \Cref{subsec:pure-dp-proof} also applies.

  It follows from \Cref{thm:pure_dp_reduction} that the sample bound in \Cref{cor:private-robustness} returns a point with score at most $2\etapriv n$ with probability at least $1-\beta/2$ on the concentration event.
  In the capped regime, the additional statistical requirement $n=\Omega(k+\log(1/\beta))$ is absorbed by $(k+\log(1/\beta))/\eps$ because $\eps\le1$.
  Apply \eqref{eq:private-robust-soundness} with $\eta_0=2\etapriv$, add the closeness allowance from item~2 of \Cref{def:certifiable-mean}, and project onto the simplex in $\ell_1$.
  For sufficiently small $c_1$ and sufficiently large $C$, the resulting total-variation error is at most $\alpha$.
  Together, the concentration and privacy-mechanism failure probabilities sum to at most $\beta$.

  For approximate DP, use the heterogeneous covariance allowance \eqref{eq:private-robust-covariance}. If $\delta>1-\beta/2$, apply the release-with-probability-$\delta$ construction from \Cref{subsec:approx_dp} to the robust estimator of \Cref{cor:robust-heterogeneous}, with trimming parameter $\etapriv$ and failure probability $\beta/2$. Its sample requirement is
  \begin{equation*}
    \widetilde O\Paren{k+\log(1/\beta)+\frac{k+\log(1/\beta)}{\alpha^2m}},
  \end{equation*}
  which the stated approximate-DP bound supplies since $\eps\le1$. Its success probability is at least $\delta(1-\beta/2)>1-\beta$.

  Otherwise, use the approximate-DP reduction with internal accuracy $c_1\alpha$ and trimming parameter $\etapriv$. Since $\omega\le1$, \eqref{eq:private-robust-covariance} gives
  \begin{equation*}
    \sqrt{1/m+\tau_{\mathrm{cov}}}\le O(1/\sqrt m+\omega).
  \end{equation*}
  Thus \Cref{cor:approx_dp_vol_ratio} bounds the worst-case intermediate log-volume ratio by
  \begin{equation*}
    \log\Paren{\frac{V_{\eta^*}(\cY)}{V_{0.8\eta^*}(\cY)}}
    \le O\Paren{k\log\Paren{1+\frac1{c_1\alpha\sqrt m}+\frac\omega{c_1\alpha}}}.
    \tageq\label{eq:private-heterogeneous-volume}
  \end{equation*}
  Using $\log(1+x)\le x$, the privacy condition \eqref{eq:approx-reduction-privacy} is therefore implied by
  \begin{equation*}
    n\ge O\Paren{\frac{k}{\alpha\sqrt m\,\eps}
      +\frac{k\omega/\alpha+\log(1/\delta)}\eps}.
  \end{equation*}

  On the concentration event used for pure DP, the score-$\etapriv n$ set contains the ball of radius $c_1\alpha$ about $\bar\mu$. By \eqref{eq:private-robust-soundness} at $\eta_0=\eta^*$ and the mean estimate in \Cref{cor:robust-heterogeneous}, every point of score at most $\eta^*n$ lies within $c_1\alpha+O(1/\sqrt m+\omega)$ of $\bar\mu$. Comparing these balls bounds $\log(V_{\eta'}(\cY)/V_{\etapriv}(\cY))$ by the right-hand side of \eqref{eq:private-heterogeneous-volume} for $\etapriv\le\eta'\le\eta^*$. When $\alpha\sqrt m$ is below the fixed split threshold in \Cref{subsec:pure-dp-proof}, $\omega\le\alpha/C$ makes the heterogeneity contribution a constant inside the logarithm, so the original $\smash[t]{\widetilde O}((k+\log(1/\beta))/(\alpha\smash[t]{\sqrt m}\,\eps))$ utility bound applies. Above that threshold, $\etapriv$ is constant, and \Cref{lem:refined-private-utility} requires only
  \begin{equation*}
    n\ge O\Paren{\frac{k}{\alpha\sqrt m\,\eps}
      +\frac{k\omega/\alpha+\log(2/\beta)}\eps}.
  \end{equation*}
  Both requirements are covered by the stated sample bound, which also supplies the concentration event. The reduction returns score at most $2\etapriv n$; applying \eqref{eq:private-robust-soundness} and projecting as in \eqref{eq:private-projected-error} gives accuracy $\alpha$. Allocating failure probability $\beta/2$ to concentration and $\beta/2$ to the reduction proves the claim.
\end{proof}

  \section{Statistical Query Lower Bound for Discrete Batches}
\label{sec:sq-lower-bound}

In this section, we prove a statistical query (SQ) lower bound for robust learning from discrete batches.
We show that no SQ algorithm with polynomially many queries of inverse-polynomial tolerance achieves error $o(\eta\min\{1,\smash[t]{\sqrt{\log(1/\eta)/m}}\})$ uniformly over the parameters. We adapt the moment-matching strategy and polynomial-correction construction of Diakonikolas, Kane, and Sun~\cite{DiakonikolasKS22} to categorical batches in $[k]^m$. We correct a biased binomial distribution to match the low-order moments of the unbiased binomial, and use this distribution for the number of batch elements in a hidden balanced subset of the alphabet. The resulting moment-matching lemma applies to every batch size. We first prove hardness of testing and then reduce testing to robust estimation.

\subsection{Testing Problem}
\label{subsec:sq-testing}

Let $\mu \in \Delta^{k-1}$ be the clean distribution. We model corruption by allowing the observed batch distribution $Q$ on $[k]^m$ to differ from the clean batch distribution $\mu^{\otimes m}$ by at most $\eta$ in total variation:
\begin{equation*}
  \normtv{Q-\mu^{\otimes m}} \le \eta \mper
\end{equation*}
An SQ query to $Q$ consists of a function $\phi\colon [k]^m \to [-1,1]$ and a tolerance $\tau>0$; the oracle returns a number within $\tau$ of $\E_{B\sim Q}\Brac{\phi(B)}$. We allow randomized, adaptive algorithms, whose guarantees must hold for every valid oracle. Query bounds are worst-case bounds. Allowing queries on the full ordered batch also covers SQ algorithms whose queries depend only on the batch average.

We give the construction and proof for even $k$; the case of odd $k$ follows similarly by leaving one symbol unused. Let $\mu_0$ be uniform on $[k]$, and set $P=\mu_0^{\otimes m}$. Pair the alphabet symbols as $\{1,2\},\{3,4\},\ldots,\{k-1,k\}$. A hidden subset $S$ contains exactly one symbol from each pair, so $\abs{S}=k/2$. For $0<\Delta<1/4$, increase the probabilities on $S$ and decrease those on its complement:
\begin{equation*}
  \mu_S(j)\coloneq
  \begin{cases}
    (1+2\Delta)/k,&j\in S,\\
    (1-2\Delta)/k,&j\notin S.
  \end{cases}
  \qquad
  \normtv{\mu_S-\mu_0}=\Delta \mper
\end{equation*}
For a batch $x=(x_1,\ldots,x_m)\in[k]^m$, let
\begin{equation*}
  N_S(x)\coloneq\sum_{i=1}^m\mathbb{1}[x_i\in S]
\end{equation*}
be the number of its elements that lie in the hidden subset. A single draw lies in $S$ with probability $1/2$ under $\mu_0$ and $1/2+\Delta$ under $\mu_S$. Writing $p_s\coloneq\operatorname{Bin}(m,1/2+s)$ for $\abs{s}<1/2$, we therefore have
\begin{equation*}
  N_S(X)\sim
  \begin{cases}
    p_0,&X\sim P,\\
    p_\Delta,&X\sim\mu_S^{\otimes m}.
  \end{cases}
\end{equation*}
Conditioned on the count $N_S$, both batch laws are uniform on the same set of batches. This lets us modify the count distribution while keeping the conditional batch distribution fixed.

To compare different hidden subsets, encode the choice from each pair by $\sigma\in\Set{-1,1}^{k/2}$: put $2u-1$ in $S_\sigma$ if $\sigma_u=1$, and $2u$ in $S_\sigma$ if $\sigma_u=-1$. Write $\mu_\sigma=\mu_{S_\sigma}$ and $N_\sigma=N_{S_\sigma}$. The sign encoding
\begin{equation*}
  v_\sigma(j)\coloneq2\mathbb{1}[j\in S_\sigma]-1,
  \qquad
  v_\sigma=(\sigma_1,-\sigma_1,\ldots,\sigma_{k/2},-\sigma_{k/2})
\end{equation*}
satisfies $\mu_\sigma(j)=(1+2\Delta v_\sigma(j))/k$. We consider the testing problem
\begin{equation}
  \label{eq:sq-testing-problem}
  \begin{aligned}
    H_0 &: \quad (\mu,Q)=(\mu_0,P),\\
    H_1 &: \quad \mu=\mu_\sigma \text{ for some }\sigma\in\Set{-1,1}^{k/2},
    \qquad \normtv{Q-\mu_\sigma^{\otimes m}}\le\eta/20 \mper
  \end{aligned}
\end{equation}
Given SQ access to $Q$, a tester must identify the clean hypothesis with probability at least $2/3$ on every instance. Each hypothesis specifies both the clean and observed distributions. If an observed distribution occurs under both hypotheses, testing is already information-theoretically impossible.

\subsection{SQ Hardness}
\label{subsec:sq-hardness}

\begin{theorem}[SQ Lower Bound for Testing]
  \label{thm:sq-testing}
  There is an absolute constant $c>0$ such that, for every integer $t\ge1$, there are constants $c_t>0$ and $k_0(t)$ with the following property. For all $k\ge k_0(t)$, $m\in\N$, and $\eta\in(0,1/4]$, set $\Delta=c_t \eta\min\{1,\smash[t]{\sqrt{\log(1/\eta)/m}}\}$. Any SQ algorithm solving \eqref{eq:sq-testing-problem} must either make a query of tolerance less than
  \begin{equation*}
    2k^{-(t+1)/8}
  \end{equation*}
  or make at least $\exp(c\smash[t]{\sqrt{k}})$ queries.
\end{theorem}

\paragraph{Proof roadmap.}
For each hidden subset $S_\sigma$, we construct a valid corrupted batch law $Q_\sigma$. The goal is to make its centered likelihood ratio $f_\sigma=dQ_\sigma/dP-1$ satisfy
\begin{equation*}
  \norm{f_\sigma}_{L^2(P)}^2\le2,
  \qquad
  \Abs{\iprod{f_\sigma,f_\omega}_{L^2(P)}}\le2\abs{\rho(\sigma,\omega)}^{t+1} \mper
\end{equation*}
Here $\rho(\sigma,\omega)\coloneq(2/k)\sum_{u=1}^{k/2}\sigma_u\omega_u$ is the centered overlap of the hidden subsets. \Cref{lem:sq-binomial} supplies the count distribution used to define $Q_\sigma$. Its total-variation bound keeps $Q_\sigma$ within the corruption budget; its moment identities eliminate Fourier degrees $1,\ldots,t$; and its chi-square bound controls the remaining Fourier mass. Independent uniformly chosen hidden subsets have small $\abs{\rho}$ with high probability. The resulting correlation bound then implies the query lower bound through \Cref{lem:sq-correlation-test}. Finally, \Cref{subsec:sq-estimation} reduces robust estimation to this testing problem. The binomial construction is proved in \Cref{subsec:sq-binomial}.

\begin{lemma}[Binomial Moment Matching]
  \label{lem:sq-binomial}
  For every integer $t\ge1$, there is a constant $c_t>0$ with the following property. Let $m\in\N$ and $\eta\in(0,1/4]$, and set
  \begin{equation*}
    \Delta\coloneq c_t\eta\sqrt{\min\Set{1,\frac{\log(1/\eta)}m}},
    \qquad
    p_s\coloneq\operatorname{Bin}(m,1/2+s) \mper
  \end{equation*}
  There is a probability distribution $\widetilde p$ on $\Set{0,\ldots,m}$ such that
  \begin{equation}
    \label{eq:sq-binomial-properties}
    \E_{N\sim\widetilde p}\Brac{N^i}=\E_{N\sim p_0}\Brac{N^i}\quad(0\le i\le t),
    \qquad
    \normtv{\widetilde p-p_\Delta}\le\eta/20,
    \qquad
    \chi^2(\widetilde p\Vert p_0)\le2 \mper
  \end{equation}
\end{lemma}
Here $\chi^2(\widetilde p\Vert p)\coloneq\sum_j(\widetilde p(j)-p(j))^2/p(j)$ denotes the chi-square divergence.

We separate the general SQ argument from the categorical construction in the following lemma; compare the SQ correlation bound in \cite[Lemma~11]{DiakonikolasKS22}. We prove this lemma after \Cref{thm:sq-testing}.

\begin{lemma}[SQ Testing from Correlation Tails]
  \label{lem:sq-correlation-test}
  Let $P$ and $\{Q_\theta:\theta\in\Theta\}$ be distributions on a finite domain $\mathcal X$, with $\Theta$ finite and nonempty and $Q_\theta\ll P$. Set $f_\theta=dQ_\theta/dP-1$. Suppose that, for some $B>0$ and $\gamma,\xi\ge0$, every $\theta\in\Theta$ satisfies
  \begin{equation*}
    \norm{f_\theta}_{L^2(P)}^2\le B,
    \qquad
    \Pr_{\omega\sim\operatorname{Unif}(\Theta)}\Brac{\Abs{\iprod{f_\theta,f_\omega}_{L^2(P)}}>\gamma}\le\xi \mper
  \end{equation*}
  Let $\tau>0$ satisfy $\tau^2>\gamma$. Any randomized, adaptive SQ tester that distinguishes $P$ from every $Q_\theta$ with success probability at least $2/3$ against every valid oracle, makes at most $q$ queries in the worst case, and uses tolerances at least $\tau$, must satisfy
  \begin{equation*}
    qB\xi\ge\frac{\tau^2-\gamma}{3} \mper
  \end{equation*}
\end{lemma}

\begin{proof}[Proof of \Cref{thm:sq-testing}]
  Take $\widetilde p$ from \Cref{lem:sq-binomial}. For each $\sigma\in\Set{-1,1}^{k/2}$, we construct an observed distribution $Q_\sigma$ satisfying $\normtv{Q_\sigma-\mu_\sigma^{\otimes m}}\le\eta/20$.

  \paragraph{Hard distributions.}
  To draw a batch from $Q_\sigma$, draw its count $N_\sigma$ from $\widetilde p$ and then choose uniformly among the batches with that count. Equivalently, define
  \begin{equation*}
    F_\sigma(x)\coloneq\frac{dQ_\sigma}{dP}(x)
    =\frac{\widetilde p(N_\sigma(x))}{p_0(N_\sigma(x))},
    \qquad
    f_\sigma\coloneq F_\sigma-1 \mper
  \end{equation*}
  Under $P$, the count $N_\sigma$ has distribution $p_0$, so $Q_\sigma$ is a probability distribution. $\mu_\sigma^{\otimes m}$ has likelihood ratio $p_\Delta(N_\sigma)/p_0(N_\sigma)$ with respect to $P$. Consequently,
  \begin{equation}
    \label{eq:sq-lift}
    \normtv{Q_\sigma-\mu_\sigma^{\otimes m}}=\normtv{\widetilde p-p_\Delta}\le\eta/20,
    \qquad
    \norm{f_\sigma}_{L^2(P)}^2=\chi^2(\widetilde p\Vert p_0)\le2 \mper
  \end{equation}

  \paragraph{Bounding correlations.}
  For $z\in\Set{-1,1}^m$, let $N(z)$ count its positive coordinates, and write the Fourier expansion
  \begin{equation*}
    g(z)\coloneq\frac{\widetilde p(N(z))}{p_0(N(z))}
    =\sum_{A\subseteq[m]}\hat{g}(A)\prod_{i\in A}z_i \mper
  \end{equation*}
  The batch likelihood ratios satisfy $F_\sigma(x)=g(v_\sigma(x_1),\ldots,v_\sigma(x_m))$. Here $\hat{g}(\emptyset)=1$. Moment matching implies $\hat{g}(A)=0$ for $1\le\abs{A}\le t$. Indeed, the conditional expectation of $\prod_{i\in A}z_i$ given $N(z)=j$, under uniform signs, is
  \begin{equation*}
    \binom{m}{\abs{A}}^{-1}\sum_{r=0}^{\abs{A}}(-1)^{\abs{A}-r}
    \binom{j}{r}\binom{m-j}{\abs{A}-r} \mper
  \end{equation*}
  This is a polynomial in $j$ of degree at most $\abs{A}$. Its expectation under $\widetilde p$ therefore equals its expectation under $p_0$, which is zero.

  For $U$ uniform on $[k]$, the signs $v_\sigma(U)$ and $v_\omega(U)$ are unbiased and satisfy $\E\Brac{v_\sigma(U)v_\omega(U)}=\rho(\sigma,\omega)$. In the product of the two Fourier expansions, terms indexed by distinct sets have expectation zero: a position in their symmetric difference contributes one unbiased sign. Terms indexed by the same set $A$ have expectation $\rho(\sigma,\omega)^{\abs{A}}$. Thus
  \begin{equation}
    \label{eq:sq-correlation}
    \iprod{f_\sigma,f_\omega}_{L^2(P)}
    =\sum_{\abs{A}>t}\hat{g}(A)^2\rho(\sigma,\omega)^{\abs{A}},
    \qquad
    \Abs{\iprod{f_\sigma,f_\omega}_{L^2(P)}}\le2\abs{\rho(\sigma,\omega)}^{t+1} \mper
  \end{equation}
  Parseval's identity and \eqref{eq:sq-lift} give the inequality. If $t\ge m$, the sum is empty and $f_\sigma=0$ for every $\sigma$.

  \paragraph{Applying the SQ lemma.}
  For each fixed $\sigma$, Hoeffding's inequality gives
  \begin{equation*}
    \Pr_{\omega\sim\operatorname{Unif}(\{\pm1\}^{k/2})}\Brac{\abs{\rho(\sigma,\omega)}>k^{-1/4}}\le2e^{-\sqrt{k}/8} \mper
  \end{equation*}
  Together with \eqref{eq:sq-lift} and \eqref{eq:sq-correlation}, this verifies the hypotheses of \Cref{lem:sq-correlation-test} with
  \begin{equation*}
    B=2,
    \qquad
    \gamma=2k^{-(t+1)/4},
    \qquad
    \xi=2e^{-\sqrt{k}/8},
    \qquad
    \tau=2k^{-(t+1)/8} \mper
  \end{equation*}
  If all query tolerances are at least $\tau$, the lemma yields
  \begin{equation*}
    q\ge\frac{\tau^2-\gamma}{3B\xi}
    =\frac16 k^{-(t+1)/4}e^{\sqrt{k}/8}
    \ge e^{\sqrt{k}/24}
  \end{equation*}
  for sufficiently large $k$ depending only on $t$. This proves the theorem with $c=1/24$.
\end{proof}

\begin{proof}[Proof of \Cref{lem:sq-correlation-test}]
  All inner products and norms below are in $L^2(P)$, and indices are independent and uniform on $\Theta$ unless otherwise specified.

  \paragraph{One query.}
  Fix $\phi\colon\mathcal X\to[-1,1]$. Let
  \begin{equation*}
    \mathcal T_\phi\coloneq\Set{\theta:\Abs{\iprod{\phi,f_\theta}}>\tau},
    \qquad
    p\coloneq\Pr_\theta\Brac{\theta\in\mathcal T_\phi},
    \qquad
    G\coloneq\E_\theta\Brac{\mathbb{1}[\theta\in\mathcal T_\phi]\sign(\iprod{\phi,f_\theta})f_\theta} \mper
  \end{equation*}
  If $p>0$, Cauchy--Schwarz and $\norm{\phi}_2\le1$ give
  \begin{equation*}
    p^2\tau^2<\iprod{\phi,G}^2\le\norm{G}_2^2
    \le\gamma p^2+B\xi p \mper
  \end{equation*}
  For the last inequality, pairs with $\Abs{\iprod{f_\theta,f_\omega}}\le\gamma$ contribute at most $\gamma p^2$. For each $\theta\in\mathcal T_\phi$, the remaining $\omega$ have probability at most $\xi$, and $\Abs{\iprod{f_\theta,f_\omega}}\le B$ by the norm bound. Rearranging gives
  \begin{equation}
    \label{eq:sq-one-query}
    p\le\frac{B\xi}{\tau^2-\gamma} \mper
  \end{equation}
  This bound also holds when $p=0$. Since $\E_{Q_\theta}\phi-\E_P\phi=\iprod{\phi,f_\theta}$, it bounds the fraction of alternatives for which the exact null answer is outside tolerance $\tau$.

  \paragraph{Adaptive queries.}
  Run the tester against the oracle returning exact expectations under $P$. Conditional on its random coins, the at most $q$ queries on this execution are fixed. By \eqref{eq:sq-one-query} and a union bound, these answers are all valid for a uniform $Q_\theta$ at the requested tolerances except with probability at most $qB\xi/(\tau^2-\gamma)$, also after averaging over the coins.

  For each $\theta$, define an oracle that returns $\E_P\phi$ whenever this is within the requested tolerance of $\E_{Q_\theta}\phi$, and returns $\E_{Q_\theta}\phi$ otherwise. This is a valid oracle on every execution and does not depend on the tester's random coins. Couple the null and alternative executions using the same coins. Their transcripts differ with probability at most $qB\xi/(\tau^2-\gamma)$, averaged over $\theta$. Success probability $2/3$ under each hypothesis requires an acceptance-probability gap of at least $1/3$. Therefore
  \begin{equation*}
    \frac{qB\xi}{\tau^2-\gamma}\ge\frac13 \mper
  \end{equation*}
\end{proof}

\subsection{From Estimation to Testing}
\label{subsec:sq-estimation}

\begin{corollary}[SQ Lower Bound for Robust Estimation]
  \label{cor:sq-robust-estimation}
  For every integer $t\ge1$, let $c_t$ and $k_0(t)$ be as in \Cref{thm:sq-testing}. For all $k\ge k_0(t)$, $m\in\N$, and $\eta\in(0,1/4]$, any SQ estimator that, for every $\mu\in\Delta^{k-1}$ and every $Q$ with $\normtv{Q-\mu^{\otimes m}}\le\eta$, guarantees
  \begin{equation*}
    \normtv{\hat{\mu}-\mu}\le\frac{c_t}{4}\eta\min\Set{1,\sqrt{\frac{\log(1/\eta)}m}}
  \end{equation*}
  with probability at least $2/3$ must either make a query of tolerance less than $2k^{-(t+1)/8}$ or make at least $\exp(c\smash[t]{\sqrt{k}})$ queries.

  Consequently, no SQ algorithm with fixed polynomial bounds on query count and inverse tolerance in $k,m,1/\eta$ achieves error $o(\eta\min\{1,\smash[t]{\sqrt{\log(1/\eta)/m}}\})$ uniformly over the parameters. This holds already along the sequence $\eta=1/k$ and $m=\ceil{(\log k)^2}$ as $k\to\infty$.
\end{corollary}

\begin{proof}
  Run the estimator and declare $H_1$ if $\normtv{\hat{\mu}-\mu_0}>\Delta/2$. On the estimator's success event, this distance is at most $\Delta/4$ under $H_0$ and at least $3\Delta/4$ under $H_1$, by the triangle inequality. Thus the estimator solves \eqref{eq:sq-testing-problem}, and \Cref{thm:sq-testing} applies.

  For the asymptotic conclusion, set $\eta=1/k$ and $m=\ceil{(\log k)^2}$. Any fixed polynomial resource bounds become $q\le k^a$ and tolerances at least $k^{-b}$, for fixed $a,b$ and sufficiently large $k$. Choose a fixed $t$ with $(t+1)/8>b$. An $o(\eta\min\{1,\smash[t]{\sqrt{\log(1/\eta)/m}}\})$ error guarantee is eventually smaller than $c_t \eta\min\{1,\smash[t]{\sqrt{\log(1/\eta)/m}}\}/4$, while $k^{-b}\ge2k^{-(t+1)/8}$ and $k^a<\exp(\smash[t]{\sqrt{k}}/24)$. This contradicts \Cref{thm:sq-testing}.
\end{proof}

With sufficiently many batches, the polynomial-time estimator of Jain and Orlitsky~\cite{JainO20a}, together with the pooled histogram for $m\le\log(1/\eta)$, achieves error $O(\eta\min\{1,\smash[t]{\sqrt{\log(1/\eta)/m}}\})$. Information-theoretically, the corruption dependence is $\Theta(\eta/\sqrt m)$~\cite{QiaoV18}. In \Cref{thm:sq-testing}, the error scale matches the dependence on $m$ and $\eta$ of the known polynomial-time corruption rate. Its constant $c_t$ depends on the moment order, and hence on the polynomial bound on inverse tolerance. We bound the attainable corruption error, rather than the number of batches required to attain it.

One can also obtain the distributions used in the SQ lower bound, up to exponentially small total-variation error, by replacing at most $\floor{\eta n}$ batches in a finite dataset. Indeed, independently maximally couple $n$ clean batches to batches with law $Q_\sigma$. Replace the mismatched batches if there are at most $\floor{\eta n}$ mismatches, and make no replacements otherwise. This is a legal batch-corruption adversary. Since $\normtv{Q_\sigma-\mu_\sigma^{\otimes m}}\le\eta/20$, a Chernoff bound shows that its output distribution is within $\exp(-c'\eta n)$ of $Q_\sigma^{\otimes n}$ in total variation, for an absolute constant $c'>0$.

\subsection{Binomial Moment Construction}
\label{subsec:sq-binomial}

We follow the polynomial-correction method of Diakonikolas, Kane, and Sun~\cite[Proposition~25]{DiakonikolasKS22}. For each fixed moment order $t$, we solve the moment equations directly on the discrete count grid. The resulting estimates, together with a direct treatment of the remaining parameter range, give the separation scale in \Cref{lem:sq-binomial} for every batch size.

\begin{proof}[Proof of \Cref{lem:sq-binomial}]
  Constants denoted $C_t$ may change from line to line and depend only on $t$. First suppose $\min\{m,\log(1/\eta)\}\ge R_t$, where $R_t>t+1$ is a sufficiently large constant depending only on $t$. Set
  \begin{align*}
    I&\coloneq\Set{j\in\Set{0,\ldots,m}:\abs{j-m/2}\le \frac{\sqrt{m\min\{m,\log(1/\eta)\}}}{16}},\\
    z_j&\coloneq\frac{16(j-m/2)}{\sqrt{m\min\{m,\log(1/\eta)\}}} \mper
  \end{align*}
  Since $\min\{m,\log(1/\eta)\}\le m$, we have
  \begin{equation*}
    \frac1{16} \min\{m,\log(1/\eta)\}\le \frac{\sqrt{m\min\{m,\log(1/\eta)\}}}{16}\le m/16
  \end{equation*}
  and $\abs{I}=\Theta(\smash[t]{\sqrt{m\min\{m,\log(1/\eta)\}}}/16)$. Consider the $(t+1)\times(t+1)$ matrix
  \begin{equation*}
    M_{uv}\coloneq\frac1{\abs{I}}\sum_{j\in I}z_j^{u+v},
    \qquad 0\le u,v\le t \mper
  \end{equation*}
  These Riemann sums converge to
  \begin{equation*}
    H_{uv}=\frac12\int_{-1}^1 x^{u+v}\,dx
    \qquad\text{as }\frac{\sqrt{m\min\{m,\log(1/\eta)\}}}{16}\to\infty,
  \end{equation*}
  with error $O_t(16/\smash[t]{\sqrt{m\min\{m,\log(1/\eta)\}}})$. The matrix $H$ is positive definite, since its quadratic form is the integral of the square of a nonzero polynomial. Thus, by increasing $R_t$, we have
  \begin{equation}
    \label{eq:sq-moment-inverse}
    \normop{M^{-1}}\le C_t \mper
  \end{equation}
  The bound is uniform over $m$ and $\eta$; odd $m$ only shifts the grid by half a step.

  \paragraph{Correcting the moments.}
  Let $b_i\coloneq\E_{N\sim p_0}\Brac{z_N^i}-\E_{N\sim p_\Delta}\Brac{z_N^i}$. We claim that
  \begin{equation}
    \label{eq:sq-moment-discrepancy}
    \abs{b_i}\le C_t\frac{16m\Delta}{\sqrt{m\min\{m,\log(1/\eta)\}}}\le C_tc_t\eta,
    \qquad 0\le i\le t \mper
  \end{equation}
  Differentiating the binomial formula gives, for any function $f$ on $\Set{0,\ldots,m}$,
  \begin{equation*}
    \frac{d}{ds}\E_{N\sim p_s}\Brac{f(N)}
    =m\E_{Y\sim\operatorname{Bin}(m-1,1/2+s)}\Brac{f(Y+1)-f(Y)} \mper
  \end{equation*}
  For $1\le i\le t$ and $f(j)=z_j^i$, the finite difference is
  \begin{equation*}
    \frac{16i}{\sqrt{m\min\{m,\log(1/\eta)\}}}\int_0^1\Paren{\frac{16(Y+u-m/2)}{\sqrt{m\min\{m,\log(1/\eta)\}}}}^{i-1}\,du \mper
  \end{equation*}
  For $0\le s\le\Delta$ and $0\le u\le1$, the random variable $16(Y+u-m/2)/\smash[t]{\sqrt{m\min\{m,\log(1/\eta)\}}}$ has bounded absolute moments of every fixed order. Its centered fluctuations have scale $16/\smash[t]{\sqrt{\min\{m,\log(1/\eta)\}}}\le16$, and its mean has absolute value at most $16(m\Delta+1)/\smash[t]{\sqrt{m\min\{m,\log(1/\eta)\}}}\le2$, after taking $c_t\le\frac1{16}$ and $R_t$ large. Here we use $\E\Brac{\abs{Y-\E Y}^r}\le C_rm^{r/2}$, which follows by integrating Hoeffding's tail bound. Integrating the derivative over $[0,\Delta]$ proves \eqref{eq:sq-moment-discrepancy}; also $b_0=0$.

  Solve $M\lambda=b$ and define
  \begin{equation*}
    h(z)\coloneq\sum_{i=0}^t\lambda_i z^i,
    \qquad
    d_j\coloneq
    \begin{cases}
      h(z_j)/\abs{I},&j\in I,\\
      0,&j\notin I,
    \end{cases}
    \qquad
    \widetilde p(j)\coloneq p_\Delta(j)+d_j \mper
  \end{equation*}
  Then $\sum_jd_jz_j^i=b_i$ for $0\le i\le t$, giving the moment identities in \eqref{eq:sq-binomial-properties} and total mass one. By \eqref{eq:sq-moment-inverse} and \eqref{eq:sq-moment-discrepancy},
  \begin{equation}
    \label{eq:sq-moment-correction}
    \sup_{\abs{z}\le1}\abs{h(z)}\le C_tc_t\eta,
    \qquad
    \sum_j\abs{d_j}\le C_tc_t\eta,
    \qquad
    \max_j\abs{d_j}\le \frac{16C_tc_t\eta}{\sqrt{m\min\{m,\log(1/\eta)\}}} \mper
  \end{equation}

  \paragraph{Nonnegativity and distance bounds.}
  For $j\in I$, both $j/m$ and $1/2+\Delta$ lie in $[1/4,3/4]$, and
  \begin{equation*}
    \abs{j-m(1/2+\Delta)}\le\frac{\sqrt{m\min\{m,\log(1/\eta)\}}}{8}
  \end{equation*}
  since $m\Delta\le \smash[t]{\sqrt{m\min\{m,\log(1/\eta)\}}}/16$. Stirling's formula gives
  \begin{equation*}
    p_\Delta(j)\ge\frac{c}{\sqrt m}\exp\Paren{-mD\Paren{j/m\,\Vert\,1/2+\Delta}} \mcom
  \end{equation*}
  where $D(\cdot\Vert\cdot)$ is binary relative entropy. The inequality $\log x\le x-1$ implies
  \begin{equation*}
    D\Paren{j/m\,\Vert\,1/2+\Delta}
    \le\frac{(j/m-1/2-\Delta)^2}{(1/2+\Delta)(1/2-\Delta)}
    \le\frac{16}{3}(j/m-1/2-\Delta)^2.
  \end{equation*}
  Therefore
  \begin{equation*}
    p_\Delta(j)\ge\frac{c}{\sqrt m}e^{-\min\{m,\log(1/\eta)\}/12}\ge\frac{c\sqrt\eta}{\sqrt m} \mper
  \end{equation*}
  Consequently,
  \begin{equation*}
    \frac{\abs{d_j}}{p_\Delta(j)}\le C_tc_t\frac{\sqrt\eta}{\sqrt{\min\{m,\log(1/\eta)\}}}\le\frac12
  \end{equation*}
  after reducing $c_t$. Thus $0\le\widetilde p\le\frac32p_\Delta$, and \eqref{eq:sq-moment-correction} gives $\normtv{\widetilde p-p_\Delta}\le\eta/20$.

  By the binomial theorem,
  \begin{equation}
    \label{eq:sq-binomial-chi}
    \begin{aligned}
      \sum_{j=0}^m\frac{p_\Delta(j)^2}{p_0(j)}
      &=\sum_{j=0}^m\binom mj\Brac{2(1/2+\Delta)^2}^{j}\Brac{2(1/2-\Delta)^2}^{m-j}\\
      &=(1+4\Delta^2)^m \mper
    \end{aligned}
  \end{equation}
  Hence
  \begin{equation*}
    1+\chi^2(\widetilde p\Vert p_0)\le\frac94(1+4\Delta^2)^m\le3 \mper
  \end{equation*}
  The last inequality holds for sufficiently small $c_t$, since $m\Delta^2\le c_t^2\eta^2\log(1/\eta)$ is uniformly small.

  \paragraph{Remaining parameter range.}
  If $\min\{m,\log(1/\eta)\}<R_t$, take $\widetilde p=p_0$. By \eqref{eq:sq-binomial-chi} and Cauchy--Schwarz,
  \begin{equation*}
    \normtv{p_0-p_\Delta}\le\frac12\sqrt{(1+4\Delta^2)^m-1}
    \le C\Delta\sqrt m\le Cc_t\eta\sqrt{R_t}\le\eta/20 \mper
  \end{equation*}
  All moments match and the chi-square divergence is zero. Every choice of constants depends only on $t$.
\end{proof}

  \paragraph*{AI Disclosure.}
  We used Codex and GPT Pro to assist with technical development and exposition, including the diagonal-majorizer and semidefinite-duality formulation, the application of the symmetric Grothendieck inequality, the approximate-DP obstruction argument, and the SQ lower-bound proof. The authors are responsible for the final manuscript.

  \newcommand{\etalchar}[1]{$^{#1}$}

  \appendix
  \section{Details for the Approximate-DP Section}
\label{app:approx-dp}

This appendix proves the obstruction argument of Hopkins, Kamath, Majid, and Narayanan~\cite{HopkinsKMN23} in sign directions $\sigma \in \set{\pm1}^k$, which is the geometric step needed for the $\ell_1$-diameter bound in \Cref{lem:approx_dp_vol_ub}. We follow Lemma~5.13 in the full version of~\cite{HopkinsKMN23}, with sign directions replacing Euclidean ones and exact $\ell_1$ duality replacing an $\ell_2$-to-$\ell_1$ conversion.

\begin{lemma}[One-dimensional obstruction for low-variance certificates]
  \label{lem:approx-dp-1d}
  Let $z_1, \ldots, z_n \in \mathbb{R}$ and suppose that at least $n/4$ of the $z_i$'s are at least $20$.
  Let $\cL$ be a degree-$6$ pseudoexpectation over variables $\set{w_i, z_i'}_{i \in [n]}$ such that
  \begin{enumerate}
    \item $w_i^2 = w_i$ for every $i \in [n]$,
    \item $\sum_{i=1}^n w_i \ge 0.99n$,
    \item $w_i(z_i' - z_i) = 0$ for every $i \in [n]$,
    \item $\cL \left[\frac{1}{n}\sum_{i=1}^n z_i' \right] = 0$.
  \end{enumerate}
  Then
  \begin{equation*}
    \cL\left[\frac{1}{n}\sum_{i=1}^n \Paren{z_i' - \frac{1}{n}\sum_{j=1}^n z_j'}^2\right] \ge 2.
  \end{equation*}
\end{lemma}

\begin{proof}
  Repeating each index four times does not change the hypotheses or conclusion, so we may assume that $n$ is divisible by $4$. Write $\bar{z}' = \frac{1}{n}\sum_{i=1}^n z_i'$.
  Suppose for contradiction that
  \begin{equation}
    \label{eq:approx-dp-1d-small-var}
    \cL\left[\frac{1}{n}\sum_{i=1}^n (z_i' - \bar{z}')^2\right] < 2.
  \end{equation}
  Because $w_i^2 = w_i$, we also have $(1-w_i)^2 = 1-w_i$.
  Hence $0 \le \cL[w_i] \le 1$ for every $i$, and
  \begin{equation*}
    w_i w_j - w_i - w_j + 1 = (1-w_i)(1-w_j)
  \end{equation*}
  is a square modulo the Boolean equalities, so $\cL[w_i w_j] \ge \cL[w_i] + \cL[w_j] - 1$. Also $0\le\cL[w_iw_j]\le1$.

  Write $z_{(1)}\le\cdots\le z_{(n)}$ for the ordered observations and set $C_1=z_{(n/4)}$ and $C_2=z_{(3n/4+1)}$.
  If $C_2 - C_1 > 8$, then there exist sets $S_+,S_- \subseteq [n]$ with $|S_+| = |S_-| = n/4$ such that $z_i - z_j \ge 8$ for every $i \in S_+$ and $j \in S_-$.
  Using Condition~3, we obtain
  \begin{align*}
    \cL\Brac{\sum_{i,j=1}^n (z_i' - z_j')^2}
    &\ge \cL\Brac{\sum_{i,j=1}^n w_i w_j (z_i' - z_j')^2} \\
    &= \sum_{i,j=1}^n (z_i - z_j)^2 \cL[w_i w_j] \\
    &\ge 2 \sum_{i \in S_+, j \in S_-} (z_i - z_j)^2 \cL[w_iw_j] \\
    &\ge 2 \cdot 8^2 \sum_{i \in S_+, j \in S_-} \cL[w_iw_j] \\
    &\ge 2 \cdot 8^2 \sum_{i \in S_+, j \in S_-} \Paren{\cL[w_i] + \cL[w_j] - 1}.
  \end{align*}
  Since $|S_+|=|S_-|=n/4$ and $\sum_i \cL[w_i] \ge 0.99n$, the total expected weight on $S_+ \cup S_-$ is at least $0.49n$.
  Therefore,
  \begin{align*}
    \sum_{i \in S_+, j \in S_-} \Paren{\cL[w_i] + \cL[w_j] - 1}
    &= \frac{n}{4}\sum_{i \in S_+}\cL[w_i] + \frac{n}{4}\sum_{j \in S_-}\cL[w_j] - \frac{n^2}{16} \\
    &\ge \frac{n}{4}(0.49n) - \frac{n^2}{16}
    > 0.06 n^2.
  \end{align*}
  Thus $\cL[\sum_{i,j} (z_i' - z_j')^2] > 6n^2$.
  Using the identity
  \begin{equation*}
    \sum_{i,j=1}^n (z_i' - z_j')^2 = 2n \sum_{i=1}^n (z_i' - \bar{z}')^2,
  \end{equation*}
  we conclude that
  \begin{equation*}
    \cL\left[\frac{1}{n}\sum_{i=1}^n (z_i' - \bar{z}')^2\right] > 3,
  \end{equation*}
  contradicting \eqref{eq:approx-dp-1d-small-var}.
  Hence we may assume that $C_2 - C_1 \le 8$.

  Let $S$ be the middle $n/2$ indices after sorting the $z_i$'s.
  Then there exists $C \ge 16$ such that
  \begin{equation*}
    z_i \in [C-4, C+4]
    \qquad \text{for all } i \in S,
  \end{equation*}
  because at least $n/4$ of the $z_i$'s are at least $20$.
  Condition~2 implies that $\sum_{i \in S}\cL[w_i] \ge 0.49n$, and so
  \begin{equation}
    \label{eq:approx-dp-weighted-sum-lb}
    \cL\Brac{\sum_{i \in S} w_i z_i'}
    = \sum_{i \in S} z_i \cL[w_i]
    \ge (C-4)\sum_{i \in S}\cL[w_i]
    \ge 0.49(C-4)n.
  \end{equation}
  Likewise,
  \begin{equation}
    \label{eq:approx-dp-weighted-sum-sq-ub}
    \cL\Brac{\Paren{\sum_{i \in S} w_i z_i'}^2}
    = \sum_{i,j \in S} z_i z_j \cL[w_i w_j]
    \le (C+4)^2\Paren{\frac{n}{2}}^2.
  \end{equation}

  Next, write
  \begin{equation*}
    \bar z'_S \coloneq \frac{1}{|S|}\sum_{i \in S} z_i' \mper
  \end{equation*}
  Because $|S| = n/2$, we have
  \begin{equation*}
    \sum_{i \in S}(z_i')^2
    =
    \sum_{i \in S}(z_i' - \bar z'_S)^2
    + \frac{1}{|S|}\Paren{\sum_{i \in S} z_i'}^2.
  \end{equation*}
  Since $\bar z'_S$ minimizes $\sum_{i \in S}(z_i' - c)^2$ over all constants $c$, we have
  \begin{equation*}
    \sum_{i \in S}(z_i' - \bar z'_S)^2
    \le
    \sum_{i \in S}(z_i' - \bar z')^2
    \le
    \sum_{i=1}^n (z_i' - \bar z')^2 \mper
  \end{equation*}
  Therefore,
  \begin{equation*}
    \sum_{i \in S}(z_i')^2
    \le
    \sum_{i=1}^n (z_i' - \bar z')^2 + \frac{2}{n}\Paren{\sum_{i \in S} z_i'}^2.
  \end{equation*}
  Taking $\cL$ of both sides and using \eqref{eq:approx-dp-1d-small-var} yields
  \begin{equation}
    \label{eq:approx-dp-second-moment-on-S}
    \cL\Brac{\sum_{i \in S}(z_i')^2}
    \le 2n + \frac{2}{n}\cL\Brac{\Paren{\sum_{i \in S} z_i'}^2}.
  \end{equation}

  Define
  \begin{equation*}
    E_{\mathrm{drop}} \coloneq \cL\Brac{\Paren{\sum_{i \in S}(1-w_i)z_i'}^2}.
  \end{equation*}
  \begin{sodarevision}
    Apply the degree-$6$ polynomial weighted Cauchy--Schwarz certificate from the proof of \Cref{lem:main-sos-lemma} with $b_i=1-w_i$. Condition~2 and the Boolean identities for $b_i$ then give
    \begin{align*}
      E_{\mathrm{drop}}
      &\le \cL\Brac{
        \Paren{\sum_{i \in S}(1-w_i)}
        \Paren{\sum_{i \in S}(1-w_i)(z_i')^2}
      } \\
      &\le 0.01n\,\cL\Brac{\sum_{i \in S}(1-w_i)(z_i')^2} \\
      &\le 0.01n\,\cL\Brac{\sum_{i \in S}(z_i')^2}.
    \end{align*}
  \end{sodarevision}
  Combining this with \eqref{eq:approx-dp-second-moment-on-S}, and then using
  \begin{equation*}
    \sum_{i \in S} z_i' = \sum_{i \in S}(1-w_i)z_i' + \sum_{i \in S} w_i z_i',
  \end{equation*}
  followed by $(a+b)^2 \le 2a^2 + 2b^2$, gives
  \begin{align*}
    E_{\mathrm{drop}}
    &\le 0.01n \Brac{2n + \frac{2}{n}\cL\Brac{\Paren{\sum_{i \in S} z_i'}^2}} \\
    &\le 0.02n^2 + 0.04E_{\mathrm{drop}} + 0.04 \cL\Brac{\Paren{\sum_{i \in S} w_i z_i'}^2}.
  \end{align*}
  Rearranging and applying \eqref{eq:approx-dp-weighted-sum-sq-ub},
  \begin{equation*}
    E_{\mathrm{drop}} \le \frac{1}{24}n^2 + \frac{1}{96}(C+4)^2n^2 \le 0.02 C^2 n^2,
  \end{equation*}
  where the last inequality uses $C \ge 16$.
  Hence
  \begin{equation}
    \label{eq:approx-dp-unweighted-sum-lb}
    \Abs{\cL\Brac{\sum_{i \in S}(1-w_i)z_i'}} \le \sqrt{E_{\mathrm{drop}}} \le 0.15 C n.
  \end{equation}

  From \eqref{eq:approx-dp-weighted-sum-lb} and \eqref{eq:approx-dp-unweighted-sum-lb}, and using $C \ge 16$, we obtain
  \begin{equation*}
    \cL\Brac{\sum_{i \in S} z_i'}
    = \cL\Brac{\sum_{i \in S} w_i z_i'} + \cL\Brac{\sum_{i \in S}(1-w_i)z_i'}
    \ge 0.35 C n - 0.15 C n
    \ge 0.2 C n.
  \end{equation*}
  Condition~4 says $\cL[\sum_{i=1}^n z_i'] = 0$, so
  \begin{equation*}
    \cL\Brac{\sum_{i \in S} z_i' - \frac{1}{2}\sum_{i=1}^n z_i'} \ge 0.2 C n.
  \end{equation*}
  Another application of Cauchy--Schwarz yields
  \begin{equation}
    \label{eq:approx-dp-center-gap}
    \cL\Brac{\Paren{\sum_{i \in S} z_i' - \frac{1}{2}\sum_{i=1}^n z_i'}^2} \ge 0.04 C^2 n^2.
  \end{equation}

  Let $R = [n] \setminus S$, so $|R| = n/2$.
  Then
  \begin{equation*}
    \sum_{i \in S} z_i' - \frac{1}{2}\sum_{i=1}^n z_i'
    = \frac{1}{2}\Paren{\sum_{i \in S} z_i' - \sum_{j \in R} z_j'}
    = \frac{1}{n}\sum_{i \in S, j \in R}(z_i' - z_j').
  \end{equation*}
  Applying Cauchy--Schwarz once more,
  \begin{equation*}
    \Paren{\frac{1}{n}\sum_{i \in S, j \in R}(z_i' - z_j') }^2
    \le \frac{1}{4}\sum_{i \in S, j \in R}(z_i' - z_j')^2.
  \end{equation*}
  Combining this with \eqref{eq:approx-dp-center-gap} shows
  \begin{equation*}
    \cL\Brac{\sum_{i \in S, j \in R}(z_i' - z_j')^2} \ge 0.16 C^2 n^2.
  \end{equation*}
  Since the left-hand side is at most $\cL[\sum_{i,j} (z_i' - z_j')^2]$, we conclude that
  \begin{equation*}
    2n \cdot \cL\Brac{\sum_{i=1}^n (z_i' - \bar{z}')^2}
    = \cL\Brac{\sum_{i,j=1}^n (z_i' - z_j')^2}
    \ge 0.16 C^2 n^2
    > 4n^2,
  \end{equation*}
  again contradicting \eqref{eq:approx-dp-1d-small-var}.
  This contradiction proves the lemma.
\end{proof}

\begin{remark}[Approximate certificates]
  \label{rem:approx-dp-handoff}
  We first treat exact pseudoexpectations in the proof of \Cref{lem:approx_dp_vol_ub}.
  Every displayed inference in its one-dimensional obstruction argument is part of a fixed degree-$6$ derivation whose coefficient norm is bounded by $\poly(n,k,m)$.
  Since $(1/m+\tau_{\mathrm{cov}})^{-1/2}\le\sqrt m$, normalizing the projections by this factor preserves the polynomial coefficient bound.
  Therefore, if a linear functional $\cL$ only $\tau$-approximately satisfies the constraints in $\cB_{\eta^* n}(\cY)$, the derivation incurs an additive loss
  \begin{equation*}
    O\Paren{\tau \poly(n,k,m)\eta^* n} \mper
  \end{equation*}
  Choosing
  \begin{equation*}
    \tau = \frac{1}{\poly(n,k,m,\alpha^{-1})}
  \end{equation*}
  makes the accumulated coefficient error $o(1)$, so the one-dimensional obstruction argument also applies to approximate certificates.
\end{remark}

\begin{proof}[Full proof of \Cref{lem:approx_dp_vol_ub}]
  Let $\cL_1,\cL_2$ be $(\alpha,\tau,\eta^* n)$-certificates for $\nu_1,\nu_2$, and write $\mu_{\cL_j} \coloneq \cL_j \mu'$ for $j \in \set{1,2}$.
  For a fixed certificate $\cL_j$, we work with the observed and repaired one-dimensional projections of the batch-average vectors in $\cY$ and the pseudo-points $y_i'$ from $\cB_{\eta^* n}(\cY)$.
  We first show that for every sign vector $\sigma \in \set{\pm 1}^k$,
  \begin{equation}
    \label{eq:approx-dp-sign-bound}
    \Abs{\iprod{\mu_{\cL_1} - \mu_{\cL_2}, \sigma}} \le O\Paren{\sqrt{\frac1m+\tau_{\mathrm{cov}}}}.
  \end{equation}

  Fix $j \in \set{1,2}$ and a sign vector $\sigma$.
  Set
  \begin{equation*}
    a_i \coloneq \frac{1}{\sqrt{1/m+\tau_{\mathrm{cov}}}}\iprod{y_i - \mu_{\cL_j}, \sigma}
    \qquad \text{and} \qquad
    a_i' \coloneq \frac{1}{\sqrt{1/m+\tau_{\mathrm{cov}}}}\iprod{y_i' - \mu_{\cL_j}, \sigma}.
  \end{equation*}
  We first treat the exact case $\tau = 0$.
  Then $\cL_j$ exactly satisfies the constraints in $\cB_{\eta^* n}(\cY)$.
  In particular,
  \begin{enumerate}
    \item $w_i^2 = w_i$,
    \item $\sum_{i=1}^n w_i \ge (1-\eta^*)n$,
    \item $w_i(a_i' - a_i) = (1/m+\tau_{\mathrm{cov}})^{-1/2}\sigma^\top w_i(y_i' - y_i) = 0$,
    \item $\cL_j[n^{-1}\sum_{i=1}^n a_i'] = (1/m+\tau_{\mathrm{cov}})^{-1/2}\iprod{\cL_j\mu' - \mu_{\cL_j}, \sigma} = 0$.
  \end{enumerate}
  Choosing $\eta^*$ smaller than $0.01$, \Cref{lem:approx-dp-1d} shows that if at least $n/4$ of the $a_i$'s were at least $20$, then
  \begin{equation}
    \label{eq:approx-dp-large-var-lb}
    \cL_j\Brac{\frac{1}{n}\sum_{i=1}^n \Paren{a_i' - \frac{1}{n}\sum_{\ell=1}^n a_\ell'}^2} \ge 2.
  \end{equation}
  By definition of $a_i'$ and $\mu_{\cL_j}$, the left-hand side is exactly
  \begin{equation*}
    \frac{1}{1/m+\tau_{\mathrm{cov}}}\cL_j\Brac{\frac{1}{n}\sum_{i=1}^n \iprod{y_i' - \mu', \sigma}^2 }.
  \end{equation*}
  Applying the covariance constraint from $\cB_{\eta^* n}(\cY)$ gives
  \begin{align*}
    \frac{1}{1/m+\tau_{\mathrm{cov}}}\cL_j\Brac{\frac{1}{n}\sum_{i=1}^n \iprod{y_i' - \mu', \sigma}^2 }
    &\le \frac{\cL_j\Brac{\sigma^\top(\diag(\mu') - \mu' \mu'^\top)\sigma}
      + m\cL_j\Brac{\sum_{\ell=1}^k \lambda_\ell}}{1+m\tau_{\mathrm{cov}}} \\
    &\le \frac{\cL_j\Brac{1 - \iprod{\mu', \sigma}^{\,2}} + m\tau_{\mathrm{cov}}}{1+m\tau_{\mathrm{cov}}}
    \le 1,
  \end{align*}
  since $\mu'$ lies in the simplex in the exact case and $\sigma_\ell^2 = 1$ for every $\ell \in [k]$.
  This contradicts \eqref{eq:approx-dp-large-var-lb}.
  Hence fewer than $n/4$ of the $a_i$'s can be at least $20$.
  Applying this argument to $-\sigma$ shows that fewer than $n/4$ of the $a_i$'s can be at most $-20$.
  Therefore, the median of $\set{\iprod{y_i, \sigma}}_{i \in [n]}$ lies within $20\smash[t]{\sqrt{1/m+\tau_{\mathrm{cov}}}}$ of $\iprod{\mu_{\cL_j}, \sigma}$.

  By Remark~\ref{rem:approx-dp-handoff}, the bound on the distance from $\iprod{\mu_{\cL_j},\sigma}$ to the observed-data median remains valid for approximate certificates.

  Both projected certificate means $\iprod{\mu_{\cL_1},\sigma}$ and $\iprod{\mu_{\cL_2},\sigma}$ therefore lie within $O(\smash[t]{\sqrt{1/m+\tau_{\mathrm{cov}}}})$ of the same median of $\set{\iprod{y_i,\sigma}}_{i\in[n]}$, which proves \eqref{eq:approx-dp-sign-bound}.
  By exact $\ell_1$ duality, \eqref{eq:approx-dp-sign-bound} implies
  \begin{equation*}
    \|\mu_{\cL_1} - \mu_{\cL_2}\|_1
    =
    \max_{\sigma \in \set{\pm 1}^k}\iprod{\mu_{\cL_1} - \mu_{\cL_2}, \sigma}
    \le O\Paren{\sqrt{\frac1m+\tau_{\mathrm{cov}}}}.
  \end{equation*}

  Finally, since $\nu_j$ is certifiable via $\cL_j$, the closeness condition from \Cref{def:certifiable-mean} gives
  \begin{equation*}
    \|\nu_j - \mu_{\cL_j}\|_1 \le \alpha + \tau \eta^* n = O(\alpha)
    \qquad \text{for each } j \in \set{1,2}.
  \end{equation*}
  Hence
  \begin{equation*}
    \|\nu_1 - \nu_2\|_1 \le O\left(\alpha+\sqrt{\frac1m+\tau_{\mathrm{cov}}}\right).
  \end{equation*}
  This proves \Cref{lem:approx_dp_vol_ub}.
\end{proof}

  \section{Pure DP Lower Bounds}

This appendix proves the pure-DP lower bound in \Cref{thm:dp-lower-bound}. Its three terms come from three different arguments: classical nonprivate estimation, a centered code packing around the uniform distribution, and a packing inside the simplex in $(k-1)$ coordinates.

\subsection{Projection Notation for Lower Bounds}

We work in $(k-1)$ coordinates for the packing argument, and introduce the required projection notation here.

\begin{definition}[Nonnegative $\ell_1$-ball (radius $1$) in $\R^{k-1}$]
  \label{def:probability-simplex-proj}
  Let $\Delta^{k-1}_{\le}$ denote
  \begin{equation*}
    \Delta^{k-1}_{\le} \coloneq \Set{x \suchthat \iprod{\bm{1}, x} \le 1, x\in \R^{k-1}_{\ge 0}} \mper
  \end{equation*}
\end{definition}

\begin{definition}[Restriction and Extension Operators]
  \label{def:restriction-extension}
  We define $\Res \colon \R^k \to \R^{k-1}$ and $\Ext \colon \R^{k-1} \to \R^k$ by
  \begin{align*}
    \Res(v) &\coloneq v_{[k-1]} \\
    \Ext(v) &\coloneq (v, 1- \sum_{i \in [k-1]} v_i) \mper
  \end{align*}
  In words, $\Res$ restricts a vector to its first $k-1$ coordinates, and $\Ext$ appends the final simplex coordinate.
\end{definition}

\begin{fact}[$\ell_1$ Distance Between Restrictions and Extensions]
  \label{fact:ell-1-distance-restriction-extension}
  For any $u,v\in\R^{k-1}$,
  \[
    \bigl\|\Ext(u)-\Ext(v)\bigr\|_1
    = \|u-v\|_1 + \Bigl|\sum_{i=1}^{k-1}(u_i-v_i)\Bigr|
    \in \bigl[\,\|u-v\|_1,\; 2\|u-v\|_1\,\bigr].
  \]
  Moreover, for any $x,y\in\R^k$ with $\sum_{i=1}^k x_i=\sum_{i=1}^k y_i$ (for example $x,y\in\Delta^{k-1}$),
  \[
    \bigl\|\Res(x)-\Res(y)\bigr\|_1
    \;\le\; \|x-y\|_1
    \;\le\; 2\,\bigl\|\Res(x)-\Res(y)\bigr\|_1.
  \]
\end{fact}

\begin{proof}
  For the extension inequalities, we can write
  \[
    \Ext(u)-\Ext(v)
    =\bigl(u-v,\; \,1-\sum_{i=1}^{k-1}u_i - (1-\sum_{i=1}^{k-1}v_i)\bigr)
    =\bigl(u-v,\; -\sum_{i=1}^{k-1}(u_i-v_i)\bigr),
  \]
  so
  \[
    \|\Ext(u)-\Ext(v)\|_1
    =\sum_{i=1}^{k-1}|u_i-v_i|+\Bigl|\sum_{i=1}^{k-1}(u_i-v_i)\Bigr|
    \in \bigl[\,\|u-v\|_1,\; 2\|u-v\|_1\,\bigr]
  \]
  by the triangle inequality. For the restriction inequalities, the lower bound is immediate since restriction can only decrease $\ell_1$ norm. If $\sum_i x_i=\sum_i y_i$, then
  $x_k-y_k=-\sum_{i=1}^{k-1}(x_i-y_i)$ and hence
  \[
    \|x-y\|_1
    =\sum_{i=1}^{k-1}|x_i-y_i|+\Bigl|\sum_{i=1}^{k-1}(x_i-y_i)\Bigr|
    \le 2\,\|\Res(x)-\Res(y)\|_1.
  \]
\end{proof}

We invoke the following result.
\begin{theorem}[Packing Lower Bounds \cite{HopkinsKM22}]
  \label{thm:packing-lower-bound}
  Let $T\ge2$ and $0<\beta\le1/3$. Let $\mathcal{P}=\left\{P_1, \ldots, P_T\right\}$ be a set of distributions, and $P_0$ be a distribution such that for every $P_i \in \mathcal{P},\left\|P_i-P_0\right\|_{\mathrm{TV}} \leq \gamma$. Let $\mathcal{G}=\left\{G_1, \ldots, G_T\right\}$ be a collection of disjoint subsets of some set $\mathcal{Y}$. If there is an $\varepsilon$-DP algorithm $M$ such that $\mathbb{P}_{X \sim P_i^n}\left[M(X) \in G_i\right] \geq 1-\beta$ for all $i \in[T]$, then
  $$
  n \geq \Omega\left(\frac{\log T+\log (1 / \beta)}{\gamma\left(e^{2 \varepsilon}-1\right)}\right)
  $$
  Note that for the usual regime $\varepsilon \leq 1$, we can replace the $e^{2 \varepsilon}-1$ in the denominator with $\varepsilon$.
\end{theorem}
We use \Cref{thm:packing-lower-bound} to show the following lower bound under pure DP. Similar lower bounds already exist in \cite{LiuSYKR20}. Here we focus on the dependence on the failure probability $\beta$.
\begin{theorem}
  \label{thm:dp-lower-bound}
  Assume $k\ge2$, $0<\alpha\le1/8$, $0<\beta\le1/3$, and $0<\eps\le1$. Any $\eps$-DP algorithm that, for every distribution $\mu\in\Delta^{k-1}$, takes $n$ independent batches of $m$ i.i.d.\ samples from $\mu$ and outputs an estimate $\hat\mu$ satisfying $\normo{\hat\mu-\mu}\le\alpha$ with probability at least $1-\beta$ must use
  \begin{equation*}
    n = \Omega\Paren{\frac{k + \log(1/\beta)}{\alpha^2 m} + \frac{k + \log(1/\beta)}{\alpha \sqrt{m} \eps} + \frac{k \log(1/\alpha) + \log(1/\beta)}{\eps} 
    } \mper
  \end{equation*} 
\end{theorem}
\begin{proof}
  \medskip\noindent First term: nonprivate estimation.
  Classical nonprivate estimation of a distribution on $[k]$ to $\ell_1$ error $\alpha$ with failure probability $\beta$ requires $\Omega((k+\log(1/\beta))/\alpha^2)$ individual samples~\cite{Canonne20a}. Since there are $nm$ samples, division by $m$ gives the first term.

  \medskip\noindent Third term: packing inside the simplex.
  From \Cref{fact:ell-1-distance-restriction-extension} we have
  \begin{equation*}
    \normo{\Res(\mu) - \Res(\mu')} \le \normo{\mu - \mu'} \le 2 \normo{\Res(\mu) - \Res(\mu')} \mper
  \end{equation*}
  If we find some points $v_1, \ldots, v_T \in \Delta^{k-1}_{\le}$ (as in \Cref{def:probability-simplex-proj}) such that $\normo{v_i - v_j} \ge 3\alpha$ for all $i \neq j$, and let $\mu_i = \Ext(v_i)$, then define
  \begin{equation*}
    P_i \coloneq \mu_i^{\tensor m}
    \qquad\text{for } i \in [T], \qquad
    P_0 \coloneq \mu_1^{\tensor m} \mper
  \end{equation*}
  Since total variation distance between any two probability distributions is at most $1$, we have $\normtv{P_i - P_0} \le 1$ for every $i \in [T]$. Thus \Cref{thm:packing-lower-bound} applies with $\gamma=1$ to the family $\set{P_i}_{i \in [T]}$ and the sets
  \begin{equation*}
    G_i = \Set{\hat{\mu} \suchthat \normo{\hat{\mu} - \mu_i} \leq \alpha} \mper
  \end{equation*}
  To construct the $3\alpha$-separated points $v_1,\ldots,v_T$, consider all mirror images of $\Delta^{k-1}_{\le}$, namely $C_{s, k}= \Set{x \suchthat \normo{x} \le 1, s_i x_i \ge 0}$ for $s \in \set{\pm 1}^{k-1}$. These sets partition the $\ell_1$ ball of radius $1$ in $\R^{k-1}$. If there exist $T'$ points in this ball at pairwise distance at least $3\alpha$, then reflecting its most populated orthant into the nonnegative orthant gives at least $T = 2^{-k+1} T'$ points in $\Delta^{k-1}_{\le}$ at the same separation. Standard $\ell_1$ packing bounds (see e.g. \cite{Wainwright19}) give $T' \ge (1/(3\alpha))^{k-1}$, hence $T \ge (1/(6\alpha))^{k-1}$, and since $\alpha\le1/8$, \Cref{thm:packing-lower-bound} yields the lower bound
  \begin{equation*}
    n \ge \Omega\Paren{\frac{k \log (1/\alpha) + \log(1/\beta)}{\eps}} \mper
  \end{equation*}

  \medskip\noindent Second term: centered code packing around the uniform distribution.
  Let $\mu_o$ be the uniform distribution on the first $2\floor{k/2}$ elements of $[k]$.
  For each $x \in \set{\pm 1}^{\floor{k/2}}$, define $s(x) \in \set{-1,0,+1}^k$ by
  \begin{equation*}
    s(x)_{2u-1} = x_u, \qquad s(x)_{2u} = -x_u \qquad \text{for } u \in [\floor{k/2}],
  \end{equation*}
  and if $k$ is odd set the last coordinate of $s(x)$ to zero.
  Define
  \begin{equation*}
    \mu(x) \coloneq \mu_o + \frac{4\alpha}{2\floor{k/2}} s(x) \mper
  \end{equation*}
  Since $\sum_{i=1}^k s(x)_i = 0$ and $4\alpha \le 1/2$ under $\alpha \le 1/8$, we have $\mu(x) \in \Delta^{k-1}$ for every $x$.
  Now note that for any such $x$ we have
  \begin{align*}
    \operatorname{KL}(\mu(x) \, \| \, \mu_o)
    &= \sum_{i=1}^{2\floor{k/2}} \mu(x)(i) \log\Paren{\frac{\mu(x)(i)}{\mu_o(i)}} \\
    &= \frac{1}{2\floor{k/2}} \sum_{i=1}^{2\floor{k/2}} \Paren{1 + 4\alpha s(x)_i} \log\Paren{1 + 4\alpha s(x)_i} \mcom
  \end{align*}
  For $\abs{t} \le 1/2$, the inequality $(1+t)\log(1+t) \le t + 2t^2$ holds.
  Applying it with $t = 4\alpha s(x)_i$ and using $\sum_i s(x)_i = 0$, we obtain
  \begin{equation*}
    \operatorname{KL}(\mu(x) \, \| \, \mu_o)
    \le
    \frac{32\alpha^2}{2\floor{k/2}}\sum_{i=1}^{2\floor{k/2}} s(x)_i^2
    \le 32\alpha^2
    = O(\alpha^2) \mper
  \end{equation*}
  By tensorization of KL and Pinsker's inequality,
  \begin{equation*}
    \normtv{\mu(x)^{\tensor m} - \mu_o^{\tensor m}}
    \le
    O(\alpha \sqrt{m}) \mper
  \end{equation*}
  A greedy binary-code packing gives $x_1,\dots,x_M\in\set{\pm1}^{\floor{k/2}}$ with $M=2^{\Omega(\floor{k/2})}=2^{\Omega(k)}$ and pairwise Hamming distance strictly greater than $\floor{k/2}/4$: each chosen word removes at most
  \begin{equation*}
    \sum_{j=0}^{\floor{\floor{k/2}/4}}\binom{\floor{k/2}}j\le2^{H_2(1/4)\floor{k/2}}
  \end{equation*}
  words, where $H_2(1/4)<1$ is the binary entropy. For $i\ne j$,
  \begin{equation*}
    \normo{\mu(x_i)-\mu(x_j)}
    =\frac{8\alpha}{\floor{k/2}}\,d_H(x_i,x_j)
    >\frac{4\alpha}{2}
    =2\alpha.
  \end{equation*}
  Thus the closed $\ell_1$ balls of radius $\alpha$ about the code distributions are disjoint.
  Therefore, applying \Cref{thm:packing-lower-bound} with $\gamma=\min\{1,O(\alpha\sqrt m)\}$ gives
  \begin{equation*}
    n \ge \Omega\Paren{\frac{k + \log(1/\beta)}{\alpha \sqrt{m} \eps}} \mper
  \end{equation*}
\end{proof}

\end{document}